\documentclass[accepted,mathfont=cm]{style/uai2022}
\pdfoutput=1

\usepackage[american]{babel}
\usepackage{natbib} 
\setcitestyle{square} 

\usepackage{booktabs}

\usepackage[utf8]{inputenc} 
\usepackage[T1]{fontenc}    
\usepackage{url}            
\usepackage{booktabs}       
\usepackage{amsfonts}       
\usepackage{nicefrac}       
\usepackage{microtype}      

\usepackage{stmaryrd} 
\usepackage{amsthm}
\usepackage{thmtools, thm-restate} 
\usepackage[normalem]{ulem} 
\usepackage{mdframed} 
\usepackage{multirow} 

\let\OLDthebibliography\thebibliography
\renewcommand\thebibliography[1]{
  \OLDthebibliography{#1}
  \setlength{\parskip}{1pt}
  \setlength{\itemsep}{4pt plus 0.3ex}
}

\usepackage{amsmath}
\usepackage{amssymb} 

\usepackage{float}
\usepackage{subcaption} 

\usepackage{vwcol} 

\usepackage{cleveref} 

\usepackage{pifont} 
\newcommand{\cmark}{\ding{51}}%

\usepackage[dvipsnames]{xcolor} 
\usepackage{multicol}
\usepackage{algorithmic}

\usepackage{mathtools} 

\usepackage{graphicx}
\graphicspath{{pictures/}} 

\usepackage{enumitem}

\usepackage{wrapfig}
\usepackage{microtype}
\usepackage{tocloft}

\usepackage{xcolor}
\usepackage[linesnumbered,ruled,vlined]{algorithm2e}

\SetKwInput{KwInput}{Input}               
\SetKwInput{KwInitialization}{Initialization} 
\SetKwInput{KwOutput}{Output} 

\usepackage{xcolor}
\usepackage{framed}
\theoremstyle{plain}

\usepackage{svg}

\usepackage{tikz}
\usetikzlibrary{shapes.geometric, arrows}
\tikzstyle{centralserver} = [rectangle, rounded corners, minimum width=3cm, minimum height=1cm,text centered, draw=black, fill=red!30]
\tikzstyle{nodeserver} = [rectangle, rounded corners, minimum width=3cm, minimum height=1cm, text centered, draw=black, fill=orange!30]
\usetikzlibrary{automata,positioning}
\usepackage{makecell}

\usepackage{xargs}

\definecolor{brickred}{rgb}{0.8, 0.25, 0.33}
\newcommandx{\aymeric}[1]{\textcolor{brickred}{\textbf{#1}}}
\newcommand{\gs}{\vspace{-.5em}}

\title{Bidirectional compression in heterogeneous settings for distributed or federated learning with partial participation: tight convergence guarantees}

\author{
  Constantin Philippenko \qquad Aymeric Dieuleveut \\
  CMAP, École Polytechnique, Institut Polytechnique de Paris \\
  \texttt{[fistname].[lastname]@polytechnique.edu} \\
}

\crefname{lemma}{Lemma}{Lemmas}
\crefname{fact}{Fact}{Facts}
\crefname{theorem}{Theorem}{Theorems}
\crefname{corollary}{Corollary}{Corollaries}
\crefname{claim}{Claim}{Claims}
\crefname{example}{Example}{Examples}
\crefname{problem}{Problem}{Problems}
\crefname{definition}{Definition}{Definitions}
\crefname{assumption}{Assumption}{Assumptions}
\crefname{subsection}{Subsection}{Subsections}
\crefname{section}{Section}{Sections}
\crefname{algorithm}{Algorithm}{Algoalgorithms}
\crefname{algocf}{alg.}{algs.}
\Crefname{algocf}{Algorithm}{Algorithms}
\crefname{proposition}{Proposition}{Propositions}
\crefname{remark}{Remark}{Remarks}

\newtheorem{theorem}{Theorem}

\newtheorem{definition}{Definition}
\newtheorem{lemma}{Lemma}
\newtheorem{assumption}{Assumption}
\newtheorem{proposition}{Proposition}
\newtheorem{remark}{Remark}

\DeclareMathOperator{\e}{e}

\DeclarePairedDelimiter\floor{\lfloor}{\rfloor}

\newcommand{\Artemis}{\texttt{Artemis}}
\newcommand{\MCM}{\texttt{MCM}}
\newcommand{\Dore}{\texttt{Dore}}
\newcommand{\Diana}{\texttt{Diana}}

\newcommand{\ffrac}[2]{\ensuremath{\frac{\displaystyle #1}{\displaystyle #2}}}
\newcommand{\qqquad}{\qquad\qquad}
\newcommand{\FF}{\mathcal{F}}
\newcommand{\GG}{\mathcal{G}}
\newcommand{\HH}{\mathcal{H}}
\newcommand{\E}{\mathbb{E}}
\newcommand{\N}{\mathbb{N}}
\newcommand{\R}{\mathbb{R}}
\newcommand{\WW}{\R^d}
\newcommand{\V}{\mathbb{V}}
\newcommand{\lnrm}{\left \|} 
\newcommand{\rnrm}{\right \|} 
\newcommand{\Expec}[2]{\E \left[#1~\middle|~#2\right]} 
\newcommand{\Var}[2]{\V \left[#1~\middle|~#2\right]} 
\newcommand{\PdtScl}[2]{ \left \langle #1~\middle|~#2\right \rangle}  
\newcommand{\SqrdNrm}[1]{ \lnrm #1\rnrm^2}  
\newcommand{\bigpar}[1]{\left( #1 \right)} 

\newcommand{\omgC}{\omega_{\mathcal{C}}}

\newcommand{\g}{\textsl{g}} 
\newcommand{\gw}{g} 
\newcommand{\gwk}{\gw_{k+1}}
\newcommand{\gwks}{\gw_{k+1, S_k}}
\newcommand{\gwkhat}{\widehat{\gw}_{k+1}}
\newcommand{\gwkhats}{\widehat{\gw}_{k+1, S_k}}

\newcommand{\gwkstar}{\gw_{k+1,*}}
\newcommand{\sigmstar}{\sigma_{*}}
\newcommand{\iN}{{i=1}^N}
\newcommand{\iS}{{i \in S_k}}

\newcommand{\qbern}{q}

\newcommand{\rset}{\mathbb{R}}
\newcommand{\W}{\mathcal W^2}
\newcommand{\C}{\mathcal C}
\newcommand{\pigv}{\Pi_{\gamma,v}}
\newcommand{\rmd}{\mathrm{d}}
\newcommand{\Rgv}{R_{\gamma,v}}
\newcommand{\up}{\mathrm{up}}
\newcommand{\dwn}{\mathrm{dwn}}
\newcommand{\tr}{\mathrm{Tr}}
\newcommand{\cov}{\mathrm{Cov}}

\newcommand{\cst}{C}

\newcommand{\Fsto}{\FF}
\newcommand{\Fupcomp}{\GG}
\newcommand{\Fdwncomp}{\HH}
\newcommand{\Fsamp}{\mathcal{B}}
\newcommand{\Fartif}{\mathcal{J}}
\newcommand{\Flast}{\mathcal{I}}

\newcommand{\sizefig}{0.35}

\newcommand{\Ftotal}{\Flast_k}
\newcommand{\Fdwncompnotsamp}{\sigma(\Flast_{k} \cup \Fsamp_k)}

\newcommand{\Fupcompnotsamp}{\sigma(\Fupcomp_{k+1} \cup \Fsamp_k)}

\begin{document}

\maketitle
\addtocontents{toc}{\protect\setcounter{tocdepth}{0}}

\begin{abstract}
    We introduce a framework -- \Artemis~-- to tackle the problem of learning in a distributed or federated setting with communication constraints and device partial participation.
    Several workers (randomly sampled) perform the optimization process using a central server to aggregate their computations. To alleviate the communication cost, \Artemis~allows to compress the information sent in \emph{both directions} (from the workers to the server and conversely) combined with a memory mechanism. 
    It improves on existing algorithms that only consider unidirectional compression (to the server), or use very strong assumptions on the compression operator, and often do not take into account devices partial participation.  We provide fast rates of convergence (linear up to a threshold) under weak assumptions on the stochastic gradients (noise's variance bounded only \textit{at optimal point}) in non-i.i.d.~setting, highlight the impact of memory for unidirectional and bidirectional compression, analyze Polyak-Ruppert averaging. We use convergence in distribution to obtain a \textit{lower bound} of the asymptotic variance that highlights practical limits of compression. We propose two approaches to tackle the challenging case of devices partial participation and provide experimental results to demonstrate the validity of our analysis. 
\end{abstract}

\gs\gs\gs
\section{Introduction}\gs\gs
\label{sect:intro}

In modern large scale machine learning applications, optimization has to be processed in a distributed fashion, using a potentially large number $N$ of workers. In the data-parallel framework, each worker only accesses a fraction of the data: new challenges have arisen, especially when communication constraints between the workers are present.

In this paper, we focus on first-order methods, especially Stochastic Gradient Descent \citep{bottou_-line_1999,robbins_stochastic_1951} in a centralized framework: a central machine aggregates the computation of the $N$ workers in a synchronized way. This applies to both the \textit{distributed} \citep[e.g.][]{li_scaling_2014} and the \textit{federated learning} \citep[introduced in][]{konecny_federated_2016,mcmahan_communication-efficient_2017} settings.

Formally, we consider a number of features $d\in \mathbb{N^*}$, and a convex cost  function $F: \mathbb{R}^d \rightarrow \mathbb{R}$. We want to solve the following convex optimization problem:
\gs\gs
\begin{equation} \label{eq:pbsetting}
\min_{w \in \mathbb{R}^d} F(w) \text{ with } F(w) = \frac{1}{N} \sum_{i=1}^N  F_i(w) \,,\gs
\end{equation}
where $(F_i)_{i=1}^N$ is a \emph{local} risk function for the  model $w$ on the worker $i$. Especially, in the classical supervised machine learning framework, we fix a loss $\ell$ and access, on a worker $i$,  $n_i$ observations $(z^i_k)_{1\le k \le n_i}$ following a distribution $D_i$. In this framework, $F_i$ can be either the (weighted) local empirical risk, $ w\mapsto(n_i^{-1}) \sum_{k=1}^{n_i} \ell(w, z^i_k)$ or the  expected risk $w\mapsto\E_{z\sim D_i}[\ell(w,z)]$. At each iteration of the algorithm, each worker can get an \emph{unbiased oracle} on the gradient of the function $F_i$ (typically either by choosing uniformly an observation in its dataset or in  a \emph{streaming fashion}, getting a new observation at each step). 

Our goal is to reduce the amount of information exchanged between workers, to accelerate the learning process, limit the bandwidth usage, and reduce energy consumption. Indeed, the communication cost has been identified as an important bottleneck in the distributed settings \citep[e.g.][]{strom_scalable_2015}. In their overview of the federated learning framework, \citet{kairouz_advances_2019} also underline in Section 3.5 two  possible directions to  reduce this cost: 1) compressing communication from workers to the central server (uplink) 2) compressing the downlink communication.

Most of the papers considering the problem of reducing the communication cost \citep{alistarh_qsgd_2017, agarwal_cpsgd_2018,wu_error_2018,karimireddy_error_2019,mishchenko_distributed_2019,horvath_stochastic_2019,li_acceleration_2020,horvath_better_2020} only focus on compressing the message sent from the workers to the central node. This  direction has the highest potential to reduce the total runtime given that (i) the bandwidth for upload is generally more limited than for download, and that (ii) for some regimes with a large number of workers,  the downlink communication, that corresponds to a ``one-to-$N$'' communication, may  not be the bottleneck compared to  the ``$N$-to-one'' uplink.

Nevertheless, there are several reasons to also consider downlink compression. First, the difference between upload and download speeds is not significant enough at all to ignore the impact of the downlink direction (see \Cref{app:sec:speedtest} for an analysis of bandwidth). If we consider for instance a small number $N$ of workers training a very heavy model -- the  size of  Deep  Learning models generally exceeds hundreds of MB \citep{dean_large_2012,huang_gpipe_2019} --, the training speed will be limited by the exchange time of the updates, thus using downlink compression is key to accelerating the process. 
Secondly, in a different framework in which  a network of smartphones collaborate to train a large scale model in a federated framework, participants to the training would not be eager to download a hundreds of MB for each update on their phone. Here again, downlink compression appears to be necessary. To encompass all situations, our framework implements compression in either or both directions with possibly different compression levels.

Bidirectional compression (i.e. compressing both uplink and downlink) raises  new challenges. In the  downlink step, if we compress the \emph{model}, the quantity compressed does \emph{not} tend to zero. Consequently the compression error significantly hinders convergence. To circumvent this problem we compress the \emph{gradient} that may asymptotically approach zero. 
Double compression has been recently considered by  \citet{tang_doublesqueeze_2019,zheng_communication-efficient_2019,liu_double_2020,yu_double_2019,philippenko_preserved_2021}. 
One of the most recent work, \texttt{Dore}, defined by \citet{liu_double_2020}, analyzed a double compression approach, combining error compensation, a memory mechanism and model compression, with a uniform bound on the gradient variance. 
In this work, we provide new results on \texttt{Dore}-like algorithms, considering a framework \emph{without error-feedback} using tighter assumptions, and quantifying precisely the impact of data heterogeneity on the convergence. 

Moreover, we focus on a \emph{heterogeneous} setting:  the data  distribution depends on  each worker (thus non i.i.d.). We \emph{explicitly control the differences between distributions}. In such a setting, the local gradient at the optimal point $\nabla F_i(w_*)$ may not vanish: to get a vanishing compression error, we introduce a ``memory'' process \citep{mishchenko_distributed_2019}. 
A very recent work by \citet{philippenko_preserved_2021} builds upon our demonstrations precisely to handle non-i.i.d. settings; however they introduce an additional mechanism (``preserved update'') that is orthogonal to this work.

Finally, we encompass in \Artemis~\emph{Partial Participation} (PP) settings, in which only some workers contribute to each step. Several challenges arise with PP, because of both  heterogeneity  and downlink compression. We propose a new algorithm in that setting, which improves on  the approaches proposed by \citet{sattler_robust_2019,tang_doublesqueeze_2019} for bidirectional compression.

Assumptions made on the gradient oracle directly influence the convergence rate of the algorithm: in this paper, we neither assume that the gradients are uniformly bounded \cite[as in][]{zheng_communication-efficient_2019} nor that their variance is uniformly bounded  \citep[as in][]{alistarh_qsgd_2017, mishchenko_distributed_2019, liu_double_2020, tang_doublesqueeze_2019, horvath_stochastic_2019}: instead we only assume that the variance is bounded by a constant~$\sigmstar^2$ \emph{at the optimal point} $w_*$, and provide linear convergence rates up to a threshold \emph{proportional to} $\sigmstar^2$ (as in \citep{dieuleveut_bridging_2018,gower_sgd_2019} for non distributed optimization).
This is a fundamental difference as the variance bound at the optimal point can be orders of magnitude smaller than the uniform bound used in previous work: this is striking when all loss functions have the same critical point, and thus the noise at the optimal point is null!
This happens for example in the  \textit{interpolation regime}, which has recently gained importance in the machine learning community \citep{belkin2019reconciling}. As the empirical risk at the optimal point is null or very close to zero,  so are all the loss functions with respect to one example. This is often the case in deep learning \citep[e.g.,][]{Zhang-rethinkgen-2017} or in large dimension regression \citep{mei_generalization_2019}. 

Overall, we make the following contributions:
\begin{enumerate}[topsep=0pt,itemsep=1pt,leftmargin=*,noitemsep]
    \item We describe a framework -- \Artemis~-- \textbf{that encompasses $6$ algorithms} (with or without up/down compression, with or without memory) which is adapted to PP. We provide and analyze in \Cref{thm:cvgce_artemis} a fast rate of convergence -- exponential convergence up to a threshold proportional to $\sigmstar^2$, the noise at the optimal point --, \textbf{obtaining tighter bounds} than in \citep{alistarh_qsgd_2017,mishchenko_distributed_2019}. 
 
    \item We explicitly tackle heterogeneity using \Cref{asu:bounded_noises_across_devices}, proving that the limit variance of \Artemis~with memory is independent from the difference between distributions (as for SGD). \textbf{This is the first theoretical guarantee for double compression that explicitly quantifies the impact of non i.i.d.~data}. 
    
    \item We propose two approaches to \textbf{tackle the case of partial device participation} when using memories. This setting is challenging due to the difficulty to synchronize memories. The second (and recommended) approach leverages the full potential of the memory to improve convergence.
    
    \item In the non strongly-convex case, we prove the convergence using  Polyak-Ruppert averaging in \Cref{thm:main_PRave}.
    
    \item  We prove \emph{convergence in distribution} of the iterates, and subsequently \textbf{provide \textit{a lower bounds} on the asymptotic variance.} This sheds light on the limits of (double) compression, which results in an increase of the algorithm's variance, and can thus only accelerate the learning process for \emph{early iterations} and up to a \emph{``moderate'' accuracy}. Interestingly, this \emph{``moderate'' accuracy} has to be understood with respect to the \emph{reduced} noise~$\sigmstar^2$.
\end{enumerate}

Furthermore, we support our analysis with various experiments illustrating the behavior of our new algorithm and we provide the code to reproduce our experiments.  See \href{https://anonymous.4open.science/r/d5793937-eb58-4a35-b043-17b83d21225a}{this anonymized repository}.
In \Cref{table:comparison_algo_bi}, we highlight the main features and assumptions of \Artemis~compared to recent algorithms using compression. 

\begin{table*}
\caption{\label{table:comparison_algo_bi}Comparison of frameworks for main algorithms handling (bidirectional) compression. By ``non i.i.d.'',~we mean that the theoretical framework encompasses \textit{and} explicitly quantifies the impact of data heterogeneity on convergence (\Cref{asu:bounded_noises_across_devices}), e.g., \Dore~does not assume i.i.d.~workers but does not quantify differences between distributions.  References: see \cite{alistarh_qsgd_2017} for \texttt{QSGD}, \cite{mishchenko_distributed_2019} for \Diana,  \cite{horvath_better_2020} for [HR20], \cite{liu_double_2020} for \Dore, \cite{philippenko_preserved_2021} for \MCM~and \cite{tang_doublesqueeze_2019} for \texttt{DoubleSquezze}}.
\centering
\begin{tabular}{lcccccccc}
 \hline
  & 
  \texttt{QSGD} 
  & 
  \texttt{Diana} 
  & [HR20]
  & 
  \Dore 
  & \thead{\texttt{Double} \\ \texttt{Squeeze}} 
  & \thead{\texttt{Dist} \\ \texttt{EF-SGD} }
  & \MCM
  & \thead{\textbf{\Artemis} \\ \textbf{(new)}} \\
  \hline
  Data & i.i.d. &\textcolor{ForestGreen}{non i.i.d.} &\textcolor{ForestGreen}{non i.i.d.}&i.i.d.& i.i.d. &i.i.d. & \textcolor{ForestGreen}{non i.i.d.} & \textcolor{ForestGreen}{non i.i.d.} \\
  Bounded variance&\thead{\textcolor{red}{Uniformly}}&\thead{\textcolor{red}{Uniformly}}&\thead{\textcolor{red}{Uniformly}}&\thead{\textcolor{red}{Uniformly}}&\thead{\textcolor{red}{Uniformly}}&\thead{\textcolor{red}{Uniformly}}&\thead{\textcolor{red}{Uniformly}}&\thead{\textcolor{ForestGreen}{At optimal} \\ 
  \textcolor{ForestGreen}{point}} \\
  Compression&\textcolor{red}{One-way}&\textcolor{red}{One-way}&\textcolor{red}{One-way}&\textcolor{ForestGreen}{Two-way}&\textcolor{ForestGreen}{Two-way}&\textcolor{ForestGreen}{Two-way}&\textcolor{ForestGreen}{Two-way}&\textcolor{ForestGreen}{Two-way} \\
  Error-feedback &&&\cmark&\cmark&\cmark&\cmark&& \\
  Memory &&\cmark&&\cmark&&&\cmark&\cmark \\
 
  Device sampling &&&\cmark&&&&\cmark&\cmark \\
  \hline
\end{tabular}
\gs\gs
\end{table*}

The rest of the paper is organized as follows: in \Cref{sect:pb_statment} we introduce the framework of \Artemis. In \Cref{sect:assumptions} we describe the assumptions, and we review related work in \Cref{sect:relatedwork}. We then give the  theoretical results in \Cref{sect:theory}, we extend the result to device sampling in \Cref{sec:extension_partial_participation}, we present experiments in \Cref{sec:expemain}, and finally, we conclude in \Cref{sect:conclusion}.

\gs\gs
\section{Problem statement}\gs\gs
\label{sect:pb_statment}
We consider the problem described in \Cref{eq:pbsetting}. In the convex case, we assume that there exist at least one optimal point which we denote $w_*$, we also denote $h_*^i = \nabla F_i(w_*)$, for $i$ in $\llbracket 1 , N \rrbracket$.
We use $\lnrm \cdot \rnrm$ to denote the  Euclidean norm. 
To solve this problem, we rely on a stochastic gradient descent (SGD) algorithm.

A stochastic gradient $\g_{k+1}^i$  is provided at iteration $k$ in $\N$ to the  device $i$ in $\llbracket 1, N \rrbracket$. This function is then evaluated at point $w_k$: to  alleviate notation, we will use $\gwk^i = \g_{k+1}^i(w_k)$ and $\gwkstar^i = \g_{k+1}^i(w_*)$ to denote the stochastic gradient vectors at points $w_{k}$ and $w_*$ on device $i$.  In the classical centralized framework (without compression), with partial  participation of devices,  SGD corresponds to:
\gs
\begin{align}
   \label{eq:sgd_statement}
    w_{k+1} = w_k - \gamma \frac{1}{N}\sum_{i=1}^N \gwk^i \gs
\gs\gs
\end{align}
where $\gamma$ is the learning rate. Here, we first consider the full participation case.

However, computing such a sequence would require the nodes to send either the gradient $\gwk^i$ or the updated local model to the central server (\emph{uplink} communication), and the central server to broadcast back either the averaged gradient $\gwk$ or the updated global model (\emph{downlink} communication). Here, in order to reduce communication cost, we perform a \textit{bidirectional} compression. More precisely, we combine two main tools:
1) an \emph{unbiased compression operator} $\mathcal{C}: \WW \rightarrow \WW$ that reduces the number of bits exchanged, and 2) a \emph{memory} process that reduces the size of the signal to compress, and consequently the error \citep{mishchenko_distributed_2019,li_acceleration_2020}. That is, instead of directly compressing the gradient, we first approximate it by the memory term and, afterwards, we compress the difference. As a consequence, the compressed term tends in expectation to zero, and the error of compression is reduced.  
Following \citet{tang_doublesqueeze_2019}, we always broadcast gradients and never models. To distinguish the two compression operations we denote $\mathcal{C}_{\up}$ and $\mathcal{C}_{\dwn}$ the compression operator for downlink and uplink. At each iteration, we thus have the following steps:
\begin{enumerate}[topsep=0pt,itemsep=1pt,leftmargin=*]
	\item First, each active local node sends to the central server a compression of gradient differences: $\widehat{\Delta}_k^i = \mathcal{C}_{\up}(\gwk^i - h_k^i)$, and updates the \emph{memory term} $h_{k+1}^i = h^i_k + \alpha \widehat{\Delta}_k^i$ with $\alpha \in \R^*$. The server recovers the approximated gradients' values by adding the received term to the memories kept on its side.
	\item Then, the central server sends back the compression of the sum of compressed gradients: $\Omega_{k+1} = \mathcal{C}_{\dwn} \left(\ffrac{1}{N}\sum_{i =1}^N \widehat{\Delta}_k^i + h_k^i \right)$.  No memory mechanism needs to be used, as the sum of gradients tends to zero in the absence of regularization.
\end{enumerate}
The update is thus given by:
\gs
\begin{align}
\label{eq:update_schema}
\left\{
    \begin{array}{l}
        \forall i \in \llbracket1, N\rrbracket\,, \quad \widehat{\Delta}_k^i = \mathcal{C}_{\up} \left(\gwk^i - h_k^i \right) \\
    	\Omega_{k+1} = \mathcal{C}_{\dwn} \left(\ffrac{1}{ N}\sum_{i=1}^N (\widehat{\Delta}_k^i + h_k^i) \right) \\
       w_{k+1} = w_k - \gamma \Omega_{k+1}
    \end{array}
\right. \hspace{-0.3cm}.
\end{align} 
Constants $\gamma, \alpha \in \R^*\times \R_ +$ are learning rates for respectively the iterate sequence and the memory sequence. The adaptation of the framework in the case of device sampling is developed in \Cref{sec:extension_partial_participation}.  This is illustrated on \Cref{algo,tikz_algo} in \Cref{app:complementary}.

As a summary, the \Artemis~framework encompasses in particular these four algorithms:  the variant with unidirectional compression ($\omgC^\dwn = 0$)  w.o.~or with memory ($\alpha = 0$ or $\alpha \neq 0$) recovers \texttt{QSGD} defined by \citet{alistarh_qsgd_2017} and \texttt{DIANA} proposed by \cite{mishchenko_distributed_2019}. 
The variant using bidirectional compression ($\omgC^\dwn \neq 0$) w.o memory ($\alpha = 0$) is called \texttt{Bi-QSGD}. 
The last and most effective variant combines bidirectional compression \emph{with} memory and is the one we refer to as \Artemis~if no precision is given. It corresponds to a simplified version of \Dore~without error-feedback, but this additional mechanism did not lead to any theoretical improvement \citep[Remark 2 in Sec. 4.1.,][]{liu_double_2020}.
\begin{remark}[Local steps]
An obvious independent direction to reduce communication is to increase the number of steps performed before communication. This is the spirit of Local-SGD~\citep{stich_local_2019}. It is an interesting extension to incorporate this into our framework, that we do not consider in order to focus on the compression insights.\gs
\end{remark}
In the  following section, we present and discuss assumptions over the function $F$, the data distribution and the compression operator. 

\gs
\subsection{Assumptions}\gs\gs
\label{sect:assumptions}

We  make classical assumptions on $F:  \mathbb{R}^d \rightarrow \mathbb{R}$.

\begin{assumption}[Strong convexity]
\label{asu:strongcvx}
$F$ is $\mu$-strongly convex, that is for all vectors $w, v$ in $\WW$:
$ 
F(v) \geq F(w) + (v -w)^T \nabla F(w) + \frac{\mu}{2} \| v - w \|^2_2\,.
$
\gs
\end{assumption}
Note that we do not need each $F_i$ to be strongly convex, but only $F$. Also remark that we only use this inequality for $v=w_*$ in the proof of \Cref{thm:cvgce_artemis,thm:main_PRave}.

Below, we introduce cocoercivity \cite[see][for more details about this hypothesis]{zhu_co-coercivity_1996}. This assumption implies that all $(F_i)_{i \in \llbracket 1, N \rrbracket}$ are $L$-smooth. 

\begin{assumption}[Cocoercivity of stochastic gradients (in quadratic mean)]
\label{asu:cocoercivity}
  We suppose that for all $k$ in $\N$,  stochastic gradients functions $(\g_k^i)_{i \in \llbracket1, N \rrbracket}$ are  L-cocoercive in quadratic mean. That is, for  $k$ in $\N$,   $i$ in $\llbracket1, N \rrbracket$ and for all vectors $w_1, w_2$ in $\WW$, we have:$
\E [ \| \g^i_{k}(w_1) - \g_{k}^i(w_2) \|^2 ] \leq 
  L  \PdtScl{\nabla F_i (w_1) - \nabla F_i(w_2)}{w_1 - w_2} \,.$
\end{assumption}
\gs\gs
E.g., this is  true under the much stronger assumption that  stochastic gradients functions $(\g_k^i)_{i \in \llbracket1, N \rrbracket}$ are \textit{almost surely} $L$-cocoercive, i.e.: 
  $  \| \g^i_{k}(w_1) - \g_{k}^i(w_2) \|^2 \leq L \PdtScl{\g^i_{k}(w_1) - \g_{k}^i(w_2)}{w_1 - w_2} \,$.
Next, we present the assumption on the stochastic gradient's noise. Again, we highlight that the noise is only controlled at the optimal point.  
To carefully control the noises process (gradient oracle, uplink and downlink compression), we introduce three filtrations $(\Fdwncomp_k,\Fupcomp_k,\Fsto_k)_{k\geq 0}$,  such that $w_k$ is $\Fdwncomp_k$-measurable for any $k\in \N$. Detailed definitions are given in \Cref{sect:filration}.

\begin{assumption}[Noise over stochastic gradients computation]
\label{asu:noise_sto_grad}
The noise over stochastic gradients at the global optimal  point, for a mini-batch of size $b$, is bounded: there exists a constant $\sigmstar \in \mathbb{R}$, s.~t.~for all $k$ in $\N$, for all $i$ in $\llbracket 1, N \rrbracket\,$, we have a.s.:
 $\quad\E [ \|\gwkstar^i - \nabla F_i(w_*)\|^2 | \Fdwncomp_{k} ] \leq \frac{\sigmstar^2}{b}.$
\end{assumption}
\gs
The constant $\sigmstar^2$ is null, e.g.~if we use deterministic (batch) gradients, \textit{or} in the interpolation regime for i.i.d.~observations, as discussed in the Introduction. As we have also incorporated here a mini-batch parameter, this reduces the variance by a factor $b$. 

Unlike \texttt{Diana} \citep{mishchenko_distributed_2019,li_acceleration_2020}, \Dore~\citep{liu_double_2020}, \texttt{Dist-EF-SGD} \citep{zheng_communication-efficient_2019}, \MCM~\citep{philippenko_preserved_2021} or \texttt{Double-Squeeze} \citep{tang_doublesqueeze_2019}, we assume that the variance of the noise is bounded \textit{only at optimal point} $w_*$ and not \textit{at any point} $w$ in $ \WW$. It results that if variance is null ($\sigmstar^2 = 0$) at optimal point, we obtain a linear convergence while previous results obtain this rate solely if the variance is null \textit{at any point} (i.e.~only for deterministic GD). Also remark that \Cref{asu:cocoercivity,asu:noise_sto_grad} both stand for the simplest Least-Square Regression (LSR) setting, while the uniform bound on the gradient's variance \emph{does not}. Next, we give the assumption that links the distributions on the different machines.

\begin{assumption}[Bounded gradient at $w_*$]
\label{asu:bounded_noises_across_devices}
There exists a constant $B \in \mathbb{R_+}$, s.t.: 
\gs
$
\frac{1}{N} \sum_{i=0}^N \| \nabla F_i(w_*)\|^2 = B^2\,.
$ 
\end{assumption}
\gs
This assumption is used to quantify how different the distributions are on the different machines. In the streaming \emph{i.i.d.~setting} -- $D_1= \dots = D_N$ and $F_1= \dots = F_N$ --  the assumption is satisfied with $B=0$. 
Combining \Cref{asu:noise_sto_grad,asu:bounded_noises_across_devices} results in an upper bound on the averaged squared norm of stochastic gradients at $w_*$: for all $k$ in $\N$, a.s.,
$\frac{1}{N} \sum_{i=1}^N\E [ \|\gwkstar^i\|^2 | \Fdwncomp_{k} ] \leq \ffrac{\sigmstar^2}{b} + B^2$.

Finally, compression operators can be classified in two main categories: quantization \citep[as in][]{alistarh_qsgd_2017,seide_1-bit_2014,zhou_dorefa-net_2018,wen_terngrad_2017,reisizadeh_fedpaq_2020, horvath_stochastic_2019} and sparsification \citep[as in][]{stich_sparsified_2018,aji_sparse_2017,alistarh_convergence_2018,khirirat_communication_2020}. Theoretical guarantees provided in this paper do not rely on a particular kind of compression, as we only consider the following assumption on the compression operators~$\mathcal{C}_\up$~and~$\mathcal{C}_\dwn$:
\begin{assumption}
\label{asu:expec_quantization_operator}
There exist constants $\omgC^\up\,,\omgC^\dwn \in \R^*_+$, such that the compression operators $\mathcal{C}_{\up}$ and $\mathcal{C}_{\dwn}$ verify the two following properties for all $\Delta$ in $\R^d$:
\[
\left\{
    \begin{array}{ll}
    	\E [\mathcal{C}_{\up\slash\dwn}(\Delta)] = \Delta \,, \\ 
    	\E [ \SqrdNrm{\mathcal{C}_{\up\slash\dwn}(\Delta) - \Delta}] \leq \omgC^{\up\slash\dwn} \SqrdNrm{\Delta} \,.
    \end{array}
\right.
\]
\end{assumption}
In other words, the compression operators are unbiased and their variances are bounded. Note that \citet{horvath_better_2020} have shown that using an unbiased operator leads to better performances. Unlike us, \citet{tang_doublesqueeze_2019} assume uniformly bounded compression error, which is a much more restrictive assumption.  We now provide additional details on related papers dealing with compression. Also note that  $\omgC^{\up\slash\dwn}$ can be considered as \emph{parameters} of the algorithm, as the compression levels can be chosen.

\gs
\subsection{Related work on compression}\gs\gs
\label{sect:relatedwork}

Quantization is a common method  for compression and is used in various algorithms. 
For instance, \citet{seide_1-bit_2014} are one of the first to propose to quantize each gradient component by either $-1$ or $1$. 
This approach has been extended in \citet{karimireddy_error_2019}. \citet{alistarh_qsgd_2017} define a new algorithm -- \texttt{QSGD} -- which instead of sending gradients, broadcasts their quantized version, getting robust results with this approach. 
On top of gradient compression, \citet{wu_error_2018} add an error compensation mechanism which accumulates quantization errors and corrects the gradient computation at each iteration. \texttt{Diana} \citep[introduced in][]{mishchenko_distributed_2019} introduces a ``memory'' term in the place of accumulating errors. 
\citet{li_acceleration_2020} extend this algorithm and improve its convergence by using an accelerated gradient descent. 
\citet{reisizadeh_fedpaq_2020} combine unidirectional quantization with device sampling, leading to a framework closer to Federated Learning settings where devices can easily be switched off.
In the same perspective, \citet{horvath_better_2020} detail results that also consider PP.
\citet{tang_doublesqueeze_2019} are the the first to suggest a bidirectional compression scheme for a decentralized network. For both uplink and downlink, the method consists in sending a compression of gradients combined with an error compensation. 
Later, \citet{yu_double_2019} choose to compress models instead of compressing gradients. 
This approach is enhanced by \citet{liu_double_2020} who combine model compression with a memory mechanism and an error compensation drawing from \cite{mishchenko_distributed_2019}. Both \citet{tang_doublesqueeze_2019} and \citet{zheng_communication-efficient_2019} compress gradients without using a memory mechanism. However, as proved in the following section, memory is key to reducing the asymptotic variance in the heterogeneous~case.
A recent work written by \citet{philippenko_preserved_2021} build upon our work and design an algorithm that is doing bidirectional compression but achieves rates of convergence identical to unidirectional compression. In their work, they take advantage of the uplink memory to handle the heterogeneous settings by reusing our demonstration's paradigm.

We now provide theoretical results about the convergence of bidirectional compression. 
\gs\gs
\section{Theoretical results}\gs\gs
\label{sect:theory}
In this section, we present our main theoretical results on the  convergence of \Artemis~and its variants. For the sake of clarity, the most complete and tightest versions of theorems are given in Appendices, and simplified versions are provided here. The main linear convergence rates are given in \Cref{thm:cvgce_artemis}. In \Cref{thm:main_PRave} we show that  \Artemis~combined with Polyak-Ruppert averaging reaches a sub-linear convergence rate. In this section, we denote $\delta_0^2 = \SqrdNrm{w_0 - w_*}$.

\begin{theorem}[Convergence of \Artemis]
\label{thm:cvgce_artemis} 
Under \Cref{asu:strongcvx,asu:cocoercivity,asu:noise_sto_grad,asu:bounded_noises_across_devices,asu:expec_quantization_operator}, for a step size $\gamma$ satisfying the conditions in \Cref{table:summary_convergence}, for a learning rate $\alpha $ and for any $k$ in $\N$,
the mean squared distance to $w_*$ decreases at a linear rate up to a constant of the order of $E$:
\begin{align*}
    \E \left[\lnrm w_k - w_* \rnrm^2\right] &\leq (1- \gamma \mu)^{k} \left(\delta_0^2 + 2 \cst \gamma^2 B^2 \right)  + \frac{2  \gamma E}{\mu N} \label{eq:main_in_theorem} \,,
\end{align*}
for constants $\cst$ and $E$ depending on the variant (independent of $k$) given in \Cref{tab:p_and_E} or in the appendix. Variants with $\alpha\neq0$ require $\alpha  \in [ 1/2 (\omgC^\up+1), \alpha_{\max}]$, the upper bound $\alpha_{\max}$ is given in \Cref{app:thm:with_mem}.
\end{theorem}

\begin{table}
    \centering
    \caption{Details on constants $\cst$ and $E$ defined in \Cref{thm:cvgce_artemis}. $\cst = 0 $ for $\alpha=0$, see Th. \ref{app:thm:with_mem} for $\alpha\neq 0$.\gs}
   \resizebox{\linewidth}{!}{
   \begin{tabular}{ll}
    $\alpha$ & $E$ \\
    \hline
     $ 0$ &  $(\omgC^\dwn + 1) \left((\omgC^\up + 1) \ffrac{\sigmstar^2}{b} + \omgC^\up B^2 \right)$\\
     $\neq 0$ & $\ffrac{\sigmstar^2}{b} \left( (2\omgC^\up + 1)(\omgC^\dwn+1) + 4 \alpha^2 \cst (\omgC^\up + 1) - 2\alpha \cst \right)$ \\
     \hline
    \end{tabular}}
    \gs\gs
    \label{tab:p_and_E}
\end{table}

This theorem is derived from \Cref{app:thm:without_mem,app:thm:with_mem} which are respectively proved in \Cref{app:proof_singlecompressnomemory,app:proof_doublecompressnomem}.
We can make the following remarks:
\begin{enumerate}[topsep=0pt,itemsep=1pt,leftmargin=*]
     \item \textbf{Linear convergence.} The convergence rate given in \Cref{thm:cvgce_artemis} can be decomposed into two terms: a bias term, forgotten at linear speed $(1-\gamma \mu)^k$, and a variance residual term which corresponds to the \emph{saturation level} of the algorithm.
    The rate of convergence $(1-\gamma \mu)$ does not depend on the variant of the algorithm. However, the variance and initial bias do vary.  
    \item \textbf{Bias term.} The initial bias always depends on $\SqrdNrm{w_0 - w_*}$, and when using memory (i.e.~$\alpha\neq 0$) it also depends on the difference between distributions (constant $B^2$). 
    \item \textbf{Variance term and memory.} On the other hand, the variance depends a) on both  $\sigmstar^2/b$, and the distributions' difference $B^2$ without memory b) only on the gradients' variance \emph{at the optimum} $\sigmstar^2/b$ with memory. Similar theorems in related literature  \cite{liu_double_2020,alistarh_qsgd_2017,mishchenko_distributed_2019,yu_double_2019,tang_doublesqueeze_2019,zheng_communication-efficient_2019} systematically had a worse bound for the variance term depending on a \emph{uniform bound of the noise variance} or under much stronger conditions on the compression operator. This paper and \citep{liu_double_2020} are also the first to give a linear convergence up to a threshold for bidirectional compression.
    \item \textbf{Impact of memory.} To the best of  our knowledge, this is the first work on double compression that explicitly tackles the non i.i.d.~case (\citet{philippenko_preserved_2021} also handle this setting but have mentioned that they get inspired from our work). We prove that memory makes the saturation threshold independent of $B^2$ for \Artemis. 
    \item \textbf{Variance term.} The variance term increases with a factor proportional to $\omgC^\up$ for the unidirectional compression, and proportional to $\omgC^\up \times \omgC^\dwn$ for bidirectional. This is the counterpart of compression, each compression resulting in a multiplicative factor on the noise. A similar increase in the variance appears in \citep{mishchenko_distributed_2019} and \citep{liu_double_2020}.  The noise level is attenuated by the number of devices $N$, to which it is inversely proportional.
    \item \textbf{Link with classical SGD.} For variant of \Artemis~with $\alpha = 0$, if $\omgC^{\up/\dwn} = 0$ (i.e.~no compression) we recover SGD results: convergence does not depend on $B^2$, but only on the noise's variance.
\end{enumerate}
\textbf{Conclusion:} Overall, it appears that \Artemis~is able to efficiently accelerate the learning  during first iterations, enjoying the same linear rate as SGD with lower communication complexity, but it saturates at a higher level,  proportional to $\sigmstar^2$ and independent of $B^2$.

The range of acceptable learning rates is an important feature for first order algorithms. 
In \Cref{table:summary_convergence}, we summarize the upper bound $\gamma_{\max}$ on $\gamma$, to  guarantee a $(1 - \gamma \mu)$ convergence of \Artemis. These bounds are derived from \Cref{app:thm:without_mem,app:thm:with_mem}, in three main asymptotic regimes: $N \gg \omgC^{\up}$, $N \approx \omgC^{\up}$ and $\omgC^{\up} \gg N$.  Using bidirectional compression impacts $\gamma_{\max}$ by a factor $\omgC^\dwn + 1$ in comparison to unidirectional compression. For unidirectional compression, if the number of machines is at least of the order of $\omgC^\up$, then $\gamma_{\max}$ nearly corresponds to $\gamma_{\max}$ for  vanilla (serial) SGD.

\begin{table}
\centering
\caption{Upper bound on $\gamma_{\max}$ to guarantee convergence. For unidirectional compression (resp. no compr.), $\omgC^\dwn = 0$ (resp. $\omgC^{\up/\dwn} = 0$, recovering classical rates for SGD). \vspace{-0.3em}}
\centering
\label{table:summary_convergence}
\begin{tabular}{lccc}
  Memory & $\alpha = 0$ &  $\alpha \neq 0$ \\
  \hline
  \vspace{.2em}
  \multirow{2}{*}{$N \gg \omgC^\up$} & $\ffrac{1}{(\omgC^\dwn + 1)L}$ & $\ffrac{1}{2(\omgC^\dwn + 1)L}$ \\
  \vspace{.2em}
  \multirow{2}{*}{$N \approx \omgC^\up$} & $\ffrac{1}{3(\omgC^\dwn + 1)L}$ & $\ffrac{1}{5(\omgC^\dwn + 1) L }$ \\ 
  \vspace{.2em}
  \multirow{2}{*}{$\omgC^\up \gg N$} & $\ffrac{N}{2\omgC^\up (\omgC^\dwn +1) L}$ & $\ffrac{N}{4 \omgC^\up (\omgC^\dwn + 1) L}$ \vspace{.2em}\\
  \hline\vspace{-2.5em}
\end{tabular}
\end{table}

We now provide a convergence guarantee for the averaged iterate without strong convexity.
\begin{theorem}[Convergence of \Artemis~with Polyak-Ruppert averaging]
\label{thm:main_PRave}
Under \Cref{asu:cocoercivity,asu:noise_sto_grad,asu:bounded_noises_across_devices,asu:expec_quantization_operator,asu:partial_participation} 
(convex case) with constants $\cst$ and $E$ as in \Cref{thm:cvgce_artemis} (see \cref{tab:p_and_E} for precision), after running $K$ in $\N$ iterations, for a learning rate
$\gamma = \min \left( \sqrt{\frac{N \delta_0^2}{2E K}}; \gamma_{\max }\right)$, 
with $\gamma_{\max} $  as in \Cref{table:summary_convergence}, we have a sublinear convergence rate for the Polyak-Ruppert averaged iterate $\bar w_K = \frac{1}{K}\sum_{k=0}^K w_k$, with $\varepsilon_F(\bar w_K) = F\left( \bar w_K  \right) - F(w_*)$:
\begin{align*}
    \varepsilon_F(w_K) &\leq 2 \max \left(\sqrt{\frac{2 \delta_0^2  E}{N K}}; \frac{\delta_0^2}{\gamma_{\max} K} \right) +  \frac{ 2 \gamma_{\max} \cst B^2}{K}  .
\end{align*}
\end{theorem}
This theorem is proved in \Cref{app:doublecompress_avg}. Several comments can be made on this theorem:
\begin{enumerate}[topsep=0pt,itemsep=1pt,leftmargin=*]
     \item \textbf{Importance of averaging} This is the first theorem given for averaging for double compression. In the context of convex optimization, averaging has been shown to be optimal \citep{rakhlin2012making}.
    
    \item \textbf{Speed of convergence, if $\sigmstar = 0$, $B\neq 0$, $K\to \infty$.} For $\alpha \neq 0$, $E=0$,  while for $\alpha =0$, $E\varpropto B^2$. Memory thus accelerates the convergence  from a rate $O(K^{-1/2})$ to $O(K^{-1})$. 
    
    \item \textbf{Speed of convergence, general case.}
    More generally, we always get a $K^{-1/2}$ sublinear speed of convergence, and a faster rate  $K^{-1}$,  when using memory,  and if $E\le \delta_0^2 N /(2K \gamma_{\max}^2)$ -- i.e.~in the context of a low noise $\sigmstar^2$, as $E\varpropto \sigmstar^2$. 
    Again, it appears that bi-compression is mostly useful in low-$\sigmstar^2$ regimes or during the first iterations: intuitively, for a fixed communication budget, while bi-compression allows to perform $\min\{\omgC^\up, \omgC^\dwn\}$-times more iterations,  this is no longer beneficial if the convergence rate is dominated by $\sqrt{2 \delta_0^2  E/ N K}$, as $E$ increases by a factor $\omgC^\up\times \omgC^\dwn$. \label{item:avg}
     \item \textbf{Memoryless case, impact of minibatch.} In the variant of \Artemis~\emph{without memory}, the asymptotic convergence rate is $\sqrt{2 \delta_0^2  E/ N K}$ with the constant $E\varpropto \sigmstar^2/b + B^2$: interestingly, it appears that in the case of non i.i.d.~data ($B^2>0$), the \emph{convergence rate saturates when the size of the mini-batch increases}: large mini-batches \emph{do not help}. On the contrary, with memory, the variance is, as classically, reduced by a factor proportional to the size of the batch, without saturation.
\end{enumerate}

\subsection{Convergence in distribution and lower bound}\label{app:cvdist}
The increase in the variance (in \cref{item:avg}) is not an artifact of the proof: we prove the existence of a limit distribution for the iterates of \Artemis, and analyze its variance. More precisely, we show a linear rate of convergence for the distribution $\Theta_k$ of $w_k$ (launched from $w_0$), w.r.t.~the Wasserstein distance $\mathcal{W}_2$~\citep{Vil_2009}: this gives us a lower bound on the asymptotic variance. Here, we further assume that the compression operator is \emph{Stochastic sparsification}~\citep{wen_terngrad_2017}.
\gs
\begin{theorem}[Convergence in distribution and lower bound on the variance]\label{thm:cvdist}
Under \Cref{asu:strongcvx,asu:cocoercivity,asu:noise_sto_grad,asu:bounded_noises_across_devices,asu:expec_quantization_operator} (full participation setting), for $\gamma,\alpha,E$ given in \Cref{thm:cvgce_artemis} and \Cref{table:summary_convergence}:
\begin{enumerate}[topsep=0pt,itemsep=1pt,leftmargin=*,noitemsep]
    \item There exists a limit distribution $\pi_{\gamma, v}$ depending on the {variant $v$} of the algorithm, s.t.~ for any $k\geq 1$, $\mathcal W_2(\Theta_k, \pi_{\gamma, v}) \le (1-\gamma \mu)^k C_0 $, with $C_0$ a constant.
    \item When $k$ goes to  $\infty$, the second order moment $\E[\SqrdNrm{w_k-w_*}]$ converges to $\E_{w \sim \pi_{\gamma,v}}[\SqrdNrm{w-w_*}]$,  which is \emph{lower bounded} by  $\Omega(\gamma E/ \mu N)$ as in \Cref{thm:cvgce_artemis} as $\gamma \to 0$, with $E$ depending on the variant.
\end{enumerate}
\end{theorem}
\gs
 \textbf{Interpretation.} The second point (2.) means that the upper bound on the saturation level provided in~\cref{thm:cvgce_artemis} \emph{is tight} w.r.t.~$\sigmstar^2, \omgC^\up, \omgC^\dwn,B^2, N$ and $\gamma$. Especially, it proves that there is indeed a quadratic increase in the variance w.r.t.~$\omgC^\up$ and $\omgC^\dwn$ when using bidirectional compression (which is itself rather intuitive). Altogether, these three theorems prove that bidirectional compression can become strictly worse than usual stochastic gradient descent in high precision regimes, a fact of major importance in practice and barely (if ever) even mentioned in previous literature.
To the best of our knowledge, only \cite{mayekar_ratq_2020} are giving a lower bound on the asymptotic variance for algorithms using compression. 
There result is more general i.e., valid for any algorithm using unidirectional compression, but weaker (worst case on the oracle does not highlight the importance of noise at the optimal point and is incompatible with linear rates).

\textbf{Proof and assumptions.} This theorem also naturally requires, for the second point, \Cref{asu:noise_sto_grad,asu:bounded_noises_across_devices,asu:expec_quantization_operator} to be ``tight'': that is, e.g.~$\text{Var}(\gwkstar^i) \geq \Omega(\sigmstar^2/{b})$; more details and the proof
are given in \Cref{app:distrib_convergence}.
Extension to other types of compression reveals to be surprisingly non-simple, and is thus out of the scope of this paper and a promising direction.

\gs\gs
\section{Partial Participation}\gs\gs
\label{sec:extension_partial_participation}

In this section we extend our work to the Partial Participation (PP) setting, considering \Cref{asu:partial_participation}. 
\begin{assumption}
\label{asu:partial_participation}
At each round $k$ in $\N$, each device has a probability $p$ of participating, independently from other workers, i.e.,  there exists a sequence $(B^i_k)_{k,i}$ of i.i.d.~Bernoulli random variables $\mathcal{B}(p)$, such that for any $k$ and $i$, $B^i_k$ marks if device $i$ is active at step $k$. We denote $S_k = \{ i \in \llbracket 1, N \rrbracket \mid B_k^i = 1 \}$ the set of active devices at round $k$ and $N_k = \mathrm{card} (S_k)$ the number of active workers. 
\end{assumption}

There are two approaches to extend the update rule.
The first one (we will refer to it as \textbf{PP1}) is the most intuitive. It consists in keeping all memories $(h_k^i)_{1\le i\le N}$ on the central server. 
This way, the central server can reconstruct at each iteration $k$ in $\N$ and for each device $i$ in $\{1, \cdots, N\}$ the stochastic gradient $\gwkhat^i$ defined as $\widehat{\Delta}_k^i +  h_k^i$. 
The update equation is $w_{k+1} = w_k - \gamma \mathcal{C}_{\dwn} \bigpar{ \frac{1}{pN} \sum_{i \in S_k} \widehat{\Delta}_k^i + h_k^i }$, and memories are updated as usually. If both compression levels are set to 0, this approach recovers classical \texttt{SGD} with PP: $w_{k+1} = w_k -   \frac{ \gamma}{pN} \sum_{i \in S_k} \g^i_{k}(w_{k}) $. It also  corresponds to the proposition of both \citet{sattler_robust_2019} and \citet[][v2 on arxiv for the distributed case]{tang_doublesqueeze_2019}.
However, it has two important drawbacks. First, the central server has to store $N$ additional memories, which may have a huge cost. Secondly, this method saturates, even with deterministic gradients $\sigma^2_{\text{unif}}=0$ and no compression  $\omgC^{\up/\dwn} = 0$ (see \Cref{fig:real_dataset_PP1}). Indeed, PP induces an additive noise term, i.e. the variance of the noise at the optimum point is not null because we have: $\text{Var}(\frac{1}{Np}\sum_{i\in S_k} \nabla F_i(w_*)) = \frac{1-p}{N^2 p}\sum_{i=1}^N \SqrdNrm{\nabla F_i(w_*)} = \frac{(1-p)B^2}{Np}$. This happens for all compression regimes in \Artemis, even for SGD.

We consider a novel approach (denoted \textbf{PP2})  that leverages the full potential of the memory, solving simultaneously the convergence issue and the need for additional memory resources. At each iteration $k$, this approach keeps a single memory $h_k$ (instead of $N$ memories) on the central server. This memory is updated at each step: $h_{k+1} = h_k +  \frac{\alpha}{N} \sum_{i \in S_k} \widehat{\Delta}_k^i$, and the update equation becomes: $w_{k+1} = w_k - \gamma \mathcal{C}_{\dwn} \bigpar{ h_k + \frac{1}{pN} \sum_{i \in S_k} \widehat{\Delta}_k^i }$. This difference is far from being insignificant. Indeed in order to reconstruct the broadcast signal, we use the memory built on all devices during previous iterations, \textit{even if} the device $i$ in $\{1, \cdots, N\}$ was not active! \textbf{The impact of this approach is major} as it follows that algorithms, even using bidirectional compression, can be \textit{faster than classical} \texttt{SGD}. In this setting, \texttt{SGD} with memory (i.e \Artemis-\textbf{PP2}~with $\omgC^{\up/\dwn} = 0$) will be the benchmark: see \Cref{fig:real_dataset_PP2}.

\begin{theorem}[Artemis with partial participation]
\label{thm:partial_participation}
Under the same assumptions and constraints on $\gamma$  and $\alpha$, when considering \Cref{asu:partial_participation} (partial participation), \Cref{thm:cvgce_artemis} is still valid for \textbf{PP2} with memory. We have: $E = (\omgC^\dwn + 1) \bigpar{\frac{2(\omgC^\up + 1)}{p} -1} \frac{\sigmstar^2}{b} + 2p C \bigpar{2\alpha^2 (\omgC^\up + 1) - \alpha } \frac{\sigmstar^2}{b}$.
\end{theorem}
The most  important observation is that we recover  again a linear convergence rate if $\sigmstar=0$. By contrast, even with memory and without compression, for \textbf{PP1}, there is an extra $(B^2(1-p)/(Np)$ term.
 
\begin{remark}[Impact of downlink compression]
With both of these approaches, we still need to synchronize the model in order to compute the stochastic gradient on the same up-to-date model for each worker. Thus, a newly active worker must catch-up the sequence of missed updates $(\Omega_k)_k$.  Of course, if a device has been inactive for too many iterations, we send the full model instead of the sequence of missed updates. 
Thus, this need for synchronization does not lead to additional computational resources. The  mechanism is described in \Cref{algo} and does not interfere with the theoretical analysis. \gs
\end{remark}

\gs\gs\gs
\section{Experiments}\label{sec:expemain}\gs\gs

In this section, we illustrate our theoretical guarantees on both synthetic and real datasets.
The goal of this section is to confirm the theoretical findings in \Cref{thm:partial_participation,thm:cvdist,thm:cvgce_artemis,thm:main_PRave}, and to underline the impact of the memory. Therefore, we  focus on five of the algorithms covered by our framework: \Artemis~with bidirectional compression (simply denoted \Artemis), \texttt{QSGD}, \texttt{Diana}, \texttt{Bi-QSGD}, and usual \texttt{SGD} without any compression.
In the Appendix (see \Cref{fig:artemis_vs_existing}), we compare \Artemis~with other existing benchmarks : \texttt{Double-Squeeze}, \texttt{Dore}, \texttt{FedSGD} and \texttt{FedPAQ} \citep[see][]{reisizadeh_fedpaq_2020}. We also perform experiments with \emph{optimized} learning rates (see \Cref{fig:app:gamma_opt_for_each_algo}).

In all experiments, we display the logarithm excess error $\log_{10}(F(w_k)- F(w_*))$ w.r.t.~the number of iterations $k$ or the number of communicated bits. We use a quantization scheme (defined in \Cref{app:quantization_schema}) with  $s=2^0$ in full participation settings, and with $s=2^1$ in PP settings. Curves are averaged over $5$ runs, and we plot error bars on all figures. These errors bars correspond to $\pm$ the standard deviation of the logarithm excess loss over the five runs.

We first consider two simple synthetic datasets: one for least-squares regression (with the same distribution over each machine), and one for logistic regression (with varying distributions across machines). More details are given in \Cref{app:experiments} on the way data is generated. 
We use $N=20$ devices, each holding $200$ points of dimension $d=20$, and run algorithms over $100$ epochs. 

To illustrate theorems on real data and higher dimension, we then consider two real-world dataset: \textit{superconduct} \citep[see][with 21 263 points and 81 features]{hamidieh_data-driven_2018} and \textit{quantum} \citep[see][with 50 000 points and 65 features]{caruana_kdd-cup_2004} with $N=20$ workers.
To simulate non-i.i.d.~\emph{and} unbalanced workers, we split the dataset in heterogeneous groups, using a Gaussian mixture clustering on the TSNE representations (defined by \citet{maaten_visualizing_2008}). Thus, the distribution  and number of points hold by each worker largely differs between devices, see \Cref{app:fig:quantum_and_superconduct_tsne}.

\begin{figure}
    \centering
    \begin{subfigure}{0.22\textwidth}
        \centering
        \includegraphics[width=1\textwidth]{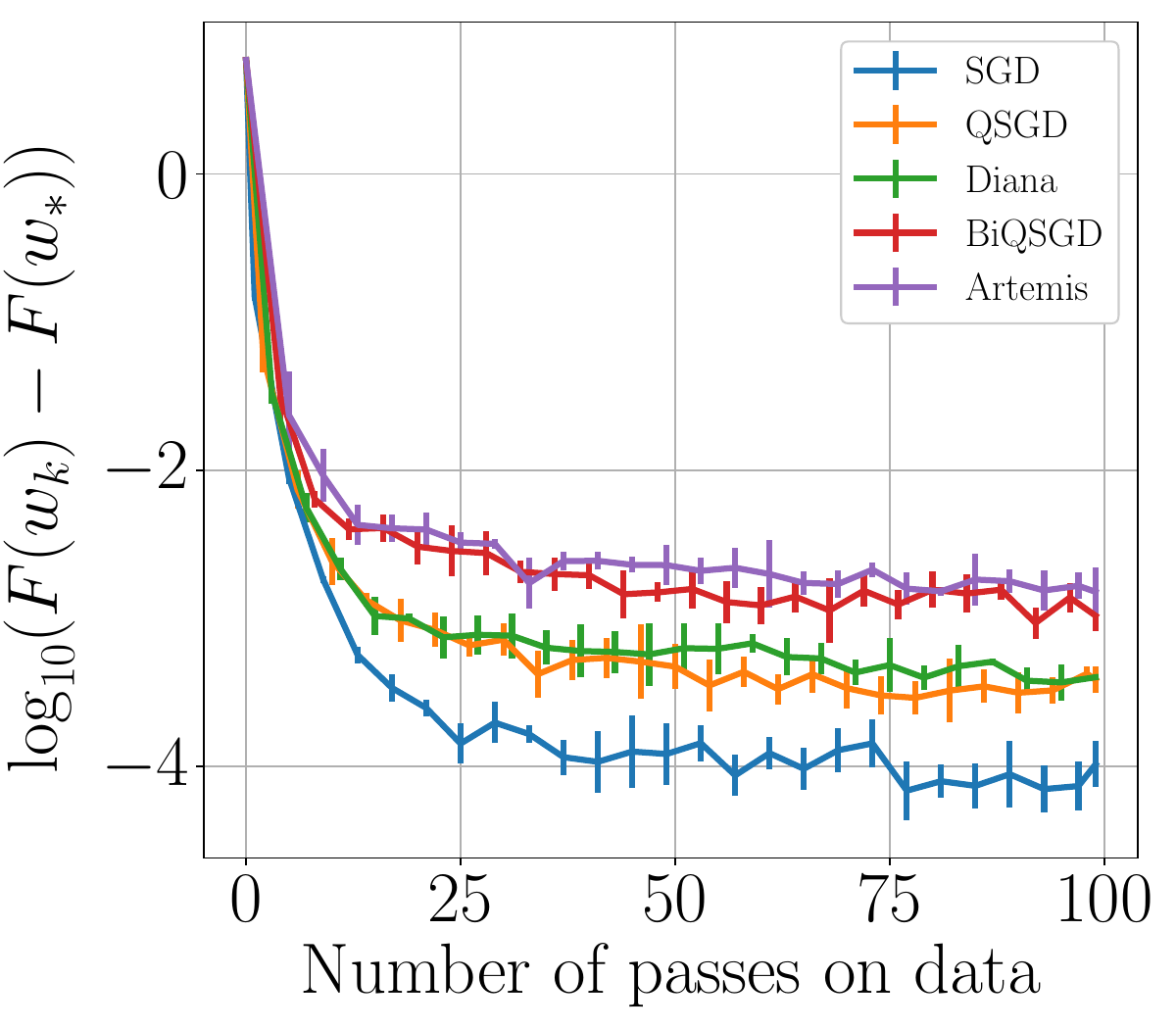}
        \caption{LSR (i.i.d.): $\sigmstar^2 \neq 0$ \gs}
        \label{fig:LSR_noised}
    \end{subfigure}
    \begin{subfigure}{0.22\textwidth}
        \centering
        \includegraphics[width=1\textwidth]{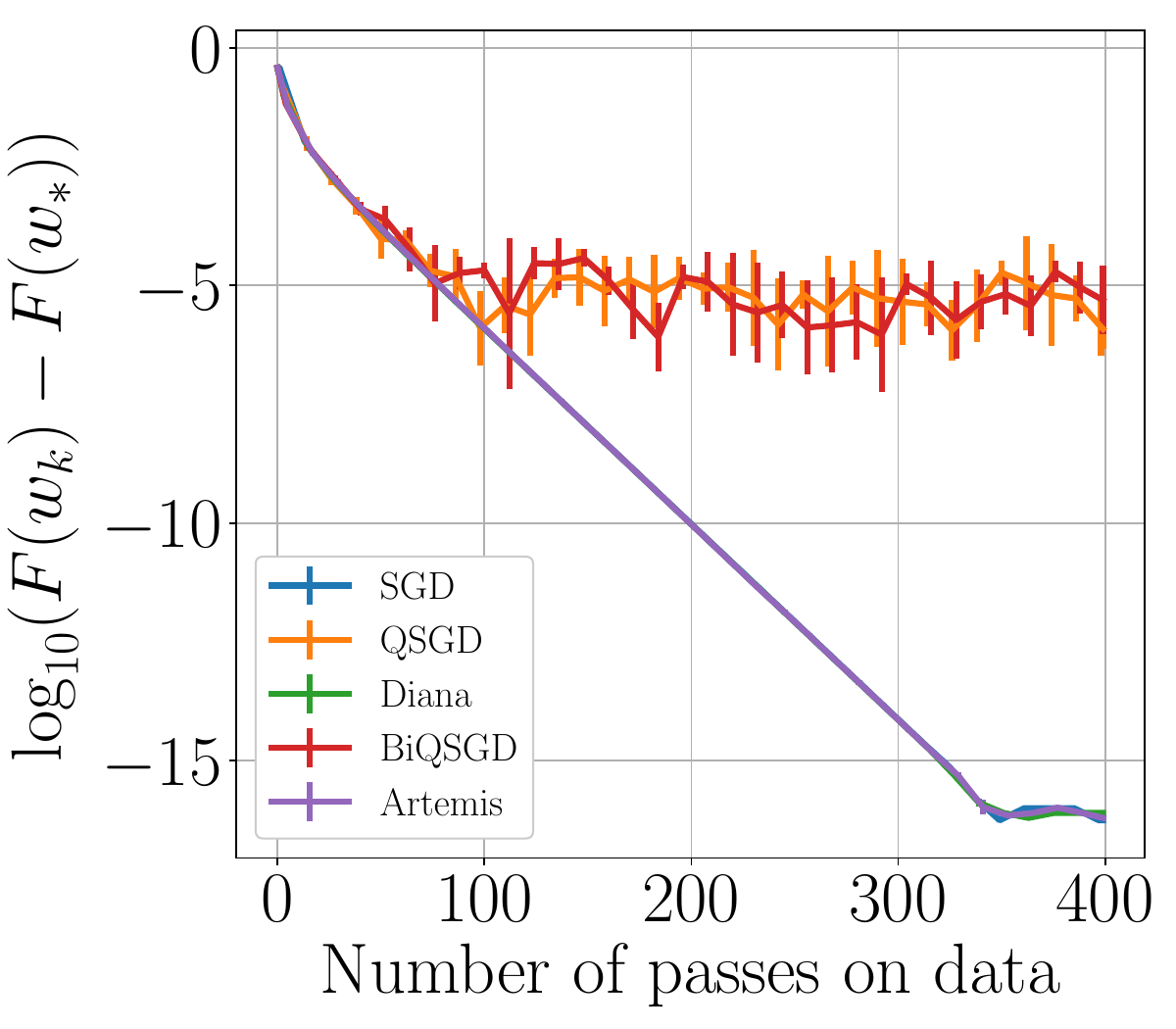}
        \caption{LR (non-i.i.d.): $\sigmstar=0$. \gs}
        \label{fig:deterministic_noniid}
    \end{subfigure}
    \caption[Figures of the main]{Left: illustration of the saturation when $\sigmstar\neq 0$ and data is i.i.d., right: illustration of the memory benefits when $\sigmstar=0$ but with non-i.i.d. data.
    \vspace{-1.1em}}
    \label{fig:main_figures_articles1}
\end{figure}

\textbf{Convergence.} \Cref{fig:LSR_noised} presents the convergence of each  algorithm w.r.t.~the number of iterations $k$. During first iterations all algorithms make fast progress. However, because $\sigmstar^2\neq0$, all algorithms saturate; and the saturation level is higher for double compression (\Artemis, \texttt{Bi-QSGD}), than for simple compression (\texttt{Diana}, \texttt{QSGD}), or than for SGD. This corroborates findings in \Cref{thm:cvgce_artemis} and \Cref{thm:cvdist}.  

\textbf{Complexity.} On \Cref{fig:real_dataset_PP2,fig:real_dataset_PP1,fig:real_dataset}, the loss is plotted w.r.t.~the theoretical number of bits exchanged after $k$ iterations for the \textit{quantum} and \textit{superconduct} dataset. This confirms that double compression should be the method of choice to achieve a reasonable precision (\emph{w.r.t.~$\sigmstar$}), whereas for high precision, a simple method like SGD results \emph{in a lower complexity.} 

\textbf{Linear convergence under \emph{null variance at the optimum}.} To highlight the significance of our new condition on the noise, we compare $\sigmstar^2\neq 0$ and $\sigmstar^2= 0$ on \Cref{fig:main_figures_articles1}. Saturation is observed in \Cref{fig:LSR_noised}, but if we consider a situation in which $\sigmstar^2= 0$, and where the uniform bound on the gradient's variance \emph{is not null} (as opposed to experiments in \citet{liu_double_2020} who consider batch gradient descent), \emph{a linear convergence rate is observed.} This illustrates that our new condition is sufficient to reach a linear convergence. Comparing \Cref{fig:LSR_noised} with \Cref{fig:app:LSR_nonoise} sheds light on the fact that the saturation level (before which double compression is indeed beneficial) is truly proportional to the noise variance \textit{at optimal point} i.e.~$\sigmstar^2$. And when $\sigmstar^2 = 0$, bidirectional compression is much more effective than the other methods (see \Cref{fig:app:without_noise} in \Cref{app:exp:leastsquare}).

\begin{figure}
    \centering
    \begin{subfigure}{0.22 \textwidth}
        \centering
        \includegraphics[width=1\textwidth]{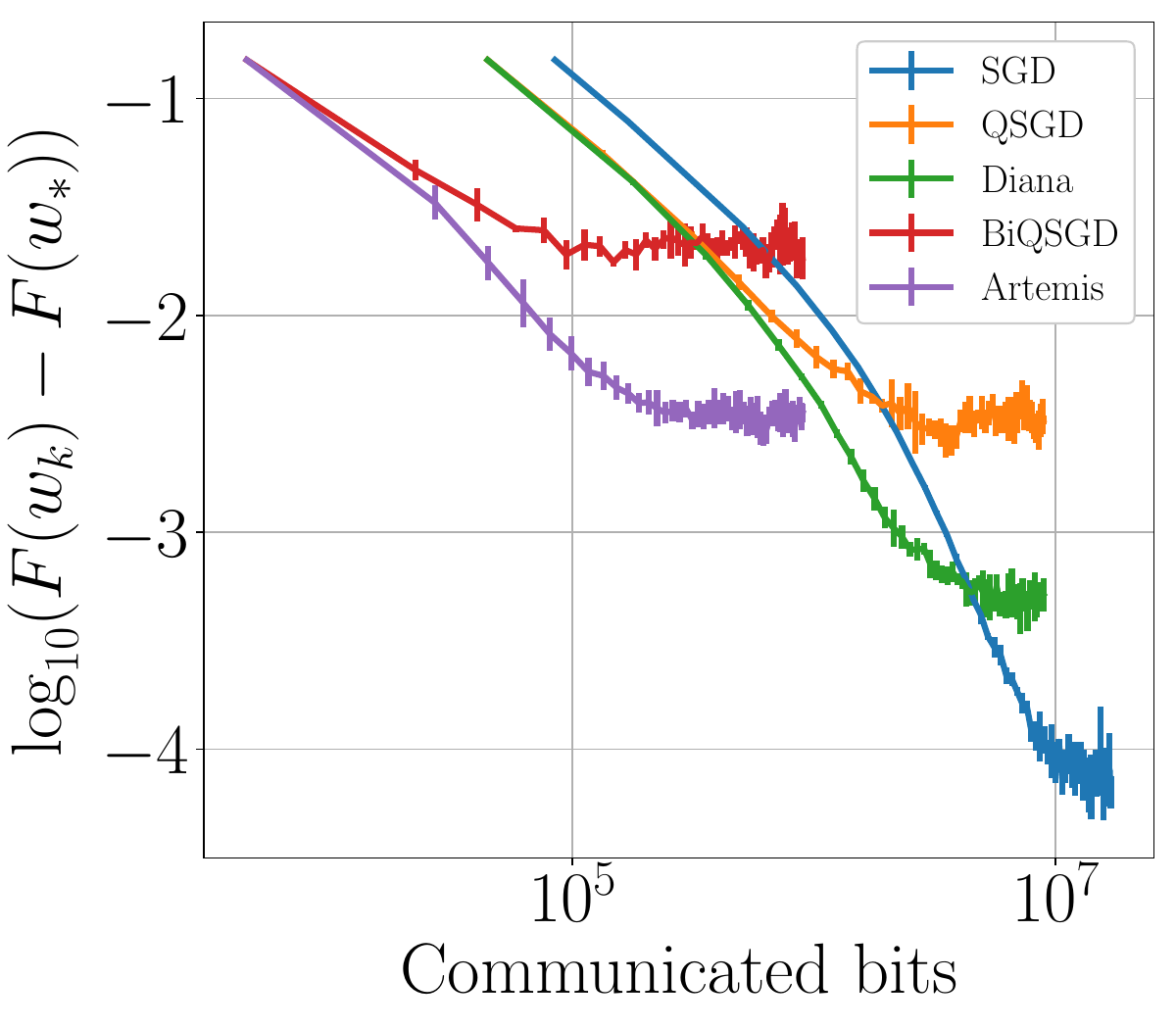}
        \caption{Quantum \gs}
        \label{fig:superconduct}
    \end{subfigure}
    \begin{subfigure}{0.22 \textwidth}
        \centering
        \includegraphics[width=1\textwidth]{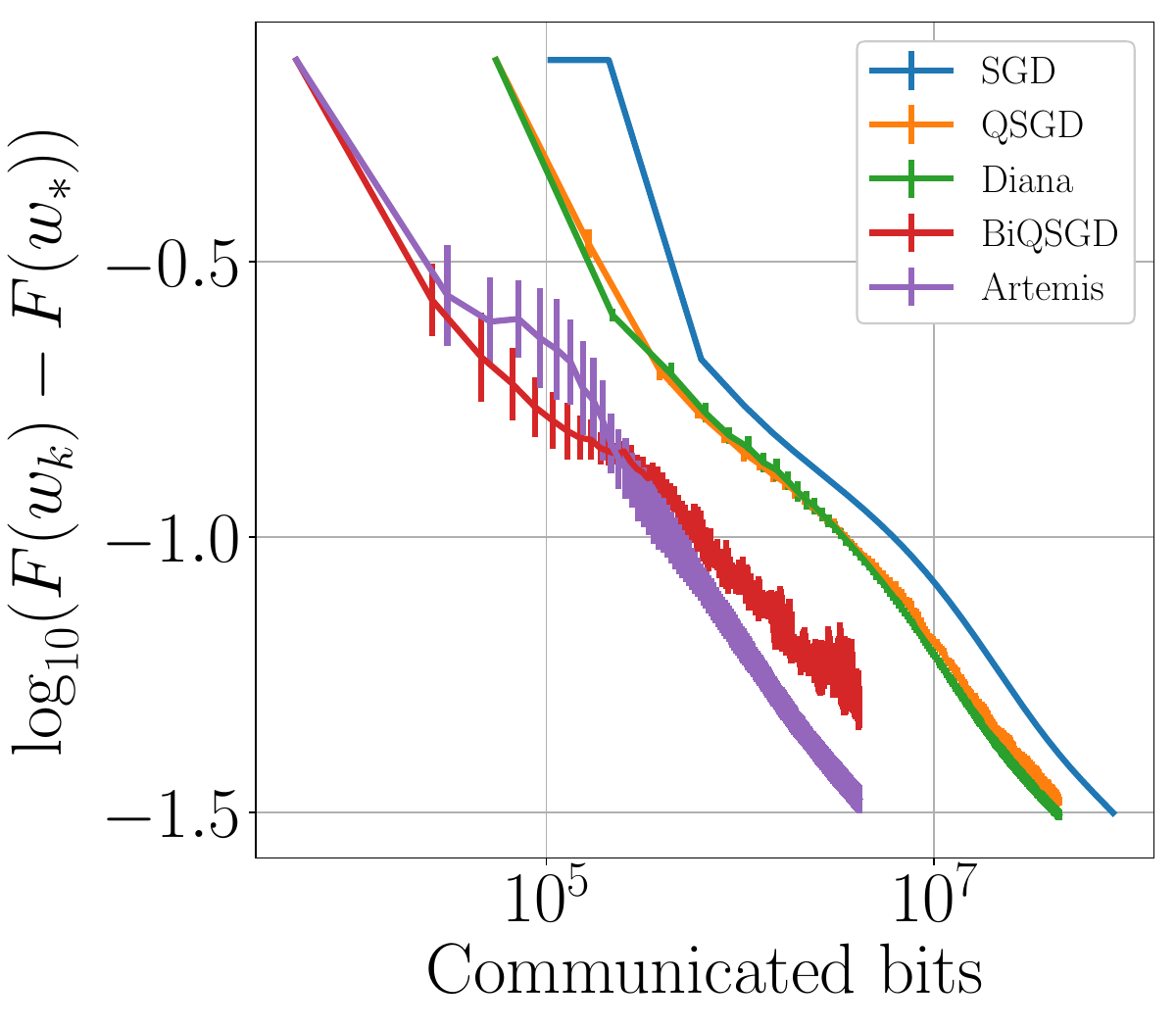}
        \caption{Superconduct \gs}
        \label{fig:quantum}
    \end{subfigure}
    \caption[Real dataset]{\textbf{Real dataset} (non-i.i.d.): $\sigmstar \neq 0$, $N=20$ workers, $p=1$, $b >1$ ($150$ iter.). X-axis in \# bits.
    \vspace{-1.1em}}
    \label{fig:real_dataset}
\end{figure}

\begin{figure}
    \centering
    \begin{subfigure}{0.22 \textwidth}
        \centering
        \includegraphics[width=1\textwidth]{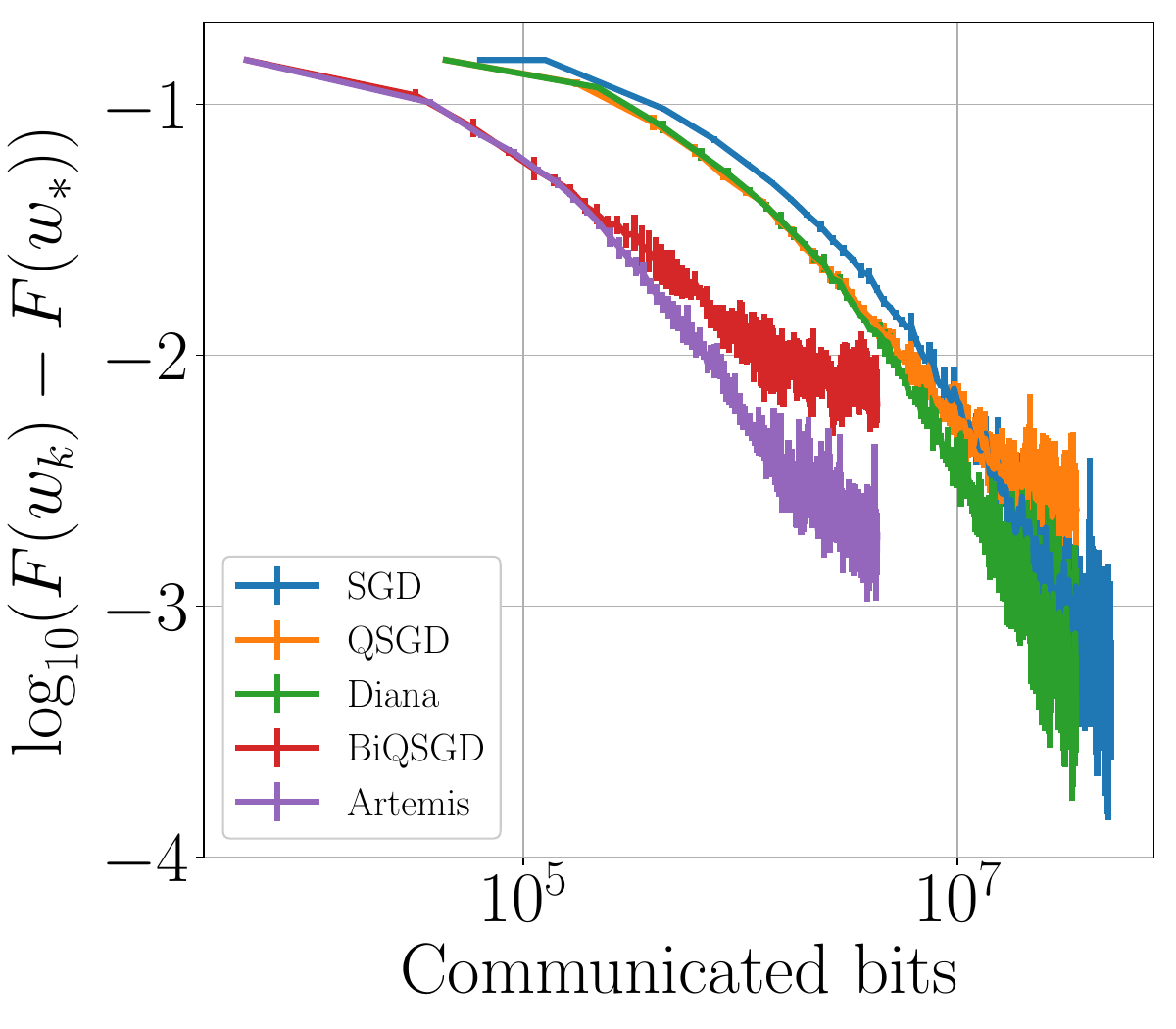}
        \caption{Quantum \gs}
    \end{subfigure}
    \begin{subfigure}{0.22 \textwidth}
        \centering
        \includegraphics[width=1\textwidth]{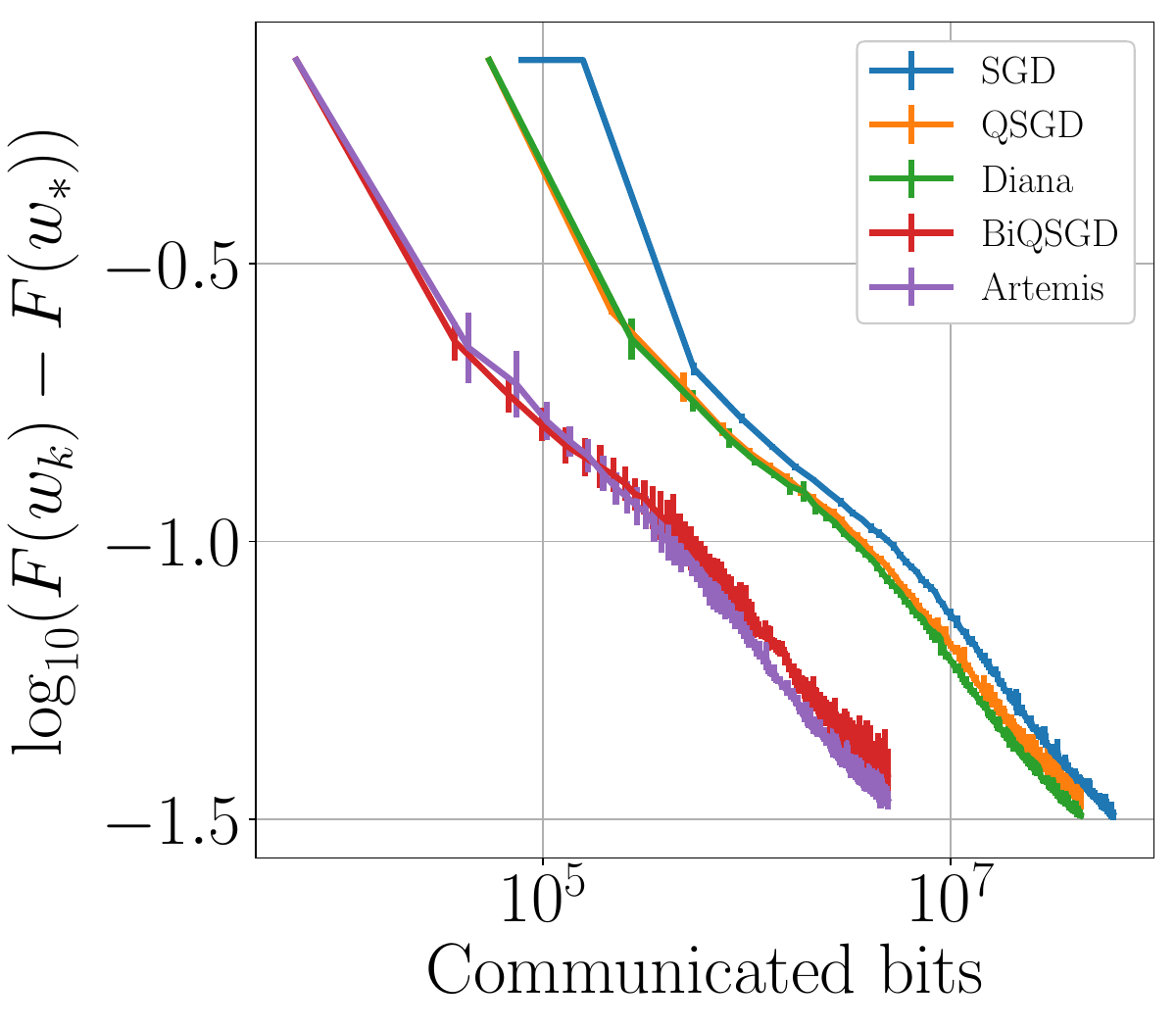}
        \caption{Superconduct \gs}
    \end{subfigure}
    \caption[Real dataset]{\textbf{Partial participation - PP1} (non-i.i.d.): $\sigmstar = 0$ (same experiments in stochastic regime: \Cref{app:fig:real_dataset_PP1:sto}), $N=20$ workers, $p=0.5$,  $b >1$ ($400$ iter.). X-axis in \# bits.
    \vspace{-1.1em}}
    \label{fig:real_dataset_PP1}
\end{figure}

\begin{figure}
    \centering
    \begin{subfigure}{0.22\textwidth}
        \centering
        \includegraphics[width=1\textwidth]{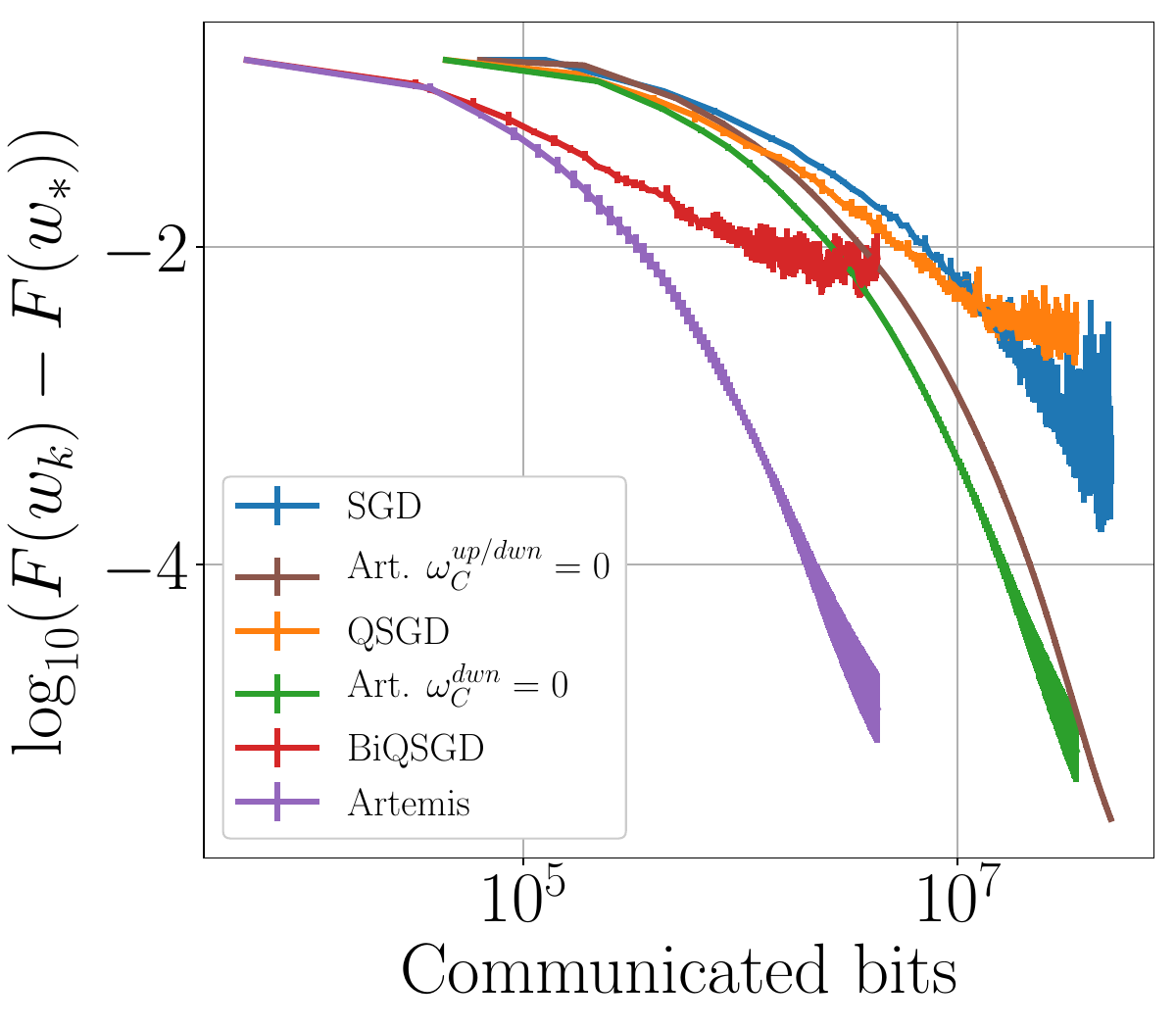}
        \caption{Quantum}
    \end{subfigure}
    \begin{subfigure}{0.22\textwidth}
        \centering
        \includegraphics[width=1\textwidth]{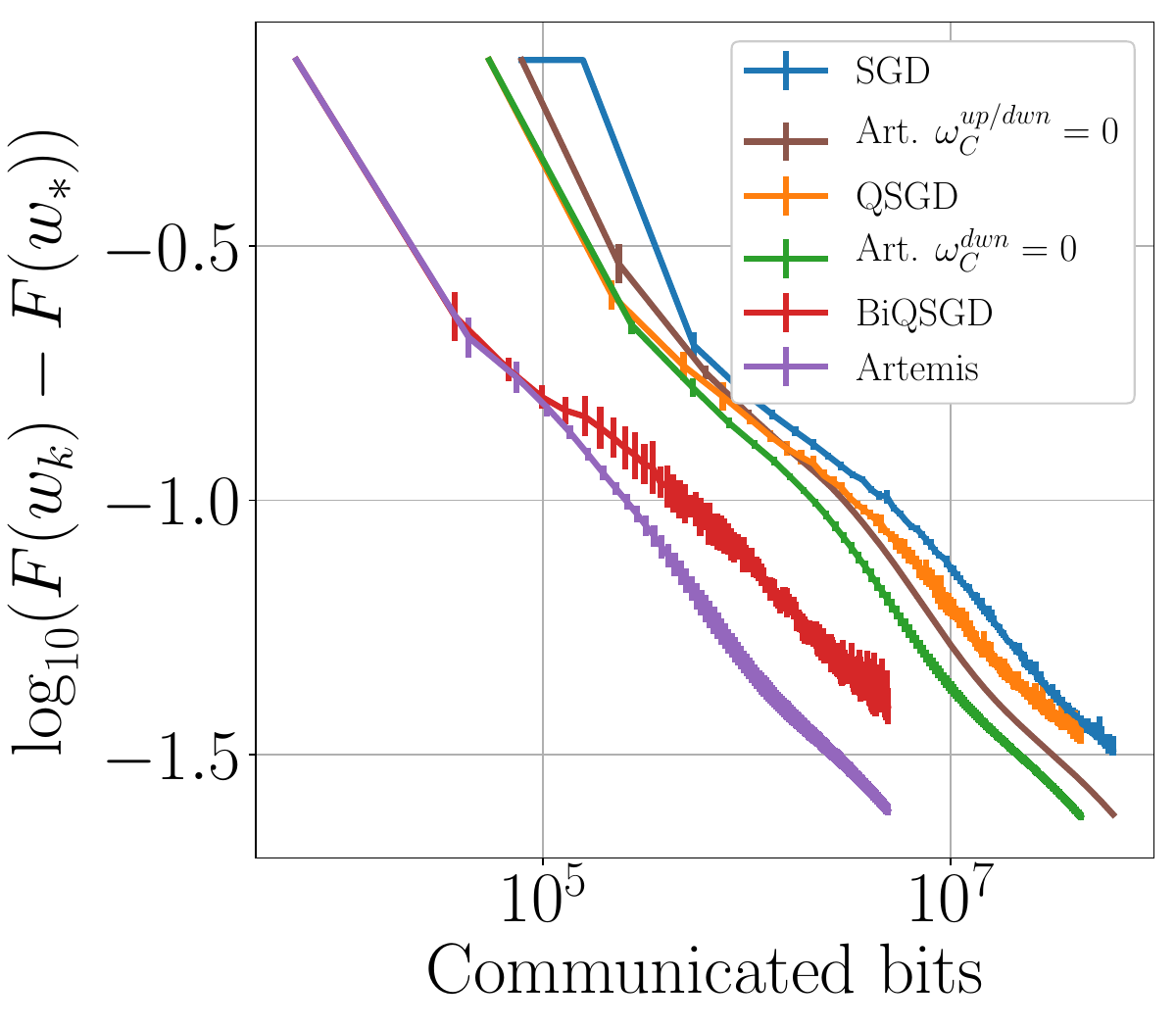}
        \caption{Superconduct}
    \end{subfigure}
    \caption[Real dataset]{\textbf{Partial participation - PP2} (non-i.i.d.): $\sigmstar = 0$ (same experiments in stochastic regime: \Cref{app:fig:real_dataset_PP2:sto}), $N=20$ workers, $p=0.5$,  $b >1$ ($400$ iter.). X-axis in \# bits. 
    \vspace{-1.1em}}
    \label{fig:real_dataset_PP2}
\end{figure}

\textbf{Heterogeneity and real datasets.} 
While in \Cref{fig:LSR_noised}, data is i.i.d.~on machines, and  \Artemis~is thus not expected to outperform \texttt{Bi-QSGD} (the difference between the two being the memory), in \Cref{fig:real_dataset_PP2,fig:real_dataset_PP1,fig:real_dataset,fig:deterministic_noniid} we use \textbf{non-i.i.d.~data}. None of the previous papers on compression directly illustrated the impact of heterogeneity on simple examples, neither compared it with i.i.d.~situations.

\textbf{Partial participation.} On \Cref{fig:real_dataset_PP2,fig:real_dataset_PP1} we run experiments with \textit{only half of the devices active} (randomly sampled) at each iteration with the two approaches (\textbf{PP1} and \textbf{PP2}) described in \Cref{sec:extension_partial_participation} in a full gradient regime (same experiments in stochastic regime are given in \Cref{app:subsubsec:PP}). This two figures \textit{emphasize the failure of the first approach} \textbf{PP1} compared to \textbf{PP2}. We observe that in this last case only, \Artemis~has a linear convergence. It also stresses the key role of the memory in this setting. On \Cref{fig:real_dataset_PP2}, all algorithms with a (single) memory (with or without up/down compression) are better than \texttt{SGD} without memory. Note that \Diana~is defined with multiple memories, thus it cannot be compared to \Artemis~with \textbf{PP2}.

\gs
\gs
\section{Conclusion}\gs\gs
\label{sect:conclusion}

We propose \Artemis, a framework using bidirectional compression to reduce the number of bits needed to perform distributed or federated learning. On top of compression, \Artemis~includes a memory mechanism which improves convergence over non-i.i.d.~data. 
As PP is a classical setting, we designed an approach (\textbf{PP2}) to tackle it while leveraging the full impact of memory, outperforming existing solutions. 
We provide three tight theorems giving guarantees of a fast convergence (linear up to a threshold), highlighting the impact of memory, analyzing Polyak-Ruppert averaging and obtaining lowers bound by studying convergence in distribution of our algorithm. 
Altogether, this improves the understanding of compression combined with a memory mechanism and sheds light on challenges ahead.

\section*{Acknowledgments}
We would like to thank Richard Vidal, Laeticia Kameni from Accenture Labs (Sophia Antipolis, France) and Eric Moulines from École Polytechnique for interesting discussions. This research was supported by the \emph{SCAI: Statistics and Computation for AI} ANR Chair of research and teaching in artificial intelligence and by \emph{Accenture Labs} (Sophia Antipolis, France).

\bibliography{main}

\newpage	
\onecolumn
\appendix
	
\begin{center}
	{\Large{\bf Bidirectional compression in heterogeneous settings for distributed or federated learning with partial participation: tight convergence guarantees. \ \\\ \\ {Supplementary material}}}
	\end{center}
	
	\setcounter{equation}{0}
	\setcounter{figure}{0}
	\setcounter{table}{0}
	\renewcommand{\theequation}{S\arabic{equation}}
	\renewcommand{\thefigure}{S\arabic{figure}}
	\renewcommand{\thetheorem}{S\arabic{theorem}}
	\renewcommand{\thelemma}{S\arabic{lemma}}
	\renewcommand{\theproposition}{S\arabic{proposition}}
	\renewcommand{\thecorollary}{S\arabic{proposition}}
	\renewcommand{\thetable}{S\arabic{table}}
	
	In this appendix, we provide additional details to our work. 
	In \Cref{app:complementary}, we give more details on \Artemis, we describe the $s$-quantization scheme used in our experiments and we define the filtrations used in following demonstrations. Secondly, in \Cref{app:sec:speedtest}, we analyze at a finer level the bandwidth speeds across the world to get a better intuition of the state of the worldwide internet usage. Thirdly, in \Cref{app:experiments}, we present the detailed framework of our experiments and  give  further illustrations to our theorems. In \Cref{app:technical}, we gather a few  technical results and introduce the lemmas required in the proofs of the main results. Those proofs are finally given in \Cref{app:all_proofs}. 
	More precisely,  \Cref{thm:cvgce_artemis} follows from  \Cref{app:thm:with_mem,app:thm:without_mem}, which are proved in \Cref{app:proof_singlecompressnomemory,app:proof_doublecompressnomem}, while \Cref{thm:main_PRave,thm:cvdist} are respectively proved in \Cref{app:doublecompress_avg,app:distrib_convergence}.
	
	\hypersetup{linkcolor = black}
	\setlength\cftparskip{2pt}
	\setlength\cftbeforesecskip{2pt}
	\setlength\cftaftertoctitleskip{3pt}
	\addtocontents{toc}{\protect\setcounter{tocdepth}{2}}
	\setcounter{tocdepth}{1}
	\tableofcontents
	
	\addtocontents{toc}{
    \protect\thispagestyle{empty}} 
    \thispagestyle{empty} 
    
    \hypersetup{linkcolor=blue}

\section{Additional details about the \Artemis~framework}
\label{app:complementary}
The aim of this section is threefold. First, we give the pseudo-code of \Artemis. Secondly, we provide supplementary details about the quantization scheme used in our work. We also explain (based on \cite{alistarh_qsgd_2017}) how quantization combined with \texttt{Elias} code \citep[see][]{elias_universal_1975} reduces the required number of bits to send information. Thirdly, we define the filtrations used in our proofs and give their resulting properties. 

\subsection{Artemis pseudo-code}

We provide the pseudo-code of \Artemis~in \Cref{algo} and for the understanding of \Artemis~implementation, we give a visual illustration of the algorithm in \Cref{tikz_algo}.

\begin{remark}
Remark that we have used in \Cref{algo} the \emph{true value} of $p$ in the update (to get an unbiased estimator of the gradient), but it is obviously possible to use an estimated value $\hat p$ in \Cref{eq:update_schema}: indeed, it is exactly equivalent to multiplying the step size $\gamma$ by a factor $p/\hat{p}$, thus neither changes the practical implementation nor the theoretical analysis.
\end{remark}

\begin{figure}
\begin{minipage}{1\textwidth}
\begin{algorithm}[H]
\caption{\texttt{Artemis} - set $\alpha > 0$ to use memory.}
\label{algo}
\LinesNumberedHidden
\DontPrintSemicolon
  \KwInput{Mini-batch size $b$, learning rates $\alpha, \gamma > 0$, initial model  $w_0 \in \WW$,  operators $\mathcal{C}_{\up}$ and $\mathcal{C}_{\dwn}$, $M_1$ and $M_2$ the sizes of the full/compressed gradients.}
  \KwInitialization{Local memory: $\forall i \in \llbracket 1, N \rrbracket$ $h_0^i = 0$ (kept on both central server and device $i$).
  Index of last participation: $k_i = 0$.}
  \KwOutput{Model  $w_K$}
  \For{$k = 0, 1, 2, \dots, K $}
    {
    Randomly sample a set of device $S_k$ \;
        \For{each device $i \in S_k$}
            {
                \textbf{Catching up.} \;
                If $k - k_i$ > $\floor{M_1 / M_2}$, send the model $w_k$ \;
                Else send $(\widehat{\Omega}_j)_{j=k_i + 1}^k$ and update local model: $\forall j \in \llbracket k_i + 1, k \rrbracket, \quad w_{j} = w_{j-1} - \gamma \Omega_{j, S_{j-1}}$ \;
                Update index of its last participation: $k_i = k$ \;
                \textbf{Local training.} \;
                Compute stochastic gradient $\gwk^i = \g_{k+1}(w_k)$ (with mini-batch)\;
                Set $\Delta_k^i = \gwk^i - h_k^i$, compress it  $\widehat{\Delta}_k^i = \mathcal{C}_{\up}(\Delta_k^i)$ \;
                Update memory term: $h_{k+1}^i = h_k^i + \alpha \widehat{\Delta}_k^i$ \;
                Send $\widehat{\Delta}_k^i$ to central server \;
            }
            Compute $\gwkhats = h_k + \ffrac{1}{pN} \sum_{i \in S_k} \widehat{\Delta}_k^i$ \;
            Update central memory: $h_{k+1} = h_k + \alpha \ffrac{1}{N} \sum_{i \in S_k} \widehat{\Delta}_k^i$ \; 
            Back compression: $\Omega_{k+1, S_k} = \mathcal{C}_{\dwn}(\gwkhats)$\;
            Broadcast $\Omega_{k+1}$ to all workers.\;
            Update model on central server: $w_{k+1} = w_k - \gamma \Omega_{k+1, S_k}$\;
    }
\end{algorithm}
      
\end{minipage}
\end{figure}

\begin{figure}[H]
\centering
\includegraphics[width=\textwidth]{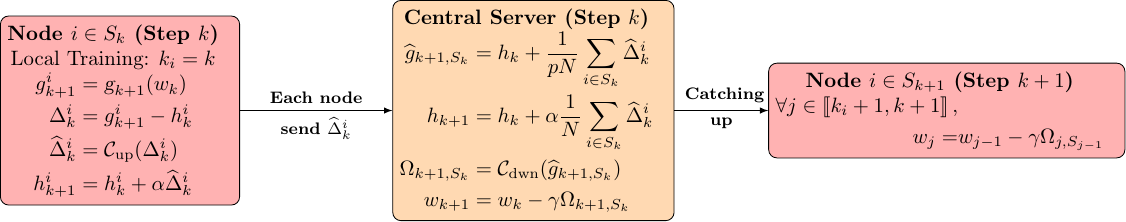}
\caption{Visual illustration of \Artemis~bidirectional compression at iteration $k \in \N$. $w_k$ is the model's parameter, $\mathcal{C}_{\up}$ and $\mathcal{C}_{\dwn}$ are the compression operators, $\gamma$ is the step size, $\alpha$ is the learning rate which simulates a memory mechanism of past iterates and allows the compressed values to tend to zero.}
\label{tikz_algo}
\end{figure}

\subsection{Quantization scheme}
\label{app:quantization_schema}

In the following, we define the $s$-quantization operator $\mathcal{C}_s$ which we use in our experiments. After giving its definition, we explain \cite[based on][]{alistarh_qsgd_2017} how it helps to reduce the number of bits to broadcast. 

\begin{definition}[$s$-quantization operator]
\label{def:s-quantization}
Given $\Delta \in \WW$, the $s$-quantization operator $\mathcal{C}_s$ is defined by:
\[ \mathcal{C}_s(\Delta):= sign(\Delta) \times \| \Delta \|_2 \times \frac{\psi}{s}\,.
\]

$\psi \in \WW$ is a random vector with $j$-th element defined as:

\[ 
\psi_j:=
\left\{
    \begin{array}{ll}
        l + 1 & \mbox{with probability } s \frac{|\Delta_j|}{\|\Delta\|_2}  - l \\
        l & \mbox{otherwise}\,.
    \end{array}
\right.
\]

where the level $l$ is such that $\ffrac{\Delta_i}{\lnrm \Delta \rnrm_2} \in \left[\ffrac{l}{s}, \ffrac{l+1}{s} \right]$.
\end{definition}

The $s$-quantization scheme verifies \Cref{asu:expec_quantization_operator} with $\omgC= \min(d/s^2, \sqrt{d}/s)$. Proof can be found in \cite[][see Appendix A.1]{alistarh_qsgd_2017}.

Now, for any vector $v \in \mathbb{R}^d$, we are in possession of the tuple $(\|v\|^2,\phi , \psi)$, where $\phi$ is the vector of signs of $(v_i)_{i=1}^d$, and $\psi$ is the vector of integer values $(\psi_j)_{j=1}$. To broadcast the quantized value, we use the \texttt{Elias} encoding \cite{elias_universal_1975}. Using this encoding scheme, it can be shown (Theorem 3.2 of \cite{alistarh_qsgd_2017}) that:

\begin{proposition}
For any vector $v$, the number of bits needed to communicate $\mathcal{C}_s(v)$ is upper bounded by:
\[
\left( 3 + \left( \frac{3}{2} + o(1) \right)\log \left( \frac{2(s^2 + d)}{s(s + \sqrt{d})} \right) \right) s(s + \sqrt{d})  + 32 \,.
\]
\end{proposition}

The final goal of using memory for compression is to quantize vectors with $s=1$. It means that we will employ $O (\sqrt{d} \log d)$ bits per iteration instead of $32d$, which reduces by a factor $\frac{\sqrt{d}}{\log d}$ the number of bits used by iteration. Now, in a FL settings, at each iteration we have a double communication (device to the main server, main server to the device) for each of the $N$ devices. It means that at each iteration, we need to communicate $2 \times N \times 32 d$ bits if compression is not used. With a single compression process like in \cite{mishchenko_distributed_2019,li_acceleration_2020,wu_error_2018,agarwal_cpsgd_2018,alistarh_qsgd_2017}, we need to broadcast 
\begin{align*}
    O \left( 32Nd + N \sqrt{d} \log d \right) &= O \left(N d \left(1 + \frac{\log d}{\sqrt{d}}\right) \right) \\
    &= O \left( Nd \right)\,.
\end{align*}
But with  a bidirectional compression, we only need to broadcast $O \left(2 N \sqrt{d} \log d \right)$. 

\paragraph{Time complexity analysis of simple vs double compression for the $1$-quantization schema.} 
Using quantization with $s=1$, and then the \texttt{Elias} code \citep[defined in][]{elias_universal_1975} to communicate between servers, leads to reduce from $O(Nd)$ to $O(N\sqrt{d} \log(d))$ the number of bits to send, for each direction. 
Getting an estimation of the total time complexity is difficult and inevitably dependant of the considered application. 
Indeed, as highlighted by \Cref{fig:app:speedtest}, download and upload speed are always different. The biggest measured difference between upload and download is found in Europa for mobile broadband ; their ratio is around $3.5$.

Denoting $v_d$ and $v_u$ the speed of download and upload (in bits per second), we typically have $v_d = \rho v_u$,  $3.5 > \rho > 1$.

Then for unidirectional compression, each iteration takes 
$
O \left(\frac{Nd}{v_d} + \frac{N \sqrt{d} \log(d)}{v_u} \right) \approx O \left(\frac{Nd}{\rho v_u} \right)
$
seconds,  while for a bidirectional one it takes only 
$
O \left( \frac{N \sqrt{d} \log(d)}{v_d} + \frac{ N \sqrt{d} \log(d)}{ v_u} \right) \approx O \left( \frac{N \sqrt{d} \log(d)}{v_u} \right)$ seconds.  

In other words, unless $\rho$ is really large (which is not the case in practice as stressed by \Cref{fig:app:speedtest}, double compression reduces by several orders of magnitude the global time complexity, and  bidirectional compression is by far superior to unidirectional.

\subsection{Filtrations}
\label{sect:filration}

In this section we provide some explanations about filtrations - especially a rigorous definition - and how it is used in the proofs of \Cref{thm:main_PRave,thm:cvdist,app:thm:without_mem,app:thm:with_mem}. We recall that we denoted by $\omgC^\up$ and $\omgC^\dwn$ the variance factors for respectively uplink and downlink compression. 

Let a probability space $(\Omega, \mathcal{A}, \mathbb{P})$ with $\Omega$ a sample space, $\mathcal{A}$ an event space, and $\mathbb{P}$ a probability function. We recall that the $\sigma$-algebra generated by a random variable $X: \Omega \rightarrow \mathbb{R}^m$ is 
\[
    \sigma(X) = \{ X^{-1}(A): A \in \mathcal{B}(\mathbb{R}^m) \} \,,
\]
where $\mathcal{B}(\mathbb{R}^m)$ is the Borel set of $\mathbb{R}^m$.

Furthermore, we recall that a filtration of $(\Omega, \Fsto, P)$ is defined as an increasing sequence $(\Fsto_n)_{n \in \N}$ of $\sigma$-algebras:
\[
    \Fsto_0 \subset \Fsto_1 \subset \Fsto_2 \subset \dots \subset \Fsto \,.
\]

Concerning randomness in our algorithm, it comes from four sources:

\begin{enumerate}
    \item Stochastic gradients. It corresponds to the noise associated with the stochastic gradients computation on device $i$ at epoch $k$. We have:
    \begin{equation*}
        \forall k \in \N\,,~\forall i \in \llbracket 0, ..., N \rrbracket, \quad\gwk^i = \nabla F_i (w_k) + \xi_{k+1}^i(w_k)\,, \text{ with $\V(\xi_{k+1}^i)$ bounded.}
    \end{equation*}
    
    \item Uplink compression: this noise corresponds to the uplink compression when local gradients are compressed.
    Let $k \in \N$ and $i \in \llbracket 0, ..., N \rrbracket$, suppose, we want to compress $\Delta_k^i \in \WW$, then the associated noise is $\epsilon_{k+1}^i$ with $\V(\epsilon_{k+1}^i(\Delta_k^i)) \leq \omgC^\up \SqrdNrm{\Delta_k^i}$, where $\omgC^\up \in \R^*$ is defined by the uplink compression schema (see \Cref{asu:expec_quantization_operator}).
    And it follows that:
    \begin{equation*}
        \forall k \in \N, ~\forall i \in \llbracket 0, ..., N \rrbracket, \quad\widehat{\Delta}_{k}^i = \Delta_{k}^i + \epsilon_{k}^i(\Delta_k^i) \Longleftrightarrow \gwkhat^i = \gwk^i + \epsilon_{k+1}^i(\Delta_k^i)\,.
    \end{equation*}
    
    \item Downlink compression. This noise corresponds to the downlink compression, when the global model parameter is compressed. Let $k \in \N$, suppose we want to compress $\gwkhats \in \WW$, then the associated noise is $\epsilon_{k+1}(\gwkhats)$ with $\V(\epsilon_{k+1}) \leq \omgC^\dwn \SqrdNrm{\gwkhats}$. There is:
    \begin{equation*}
            \forall k \in \N,\quad \Omega_{k+1, S_k} = \mathcal{C}_s(\gwkhats) = \gwkhats + \epsilon_{k+1}(\gwkhats)\,.
    \end{equation*}
    
    \item Random sampling. This randomness corresponds to the partial participation of each device. We recall that according to \Cref{asu:partial_participation}, each device has a probability $p$ of being active. For $k$ in $\N$, for $i$ in $\llbracket 1, N \rrbracket$, we note $B_k^i \sim \mathcal{B}(p)$ the Bernoulli random variable that marks if a device is active or not at step $k$.
    
\end{enumerate}

\begin{figure*}
    \centering
    $\mathlarger{
    w_k \xrightarrow{\quad \mathlarger{\xi_{k+1}^i} \quad}
    \gwk^i  \xrightarrow{\quad \mathlarger{\epsilon_{k+1}^i} \quad}
    \gwkhat^i\xrightarrow{\quad \mathlarger{B_k^i} \quad} \gwkhats = \sum_{i \in S_k} \gwkhat^i \xrightarrow{\quad \mathlarger{\epsilon_{k+1}} \quad} \Omega_{k+1, S_k} = \mathcal{C} \left( \gwkhats \right)}
    $
    \caption{The sequence of successive additive noises in the algorithm.}
    \label{fig:sequence_of_noise}
\end{figure*}

This ``succession of noises'' in the algorithm is illustrated in \Cref{fig:sequence_of_noise}. In order to handle these four sources of randomness, we define five sequences of nested $\sigma$-algebras. 

\begin{definition}
We note $(\Fsto_k)_{k \in \N}$ the filtration associated to the stochastic gradient computation noise, $(\Fupcomp_k)_{k \in \N}$ the filtration associated to the uplink compression noise and $(\Fdwncomp_k)_{k \in \N}$ the filtration associated to the downlink compression noise, $(\Fsamp_k)_{k \in \N}$ the filtration associated to the random device participation randomness. For $k \in \N^*$, we define:
\begin{align*}
    \Fsto_k &= \sigma \left(\Gamma_{k-1}, (\xi_{k}^i)_{i=1}^N \right) \\
    \Fupcomp_k &= \sigma \left(\Gamma_{k-1}, (\xi_{k}^i)_{i=1}^N, (\epsilon_{k}^i)_{i=1}^N \right) \\
    \Fdwncomp_k &= \sigma \left(\Gamma_{k-1}, (\xi_{k}^i)_{i=1}^N, (\epsilon_{k}^i)_{i=1}^N, \epsilon_{k} \right) \\
    \Fsamp_{k-1} &= \sigma \left( (B_{k-1}^i)_{i=1}^N \right) \\
    \Flast_k &= \sigma \left(\Fdwncomp_k \cup \Fsamp_{k-1} \right)\, , 
\end{align*}
with
\[
\Gamma_k = \{(\xi_t^i)_{i \in \llbracket 1, N \rrbracket}, (\epsilon_t^i)_{i \in \llbracket 1, N \rrbracket}, \epsilon_t, (B_{k-1}^i)_{i=1}^N\}_{t \in \llbracket 1, k \rrbracket} \text{\quad and \quad} \Gamma_{0} = \{ \varnothing \}\,.
\]

\end{definition}

We can make the following observations for all $k \geq 1$:
\begin{itemize}
    \item From these three definitions, it follows that our sequences are nested.
\[
\Fsto_1 \subset \Fupcomp_1 \subset \Fdwncomp_1 \subset \Fsto_2 \subset \dots \subset \Fdwncomp_K\,.
\]
However, $(\Fsamp_k)_{k \in \N}$ is independent of the other filtrations. 
\item $\Flast_{k} = \sigma \left(\Fdwncomp_k \cup \Fsamp_{k-1} \right) = \sigma(\Gamma_{k})$, and the aim is to express the expectation w.r.t.~all randomness i.e $\Ftotal$.
\item $w_k$ is $\Flast_{k}$-measurable.
\item $\gwk(w_k)$ is $\Fsto_{k+1}$-measurable. 
\item $\gwkhat(w_k)$ is $\Fupcomp_{k+1}$-measurable.
\item $B_{k-1}^i$ is $\Fsamp_{k-1}$-measurable.
\item $\gwks\,, \gwkhats$ and $\Omega_{k+1, S_k}$ are respectively $\sigma(\Fsto_{k+1} \cup \Fsamp_k)$-measurable, $\sigma(\Fupcomp_{k+1} \cup \Fsamp_k)$-measurable and $\sigma(\Fdwncomp_{k+1} \cup \Fsamp_k)$-measurable. Note that $\Fsto_{k+1}$ contains $\Gamma_k$, and thus all $(B_{k-1}^i)^N_{i=1}$, but does not contain all the $(B_{k}^i)^N_{i=1}$.
\end{itemize}

As a consequence, we have \Cref{prop:stochastic_expectation,prop:uplink_expectation,prop:variance_uplink,prop:downlink_expectation,prop:variance_downlink,prop:exp_samp,prop:variance_samp}. Please, take notice that for sake of clarity \Cref{prop:stochastic_expectation,prop:uplink_expectation,prop:variance_uplink,prop:downlink_expectation,prop:variance_downlink} are stated without taking into account the random participation $S_k$. All this proposition remains identical when adding partial participation, the results only have to be expressed w.r.t.~active nodes $S_k$. 

Below \Cref{prop:stochastic_expectation} gives the expectation over stochastic gradients conditionally to $\sigma$-algebras $\Fdwncomp_k$ and $\Fsto_{k+1}$.

\begin{proposition}[Stochastic Expectation]
\label{prop:stochastic_expectation}
Let $k \in \N$ and $i \in \llbracket 1, N \rrbracket$. Then on each local device $i \in \llbracket 1, N \rrbracket$ we have almost surely (a.s.):
\[
\left\{
    \begin{array}{ll}
    	\Expec{\gwk^i}{\Fsto_{k+1}} &= \gwk^i \\ 
    	\Expec{\gwk^i}{\Fdwncomp_{k} } &= \nabla F_i(w_k) \,,
    \end{array}
\right.
\]
which leads to:
\[\left\{
    \begin{array}{ll}
    	\Expec{\gwk}{\Fsto_{k+1}} &= \gwk \\
    	\Expec{\gwk}{\Fdwncomp_{k}} &= \nabla F(w_k) \,.
    \end{array}
\right.
\]
\end{proposition}

\Cref{prop:uplink_expectation} gives expectation of uplink compression (information sent from remote devices to central server) conditionally to $\sigma$-algebras $\Fsto_{k+1}$ and $\Fupcomp_{k+1}$.

\begin{proposition}[Uplink Compression Expectation]
\label{prop:uplink_expectation}
Let $k \in \N$ and $i \in \llbracket 1, N \rrbracket$. Recall that $\widehat{g}_k^i = g_k^i + \epsilon_k^i$, then on each local device $i \in \llbracket 1, N \rrbracket$, we have a.s:

\[\left\{
    \begin{array}{ll}
    	\Expec{ \gwkhat^i}{\Fupcomp_{k+1}} = \gwkhat^i \\
    	\Expec{ \gwkhat^i}{\Fsto_{k+1}} = \gwk^i  \,,
    \end{array}
\right.
\]

which leads to 
\[\left\{
    \begin{array}{l}
    	\Expec{\gwkhat}{\Fupcomp_{k+1}} = \gwkhat \\
    	\Expec{\gwkhat}{\Fsto_{k+1}} = \Expec{\frac{1}{N} \sum_{i=1}^N\gwkhat^i}{\Fsto_{k+1}} = \gwk  \,. \\
    \end{array}
\right.
\]
\end{proposition}

From \Cref{asu:expec_quantization_operator}, if follows that variance over uplink compression can be bounded as expressed in \Cref{prop:variance_uplink}. 
\begin{proposition}[Uplink Compression Variance]
\label{prop:variance_uplink}
Let $k \in \N$ and $i \in \llbracket 1, N \rrbracket$. Recall that $\Delta_k^i = \gwk^i + h_k^i$, using \Cref{asu:expec_quantization_operator} following hold a.s:
\begin{align}
    &\Expec{ \|\widehat{\Delta}_{k+1}^i - \Delta_{k+1}^i \|^2}{\Fsto_{k+1}} \leq \omgC^\up \| \Delta_{k+1}^i \|^2 \\
(\Longleftrightarrow\qquad &\Expec{ \|\gwkhat^i - \gwk^i \|^2}{\Fsto_{k+1}} \leq \omgC^\up \| \gwk^i \|^2  \text{ when no memory } )\,.
\end{align}
\end{proposition}

Concerning downlink compression (information sent from central server to each node), \Cref{prop:downlink_expectation} gives its expectation w.r.t $\sigma$-algebras $\Fupcomp_{k+1}$ and $\Fdwncomp_{k+1}$.
\begin{proposition}[Downlink Compression Expectation]
\label{prop:downlink_expectation}
Let $k \in \N$, recall that $\Omega_{k+1} = \mathcal{C}_\dwn(\gwkhat) = \gwkhat + \epsilon_k$, then a.s:
\[\left\{
    \begin{array}{ll}
    	\Expec{ \Omega_{k+1}}{\Fdwncomp_{k+1}} = \Omega_{k+1} \\
    	\Expec{ \Omega_{k+1}}{\Fupcomp_{k+1}} = \gwkhat  \,.
    \end{array}
\right.
\]
\end{proposition}

Next proposition states that downlink compression can be bounded as for \cref{prop:variance_uplink}.

\begin{proposition}[Downlink Compression Variance]
\label{prop:variance_downlink}
Let $k \in \N$, using \Cref{asu:expec_quantization_operator} following holds a.s:
\[\Expec{ \|\Omega_{k+1} - \gwkhat \|^2}{\Fupcomp_{k+1}} \leq \omgC^\dwn \| \gwkhat \|^2  \,.
\]
\end{proposition}

Now, we give in \Cref{prop:variance_samp,prop:exp_samp} the expectation and the variance w.r.t.~devices random sampling noise.

\begin{proposition}[Expectation of device sampling]
\label{prop:exp_samp}
Let $k \in \N$, let's note $a_{k+1} = \frac{1}{N} \sum_{i=1}^N a_{k+1}^i$ and $a_{k+1, S_k} = \frac{1}{pN} \sum_{i \in S_k} a_{k+1}^i$, where $(a_{k+1}^i)_{i=0}^N \in (\WW)^N$ are $N$ random variables independent of each other and $\Fartif_{k+1}$-measurable, for a $\sigma$-field $\Fartif_{k+1}$ s.t.~$(B_{k}^i)^N_{i=1}$ are independent of $\Fartif_{k+1}$. We have a.s:
\[
\Expec{a_{k+1, S_k}}{\Fartif_{k+1}} = a_{k+1} \,.
\]
\end{proposition}
The vector $a_{k+1}$ (resp. the $\sigma$-field $\Fartif_{k+1}$) may represent various objects, for instance : $\gwk$, $\gwkhat$, $\Omega_{k+1}$ (resp. $\Fsto_{k+1}$, $\Fupcomp_{k+1}$, $\Fdwncomp_{k+1}$).
\begin{proof}
For any $k\in \N^*$, we have that:
\begin{align*}
    \Expec{a_{k+1, S_k}}{\Fartif_{k+1}} &= \Expec{\frac{1}{pN} \sum_{i \in S_k} a_{k+1}^i}{\Fartif_{k+1}} = \Expec{\frac{1}{pN} \sum_{i=0}^N a_{k+1}^i B^i_k}{\Fartif_{k+1}} \\
    &= \frac{1}{pN} \sum_{i=0}^N \Expec{a_{k+1}^i B_k^i}{\Fartif_{k+1}}  \text{ by linearity of the expectation,}\\
    &= \frac{1}{pN} \sum_{i=0}^N a_{k+1}^i \Expec{B_k^i}{\Fartif_{k+1}} \text{because $(a_{k+1}^i)^N_{i=1}$ are $\Fartif_{k+1}$-measurable,} \\
    &= \frac{1}{N} \sum_{i=0}^N a_{k+1}^i = a_{k+1}\, \text{because $(B_{k}^i)^N_{i=1}$ are independent of $\Fartif_{k+1}$,} 
\end{align*}
which allows to conclude.
\end{proof}

\begin{proposition}[Variance of device sampling]
\label{prop:variance_samp}
Let $k \in \N^*$, with the same notation as \Cref{prop:exp_samp}, we have a.s:
\[
\Var{a_{k+1, S_k}}{\Fartif_{k+1}} = \ffrac{1-p}{p N^2} \sum_{i=0}^N \SqrdNrm{a_{k+1}^i} \,.
\]
\end{proposition}

\begin{proof}
Let $k \in \N^*$,
\begin{align*}
    \Var{a_{k+1, S_k}}{\Fartif_{k+1}} &= \Var{\frac{1}{pN} \sum_{i=0}^N a_{k+1}^i B^i_k}{\Fartif_{k+1}} \\
    &= \frac{1}{p^2N^2} \sum_{i=0}^N \Var{a_{k+1}^iB_k^i}{\Fartif_{k+1}} \text{because  $(B_{k}^i)^N_{i=1}$ are independent,} \\
    &= \frac{1}{p^2 N^2} \sum_{i=0}^N \SqrdNrm{a_{k+1}^i} \Var{B_k^i}{\Fartif_{k+1}} \text{because $(a_{k+1}^i)^N_{i=1}$ are $\Fartif_{k+1}$-measurable,}\\
    &= \ffrac{1-p}{p N^2} \sum_{i=0}^N \SqrdNrm{a_{k+1}^i} \,\text{because $(B_{k}^i)^N_{i=1}$ are independent of $\Fartif_{k+1}$.} 
\end{align*}
\end{proof}

\section{Bandwidth speed}
\label{app:sec:speedtest}

In a network configuration where download would be much faster than upload, bidirectional compression would present no benefit over unidirectional, as downlink communications would have a negligible cost. However, this is not the case in practice:  to assess this point, we gathered  broadband speeds, for both download and upload communications, for fixed broadband (cable, T1, DSL ...) or mobile (cellphones, smartphones, tablets, laptops ...) from studies carried out in $2020$ over the $6$ continents by \href{https://www.speedtest.net/global-index}{\textit{Speedtest.net}} \cite[see][]{noauthor_speedtest_nodate}. Results are provided in  \Cref{fig:app:speedtest}, comparing download and upload speeds. The ratios (averaged by continents) between upload and download speeds stand between $1$ (in Asia, for fixed broadband) and  $3.5$ (in Europe, for mobile broadband): there is thus no apparent reason to simply disregard the downlink communication, and bi-directionnal compression is unavoidable to achieve substantial speedup.
More precisely, if we denote $v_d$ and $v_u$ the speed of download and upload (in Mbits per second), we typically have $v_d = \rho v_u$, with  $1< \rho < 3.5$. Using quantization with $s=1$ (see  \Cref{app:quantization_schema}), for unidirectional compression, each iteration takes $O \left(\frac{Nd}{\rho v_u} \right)$ seconds,  while for a bidirectional one it takes only $O \left( \frac{N \sqrt{d} \log(d)}{v_u} \right)$ seconds.

The dataset is pickled from a study carried out by \href{https://www.speedtest.net/global-index}{\textit{Speedtest.net}} \cite[see][]{noauthor_speedtest_nodate}. 
This study has measured the bandwidth speeds in 2020 accross the six continents. In order to get a better understanding of this dataset, we illustrate the speeds distribution on \Cref{fig:app:speedtest,fig:ratio_boxplot,fig:ratio_mobile,fig:ratio_fixed}. 

\begin{figure}
    \centering
\includegraphics[width=0.5\textwidth]{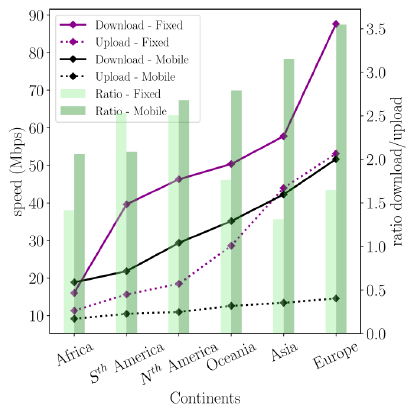}
    \caption{Left axis: upload and download speed for mobile and fixed broadband.  Left axis: speeds (in Mbps), right axis: ratio (green bars). The dataset is gathered from \href{https://www.speedtest.net/global-index}{\textit{Speedtest.net}}, see \cite{noauthor_speedtest_nodate}.\gs\gs}
    \label{fig:app:speedtest}
\end{figure}

In \Cref{fig:ratio_boxplot,fig:ratio_mobile,fig:ratio_fixed}, unlike \Cref{fig:app:speedtest}, we do not aggregate data by countries of a same continents. 
This allows to analyse the speeds ratio between upload and download with the \textit{proper} value of each countries. Looking at \Cref{fig:ratio_mobile,fig:ratio_fixed,fig:ratio_boxplot}, it is noticeable that in the world, the ratio between upload and download speed is between $1$ and $5$, and not between $1$ and $3.5$ as \Cref{fig:app:speedtest} was suggesting since we were aggregating data by continents. 
There are only nine countries in the world having a ratio higher than $5$. 
In Europe : Malta, Belgium and Montenegro. 
In Asia : South Korea. 
In North America : Canada, Saint Vincent and the Grenadines, Panama and Costa Rica. 
In Africa : Western Sahara. The highest ratio is $7.7$ observed in Malta.

\begin{figure}
\centering
    \begin{subfigure}{0.47\textwidth}
        \centering
        \includegraphics[width=1\textwidth]{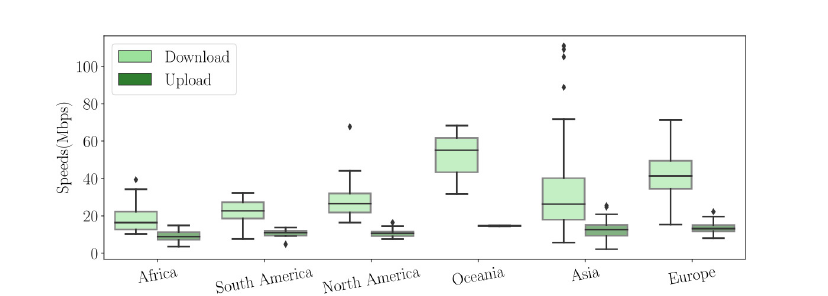}
        \caption{Mobile broadband.\gs}
        \label{fig:ratio_mobile}
    \end{subfigure} \hspace{0.04\textwidth}
    \begin{subfigure}{0.47\textwidth}
        \centering
         \includegraphics[width=1\textwidth]{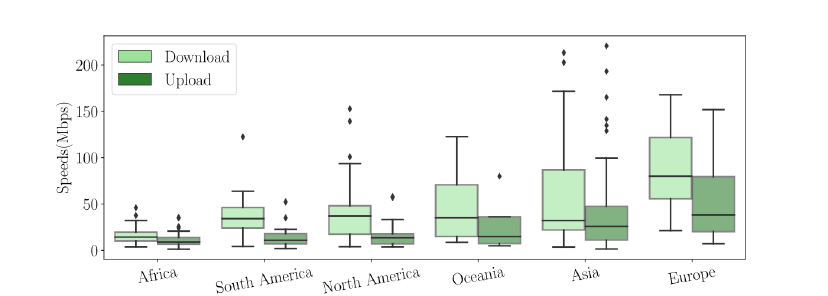}
        \caption{Fixed broadband. \gs}
        \label{fig:ratio_fixed}
    \end{subfigure}\hfill
    \caption{Upload/download speed (in Mbps). Best seen incolors.}
    \label{app:fig:upload_download}
\end{figure}

\begin{figure}
    \centering
    \includegraphics[width=0.47\textwidth]{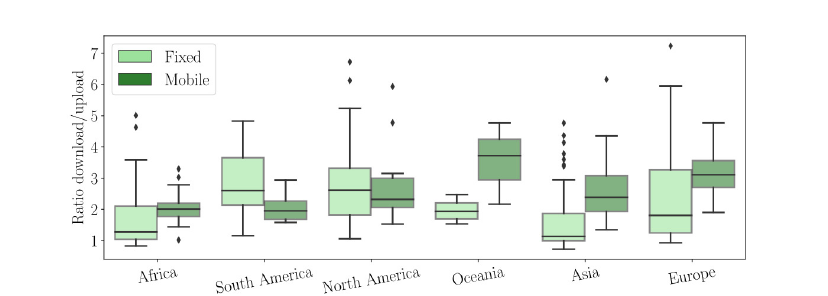}
    \caption{Distribution of the download/upload speeds ratio by continents. Best seen in colors.\gs\gs}
    \label{fig:ratio_boxplot}
\end{figure}

\section{Experiments}
\label{app:experiments}

In this section we provide additional details about our experiments.  We recall that we use two kind of datasets: 1) toy-ish synthetic datasets and 2) real datasets: \textit{superconduct} \citep[see][21263 points, 81 features]{hamidieh_data-driven_2018} and \textit{quantum} \citep[see][50,000 points, 65 features]{caruana_kdd-cup_2004}. 
The aim of using synthetic datasets is mainly to underline the properties resulting from \Cref{thm:cvdist,thm:cvgce_artemis,thm:main_PRave}.

We use the same $1$-quantization scheme (defined in \Cref{app:quantization_schema}, $s=1$ is the most drastic compression) for both uplink and downlink, and thus, we consider that $\omgC^\up = \omgC^\dwn$. In addition, we choose $\alpha^\up = \alpha^\dwn = \ffrac{1}{2(1+\omgC^{\mathrm{up}/\mathrm{dwn}})}$. 

For each figure, we plot the convergence w.r.t. the number of iteration $k$ or w.r.t. the theoretical number of bits exchanged after $k$ iterations. On the Y-axis we display $\log_{10}(F(w_k) - F(w_*))$, with $k$ in $\N$.
All experiments have been run $5$ times and averaged before displaying the curves. 
We plot error bars on all figures. To compute error bars we use the standard deviation of the logarithmic difference between the loss function at iteration $k$ and the objective loss, that is we take standard deviation of $\log_{10}(F(w_k) - F(w_*))$. We then plot the curve $\pm$ this standard deviation.
 
All the code is available in supplementary material.

\subsection{Synthetic dataset}

We build two different synthetic dataset for i.i.d.~or non-i.i.d. cases. We use linear regression to tackle the i.i.d~case and logistic regression to handle the non-i.i.d.~settings. As explained in \Cref{sect:intro}, each worker $i$ holds $n_i$ observations $(z_j^i)_{1 \leq j \leq n_i} = (x_j^i, y_j^i)_{1 \leq j \leq n_i} = (X^i, Y^i)$ following a distribution $D_i$.

We use  $N=10$ devices, each holding $200$ points of dimension $d=20$ for least-square regression and $d=2$ for logistic regression. We ran algorithms over $100$ epochs. 

\paragraph{Choice of the step size for synthetic dataset.} For stochastic descent, we use a step size $\gamma = \frac{1}{L\sqrt{k}}$ with $k$ the number of iteration, and for the batch descent we choose $\gamma = \frac{1}{L}$. 

\textbf{For i.i.d. setting}, we use a linear regression model without bias. For each worker $i$, data points are generated from a normal distribution $(x_j^i)_{1 \leq j \leq n_i} \sim \mathcal{N}(0, \Sigma)$. And then, for all $j$ in $\llbracket 1 , n_i \rrbracket$, we have: $y_j^i = \PdtScl{w}{x_j^i} + e_i$ with $e_i\sim\mathcal{N}(0, \lambda^2)$ and $w$ the true model.

To obtain $\sigmstar = 0$, it is enough to remove the noise $e_i$ by setting the variance $\lambda^2$ of the dataset distribution to $0$. 
Indeed, using a least-square regression, for all $i$ in $\llbracket 1, N\rrbracket$, the cost function evaluated at point $w$ is $F_i(w) = \frac{1}{2} \|{{X^i}^T w - Y^i}\|^2$. Thus the stochastic gradient $j$ in $\llbracket 1 , n_i \rrbracket$ on device $i$ in $\llbracket 1, N \rrbracket$ is $\g_j^i(w) = ({X^i_j}^T w - Y^i_ j) X^i_j$. On the other hand, the true gradient is $\nabla F_i(w) = \E X^i {X^i}^T (w - w^*)$. Computing the difference, we have for all device $i$ in $\llbracket 1 , N \rrbracket$ and all $j$ in $\llbracket1, n_i \rrbracket$: 
\begin{align}
\label{eq:app:explication_sigmastar_zero}
\g_j^i(w) - F_i(w) =  \underbrace{(X^i_j {X^i_j}^T - \E X^i {X^i}^T) (w - w_*)}_{\text{multiplicative noise equal to }0 \text{ in } w_*} + \underbrace{({X^i_j}^T w_* - Y_j^i)}_{\sim \mathcal{N}(0, \lambda^2)} X^i_j
\end{align}

This is why, if we set $\lambda = 0$ and evaluate \cref{eq:app:explication_sigmastar_zero} at $w_*$, we get back \Cref{asu:noise_sto_grad} with $\sigmstar = 0$, and as a consequence, the stochastic noise at the optimum is removed. Remark that it remains a stochastic gradient descent, and the uniform bound on the gradients noise \textbf{is not 0}. We set $\lambda^2 = 0 (\Leftrightarrow \sigmstar^2 = 0)$ in \Cref{fig:app:without_noise}. Otherwise, we set $\lambda^2 = 0.4$. 

\textbf{For non-i.i.d.}, we generate two different datasets based on a logistic model with two different parameters: $w_1 = (10, 10)$ and $w_2 = (10, -10)$. Thus the model is expected to converge to $w_* = (10, 0)$. We have two different data distributions $x_1\sim \mathcal{N} \left(0, {\Sigma_{1}} \right)$ and $x_2\sim \mathcal{N}\left(0, {\Sigma_{2}} \right)$, and for all $i$ in $\llbracket 1, N \rrbracket$, for all $k$ in $\llbracket 1, n_i \rrbracket\,, y_{k}^i = \mathcal{R}\left(\text{Sigm} \left(\PdtScl{w_{(i \mod 2) + 1}}{ x_{(i \mod 2) + 1}^k} \right) \right) \in \{-1, +1\}$. That is, half the machines use the first distribution $\mathcal{N} \left(0, {\Sigma_{1}} \right)$ for inputs and model $w_1$ and the other half the second distribution for inputs and model $w_2$.
Here, $\mathcal{R}$ is the Rademacher distribution and $\text{Sigm}$ is the sigmoid function defined as $\text{Sigm:} x \mapsto \ffrac{\e^x}{1+ \e^x}$. These two distributions are presented on \Cref{fig:distibution_logistic}.

\begin{figure*}
    \centering
    \begin{subfigure}{0.45\textwidth}
        \centering
        \includegraphics[width=1\textwidth]{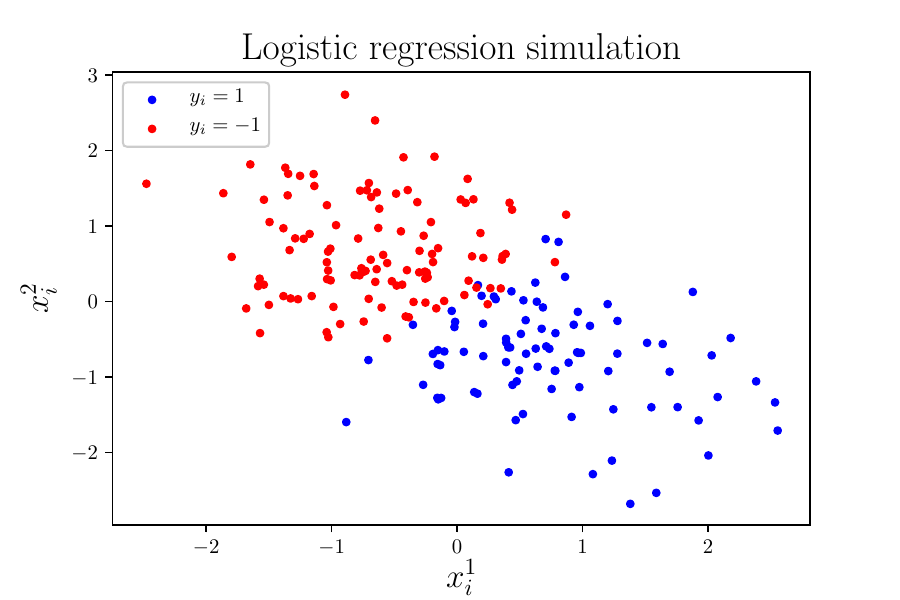}
        \caption{Dataset 1}
    \end{subfigure}
    \begin{subfigure}{0.45\textwidth}
        \centering
        \includegraphics[width=1\textwidth]{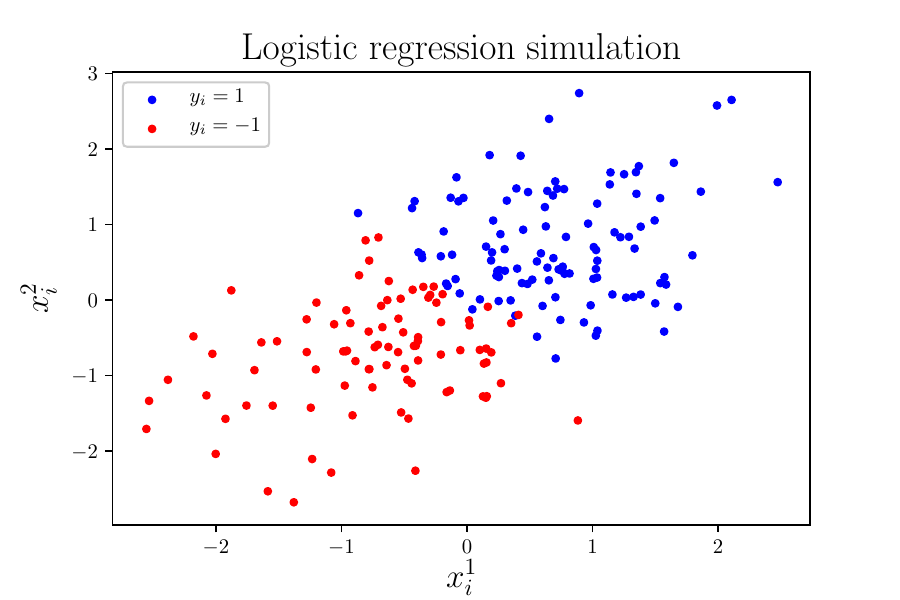}
        \caption{Dataset 2}
    \end{subfigure}
    \caption{Data distribution for logistic regression to simulate non-i.i.d. data. Half of the device hold first dataset, and the other half the second one.\gs\gs\gs}
    \label{fig:distibution_logistic}
\end{figure*}

\subsubsection{Least-square regression}
\label{app:exp:leastsquare}

In this section, we present all figures generated using Least-Square regression. \Cref{fig:app:with_noise} corresponds to \Cref{fig:LSR_noised}.

As explained in the main of the paper, in the case of $\sigmstar \neq 0$ (\Cref{fig:app:with_noise}), algorithm using memory (i.e \texttt{Diana} and \Artemis) are not expected to outperform those without (i.e \texttt{QSQGD} and \texttt{Bi-QSGD}). On the contrary, they saturate at a higher level. However, as soon as the noise at the optimum is $0$ (\Cref{fig:app:without_noise}), all algorithms (regardless of memory), converge at a linear rate exactly as classical SGD.

\begin{figure}
    \centering
    \begin{subfigure}{\sizefig\textwidth}
        \centering
        \includegraphics[width=1\textwidth]{pictures/with_noise_noavg_it-eps-converted-to.pdf}
        \caption{LSR: $\sigmstar^2 \neq 0$}
        \label{fig:app:LSR_noised}
    \end{subfigure}
    \begin{subfigure}{\sizefig\textwidth}
        \centering
        \includegraphics[width=1\textwidth]{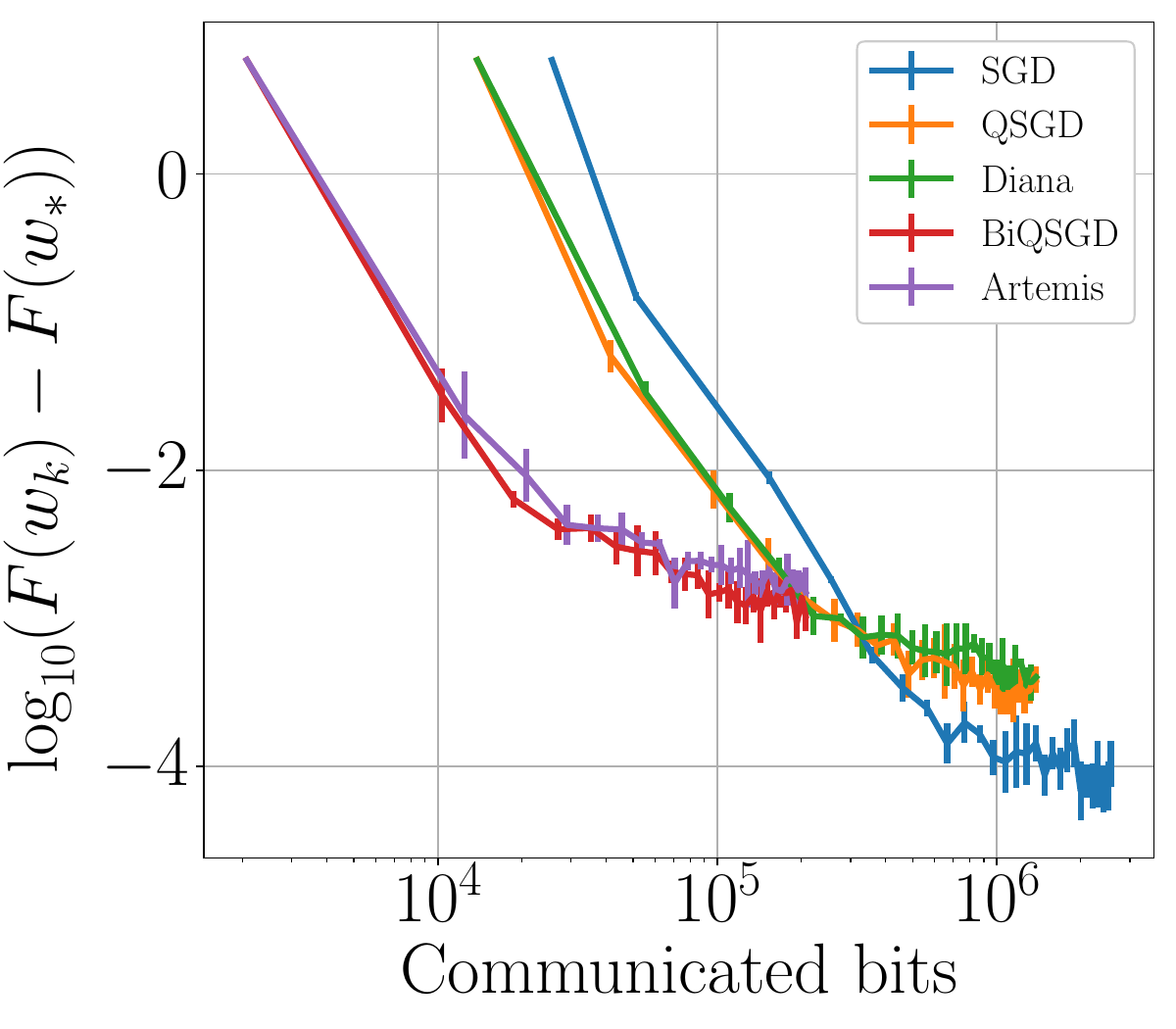} \label{fig:app:LSR_bits}
        \caption{X-axis in \# bits.}
    \end{subfigure}\hfill
    \caption[Figures of the main]{\textbf{Synthetic dataset, Least-Square Regression with noise} ($\sigmstar \neq 0$). In a situation where data is i.i.d., the memory does not present much interest, and has no impact on the convergence. Because $\sigmstar^2\neq0$, all algorithms saturate ; and saturation level is higher for double compression (\Artemis, \texttt{Bi-QSGD}), than for simple compression (\texttt{Diana}, \texttt{QSGD}) or than for SGD. This corroborates findings in \Cref{thm:cvgce_artemis} and \Cref{thm:cvdist}.\gs\gs\gs}
    \label{fig:app:with_noise}
\end{figure}\gs

\begin{figure}
    \centering
    \gs 
    \begin{subfigure}{\sizefig\textwidth}
        \centering
        \includegraphics[width=1\textwidth]{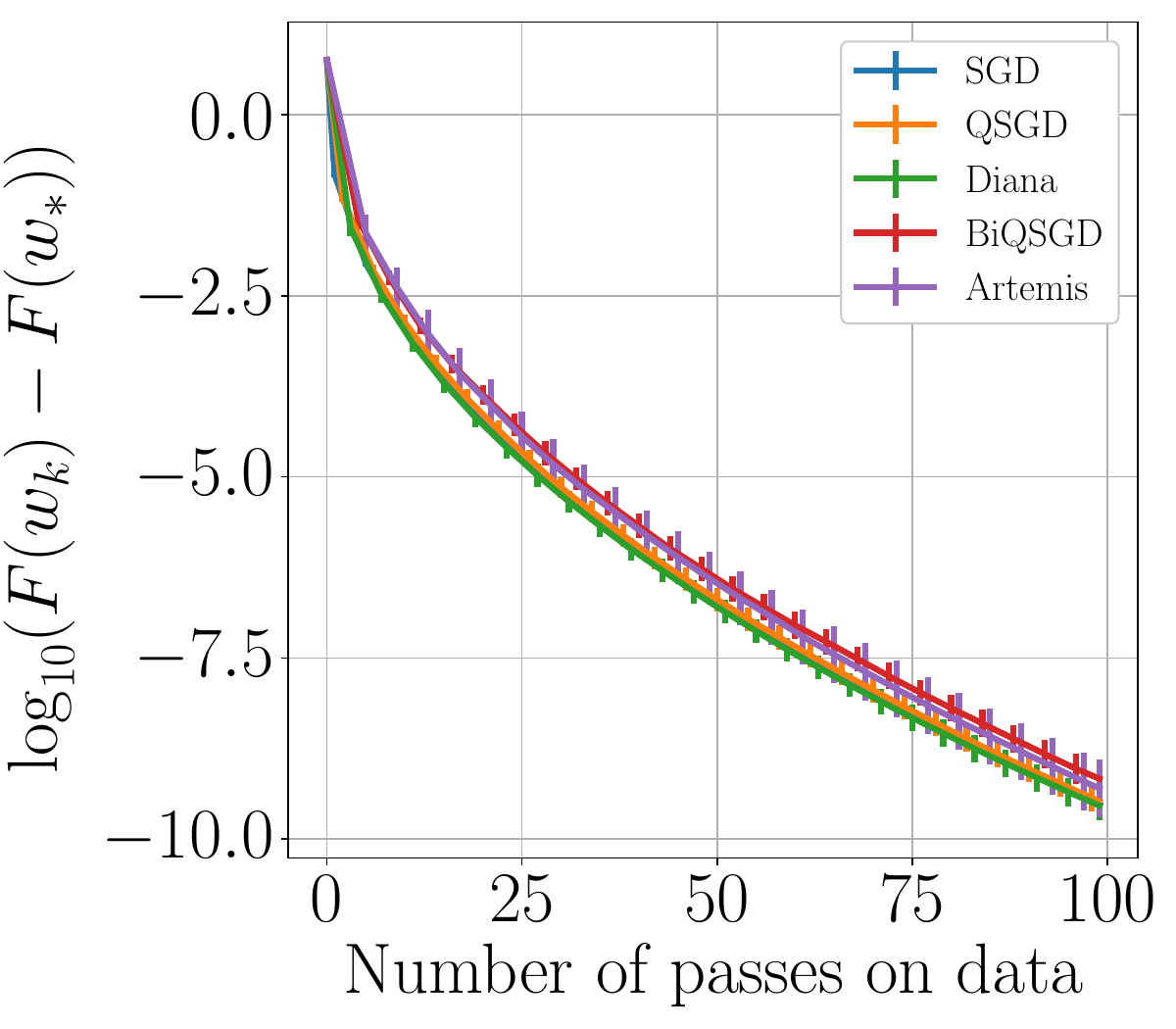}
        \caption{LSR: $\sigmstar^2 = 0$}
        \label{fig:app:LSR_nonoise}
    \end{subfigure}
    \begin{subfigure}{\sizefig \textwidth}
        \centering
        \includegraphics[width=1\textwidth]{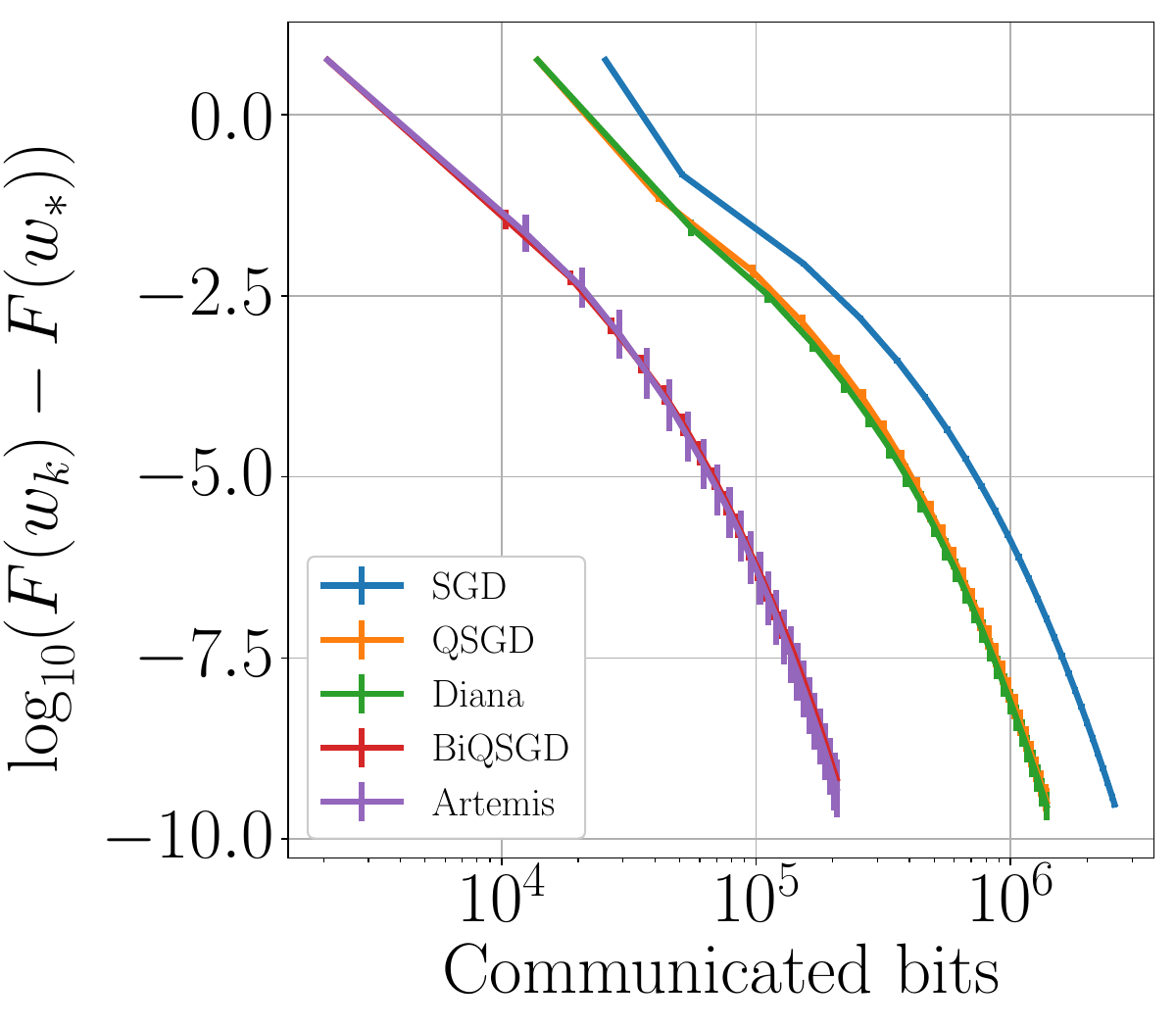} 
        \caption{X-axis in \# bits.}
    \end{subfigure}\hfill
    \caption[Figures of the main]{
    \textbf{Synthetic dataset, Least-Square Regression without noise} ($\sigmstar = 0$). Without surprise, with i.i.d~data and $\sigmstar = 0$, the convergence of each algorithm is linear. Thus, in i.i.d.~settings, the impact of the memory is negligible, but this will not be the case in the non-i.i.d.~settings as underlined by \Cref{fig:app:determinist_gradient_without_avg}.\gs\gs\gs}
    \label{fig:app:without_noise}
\end{figure}

\subsubsection{Logistic regression}
\label{app:exp:logistic_regression}

In this section, we present all figures generated using a logistic regression model. \Cref{fig:app:determinist_gradient_without_avg} corresponds to \Cref{fig:deterministic_noniid}. Data is non-i.d.d.~and we use a full batch gradient descent to get $\sigmstar = 0$ to shed into light the impact of memory over convergence.

\begin{figure}
    \centering
    \begin{subfigure}{\sizefig \textwidth}
        \centering
        \includegraphics[width=1\textwidth]{pictures/logistic_deterministic_noavg_it-eps-converted-to.pdf}
        \caption{LR: $\sigmstar^2 = 0$}
    \end{subfigure}
    \begin{subfigure}{\sizefig \textwidth}
        \centering
        \includegraphics[width=1\textwidth]{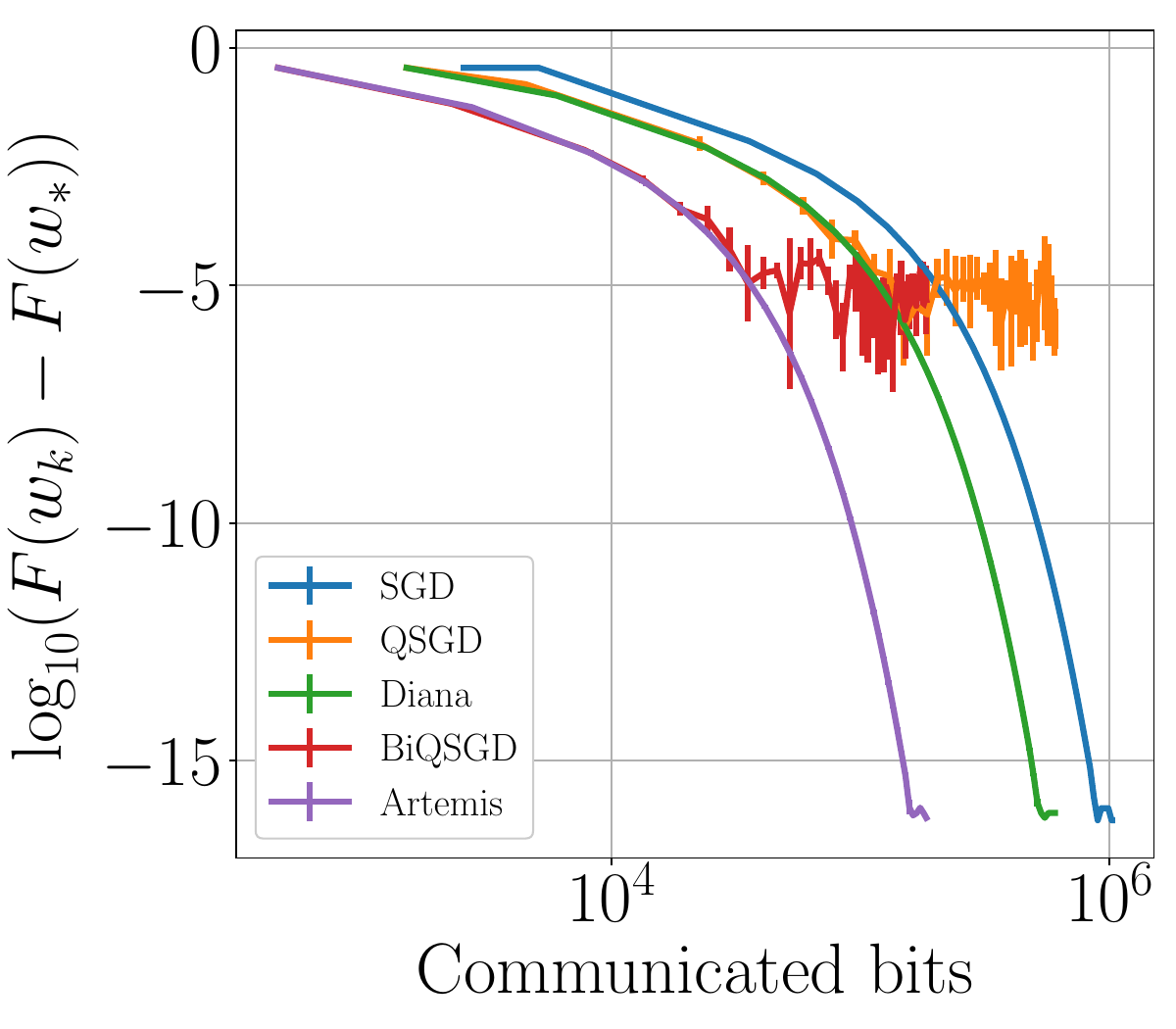} 
        \caption{X-axis in \# bits.}
    \end{subfigure}\hfill
    \caption[Figures of the main]{\textbf{Synthetic dataset, Logistic Regression on non-i.i.d.~data} using a batch gradient descent (to get $\sigmstar = 0$). The benefit of memory is obvious, it makes the algorithm to converge linearly, while algorithms without are saturating at a high level. This stress on the importance of using the memory in non-i.i.d.~settings.\gs\gs\gs}
    \label{fig:app:determinist_gradient_without_avg}
\end{figure}

\begin{figure}
    \centering
    \begin{subfigure}{\sizefig\textwidth}
        \centering
        \includegraphics[width=1\textwidth]{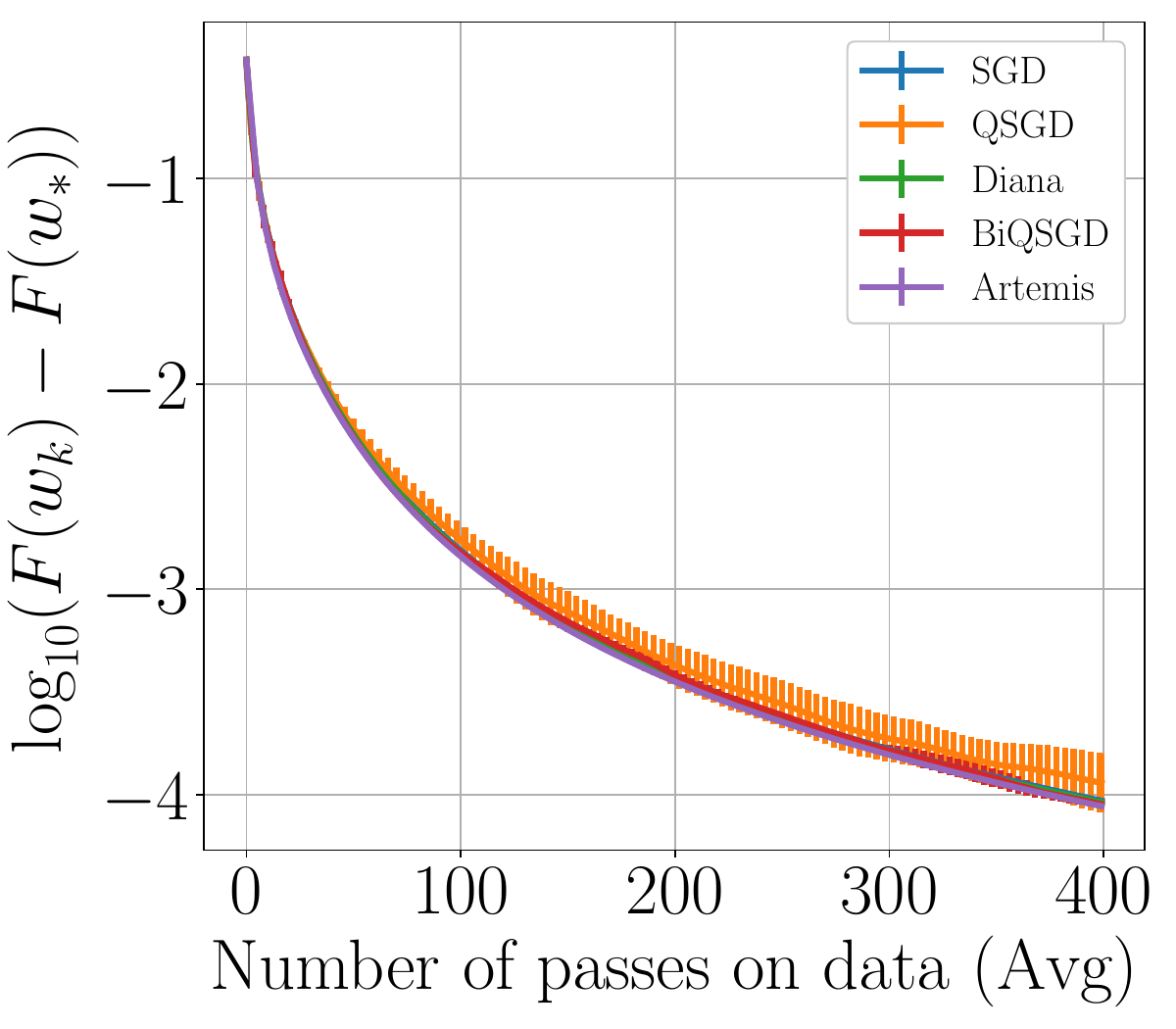}
        \caption{LR: $\sigmstar^2 = 0$}
    \end{subfigure}
    \begin{subfigure}{\sizefig\textwidth}
        \centering
        \includegraphics[width=1\textwidth]{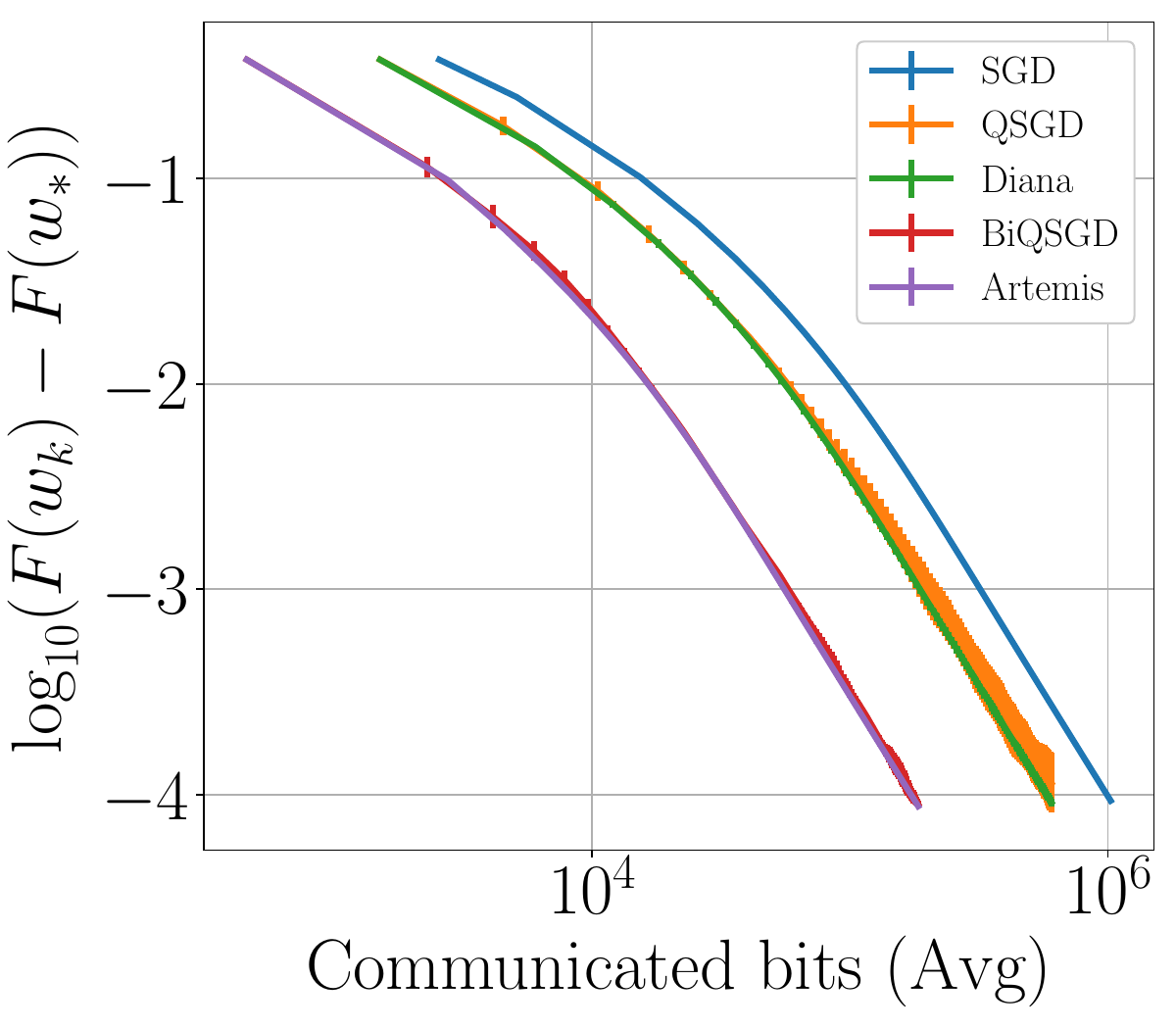} 
        \caption{X-axis in \# bits.\gs}\gs
    \end{subfigure}\hfill
    \caption[Polyak-Ruppert]{\textbf{Polyak-Ruppert averaging, synthetic dataset.} Logistic Regression on non-i.i.d.~data using a batch gradient descent (to get $\sigmstar = 0$) and a Polyak-Ruppert averaging. The convergence is linear as predicted by \Cref{thm:main_PRave} because $\sigmstar = 0$. Best seen in colors.\gs\gs\gs}
    \label{fig:app:determinist_gradient_with_avg}
\end{figure}

\Cref{fig:app:determinist_gradient_with_avg} is using same data and configuration as \Cref{fig:app:determinist_gradient_without_avg}, except that \textit{it is combined with a Polyak-Ruppert averaging}. Note that in the absence of memory the variance increases compared to algorithms using memory. To generate these figures, we didn't take the optimal step size. But if we took it, the trade-off between variance and bias would be worse and algorithms using memory would outperform those without. 

\subsection{Real datasets: \textit{Quantum} and \textit{Superconduct}}

In this section, we present details about experiments conducted on real-life datasets: \textit{superconduct} (from \citet{caruana_kdd-cup_2004}) where we use a least-square regression, and \textit{quantum} (from \citet{hamidieh_data-driven_2018}) with a logistic regression. All figures can be found in the notebooks provided in supplementary materials. 

In this following, we present results on superconduct and quantum in the setting of full device participation. We detail experiments in the PP setting in \Cref{app:subsubsec:PP}. Next, we address the issue of the optimal step size in \Cref{app:subsect:opt_step_size}. In \Cref{app:subsec:artemis_vs_existing} we compare \Artemis~to other existing algorithm doing compression in a distributed learning framework. Finally, we estimate in \Cref{app:subsec:cpu_usage} the carbon footprint of our work.

In order to simulate non-i.i.d.~data and to make the experiments closer to real-life usage, we split the dataset in heterogeneous groups using a Gaussian mixture clustering on TSNE representations (defined by \citet{maaten_visualizing_2008}). Thus, the data are highly non-i.i.d.~\textit{and} unbalanced over devices. We plot on \Cref{app:fig:quantum_and_superconduct_tsne} the TSNE representation of the two real datasets. 

\begin{figure}
\centering
    \begin{subfigure}{0.47\textwidth}
        \centering
        \includegraphics[width=1\textwidth]{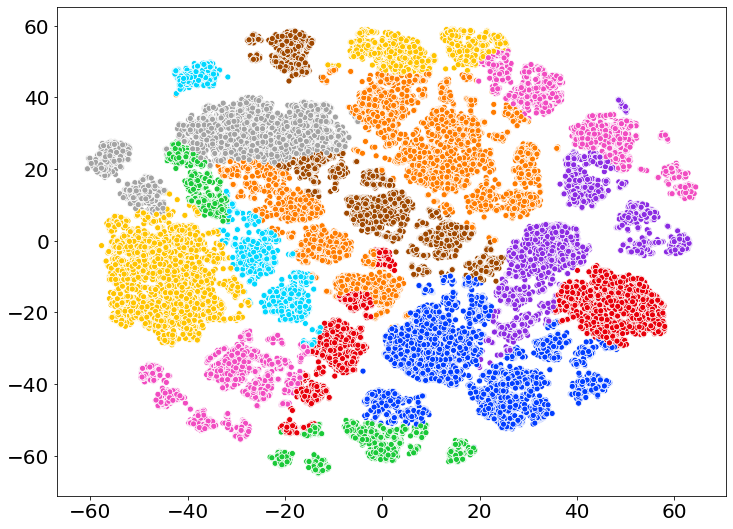} 
        \caption{Quantum dataset: $20$ clusters. Each cluster has between $900$ and $10500$ points with a median at $2300$ points.}
    \end{subfigure}
    \begin{subfigure}{0.47\textwidth}
        \centering
        \includegraphics[width=1\textwidth]{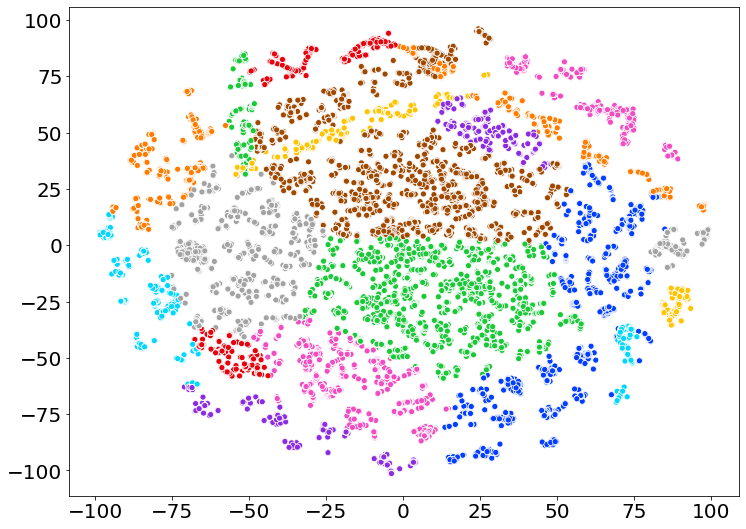}
        \caption{Superconduct dataset: $20$ cluster. Each cluster has between $250$ and $3900$ points with a median at $750$ points.}
    \end{subfigure} \hspace{0.04\textwidth} \hfill
    \caption{TSNE representations. Best seen in colors.}
    \label{app:fig:quantum_and_superconduct_tsne}
\end{figure}

There are $N = 20$ devices for \textit{superconduct} and \textit{quantum} datasets. For \textit{superconduct}, there are between $250$ and $3900$ points by worker, with a median at $750$ ; and for \textit{quantum}, there are between $900$ and $10500$ points, with a median at $2300$. On each figure, we indicate which step size $\gamma$ has been used.

\textbf{Convex settings} are given in \Cref{app:tab:settings_convex}. Experiments have been performed with $150$ epochs in the stochastic regime, and $800$ epochs in the full batch regime. We use quantization \citep[defined in][]{alistarh_qsgd_2017} with $s = 2^0$ for all experiments, except in the case of partial participation where we used $s=2^1$.

\begin{table}[!htp]
\caption{Settings of experiments.}
\label{app:tab:settings_convex}
\centering
\begin{tabular}{lcccc}
Settings & quantum & superconduct \\
\hline 
references & \cite{caruana_kdd-cup_2004} & \cite{hamidieh_data-driven_2018} \\
model & LR & LSR \\
dimension $d$ & $66$ & $82$ \\
training dataset size & $50,000$  & $21,200$ \\
batch size $b$& $256$& $64$  \\
compression rate $s$ & \multicolumn{2}{c}{$2^0$ (\textit{i.e.} two levels)}\\
norm quantization  &\multicolumn{2}{c}{$\| \cdot \|_{2}$}  \\
momentum $m$ & \multicolumn{2}{c}{no momentum}\\
step size $\gamma$ & \multicolumn{2}{c}{$1/L$} \\
\bottomrule
\end{tabular}
\end{table}

\Cref{fig:app:superconduct_sto,fig:app:quantum_sto} correspond to \Cref{fig:real_dataset}.
We observe on these figures the benefit of the memory.
The level of saturation of algorithms using memory is much lower than those without memory. 
Additionally, \Cref{thm:cvgce_artemis} highlights that the level of saturation (see constant $E$ of \Cref{tab:p_and_E}) is proportional to the level of compression $\omgC^{\up/\dwn}$. This is indeed observed on \Cref{fig:app:superconduct_sto,fig:app:quantum_sto,fig:app:superconduct_full,fig:app:quantum_full}.

In the case of the \textit{quantum} dataset (see \Cref{fig:app:quantum_sto}), \Artemis~is not only better than $\texttt{Bi-QSGD}$, but in fact, as good as $\texttt{QSGD}$. That is to say, we achieve to make an algorithm using a bidirectional compression, as good as an algorithm handling unidirectional compression.

On \Cref{fig:app:quantum_full,fig:app:superconduct_full}, we represent the convergence of the five algorithms in a full batch mode resulting to $\sigmstar = 0$. In this case, as the dependency on $B^2$ is removed, \Cref{thm:cvgce_artemis} predicts that we must have a linear convergence for algorithms using memory. This is experimentally observed.

\textbf{Memory trade-off: batch size, noise at the optimum, and heterogeneity.} Because the variance of the algorithm (see constant $E$ of \Cref{tab:p_and_E}) is divided by the batch size $b$, the choice of this hyperparameter is not without importance. Indeed, reducing the batch size will increase the impact of $\sigmstar$ on the convergence's rate, while the impact of $B^2$ will remain constant. Thus, there is a \textit{trade-off}: if the batch-size is too small, the quantity $\sigmstar/b$ will become larger than $B^2$, and the impact of the memory will be hidden by the second term depending on the dataset heterogeneity. This will lead \Artemis-like algorithms to fail: the memory term is canceled by the high heterogeneity.
On the other hand, if the dataset does not present enough heterogeneity, the constant $B^2$, will be negligible making memory useless, or even penalizing. 

To summarize, \Cref{fig:app:superconduct_sto,fig:app:quantum_sto,fig:app:superconduct_full,fig:app:quantum_full} underline the benefit of using memory in the stochastic and full batch regime for non-i.i.d. datasets.

\begin{figure}
    \centering
    \begin{subfigure}{\sizefig \textwidth}
        \centering
        \includegraphics[width=1\textwidth]{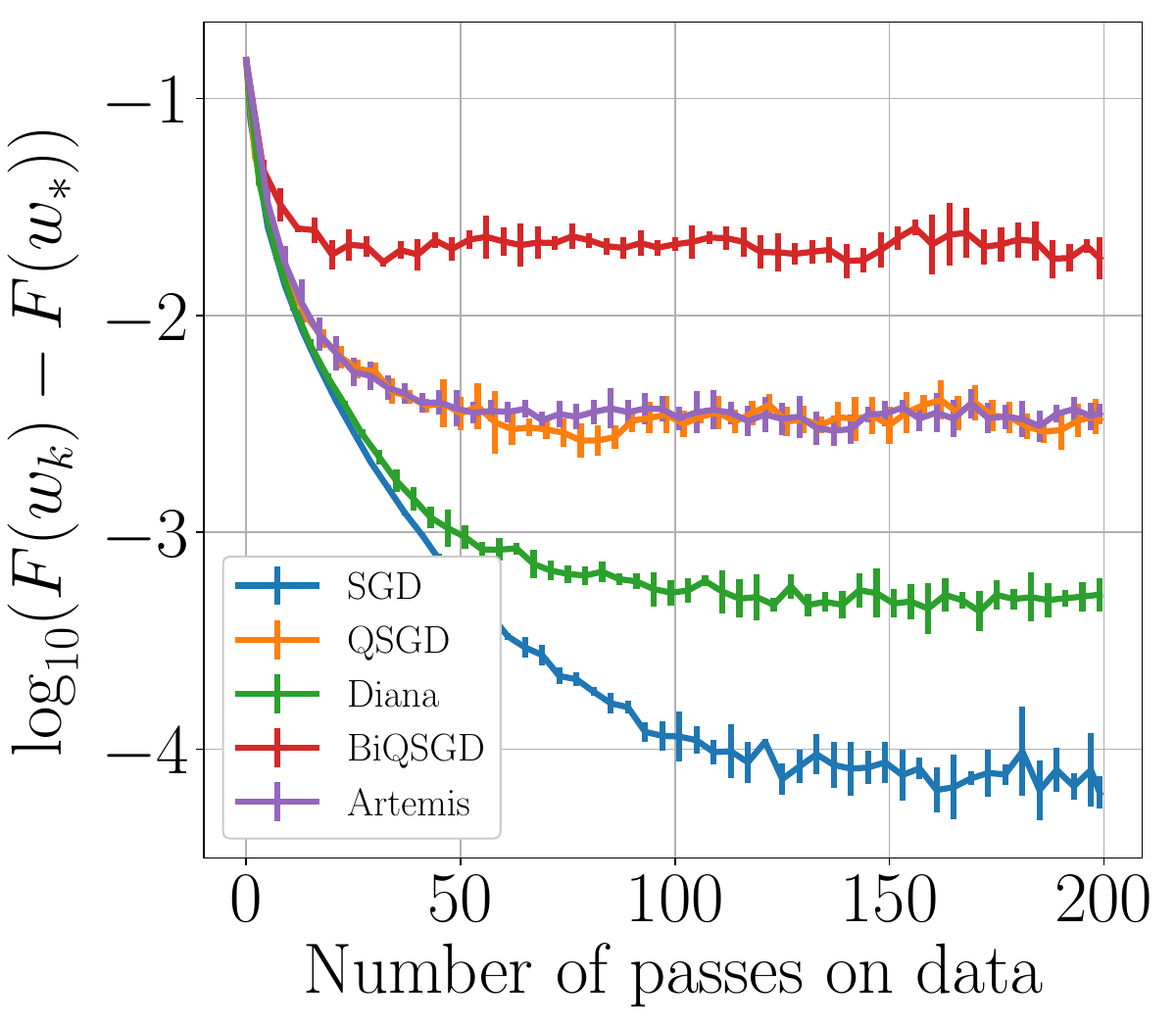}
        \caption{X-axis in \# epoch}
        \label{fig:app:quantum_it_sto}
    \end{subfigure}
    \begin{subfigure}{\sizefig \textwidth}
        \centering
        \includegraphics[width=1\textwidth]{pictures/quantum/bits-noavg-Qtzd-sto-b256-eps-converted-to.pdf}
        \caption{X-axis in \# bit}
        \label{fig:app:quantum_bits_sto}
    \end{subfigure}\hfill
    \caption[Figures of the main]{\textbf{\textit{Quantum}}. least-square regression, $\sigmstar \neq 0$, $\gamma = 1/L$, $b=256$, non-i.i.d.. Best seen in colors.\gs\gs\gs}
    \label{fig:app:quantum_sto}
\end{figure}

\begin{figure}
    \centering
    \begin{subfigure}{\sizefig \textwidth}
        \centering
        \includegraphics[width=1\textwidth]{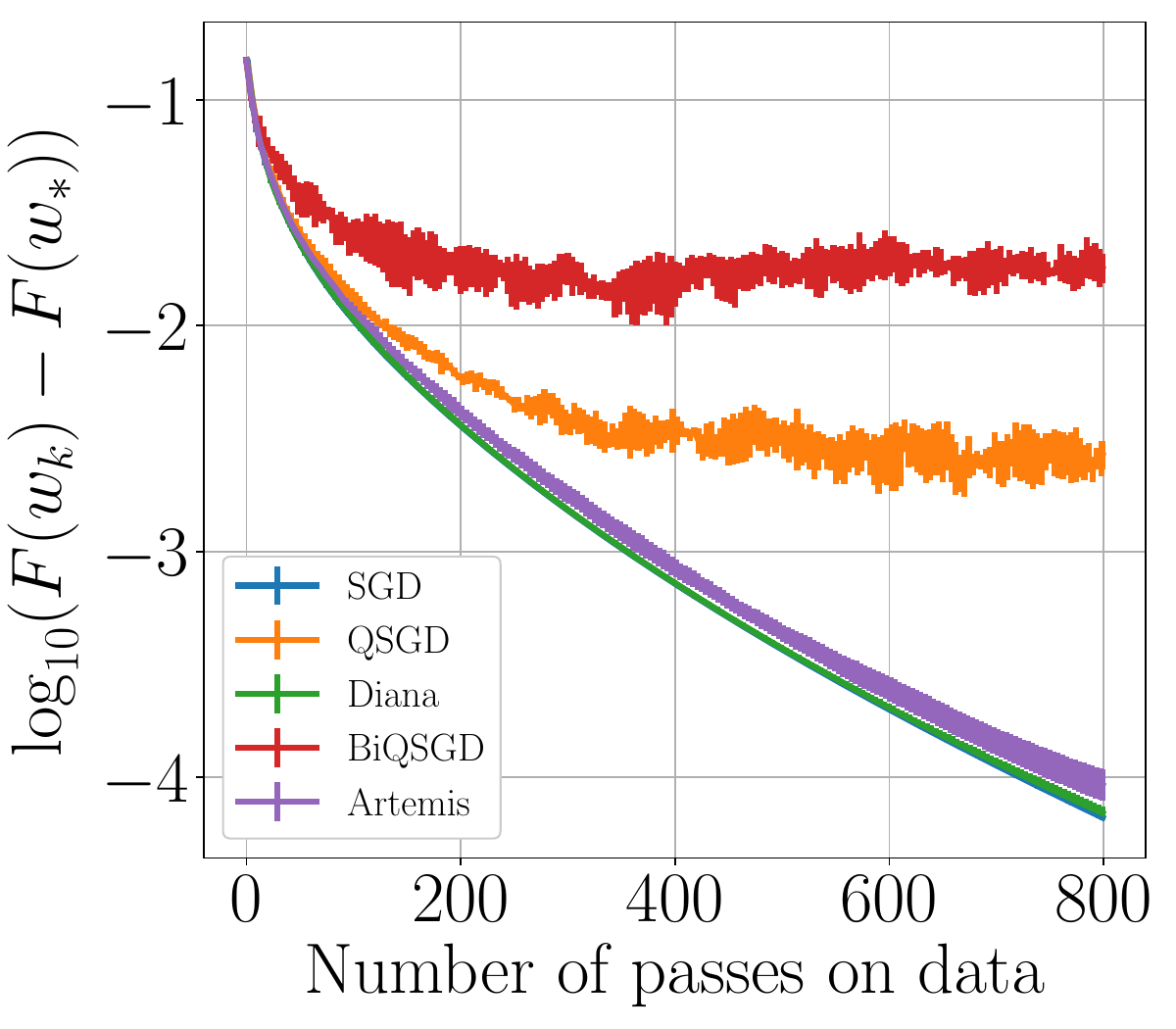}
        \caption{X-axis in \# epoch}
        \label{fig:app:quantumm_it_full}
    \end{subfigure}
    \begin{subfigure}{\sizefig \textwidth}
        \centering
        \includegraphics[width=1\textwidth]{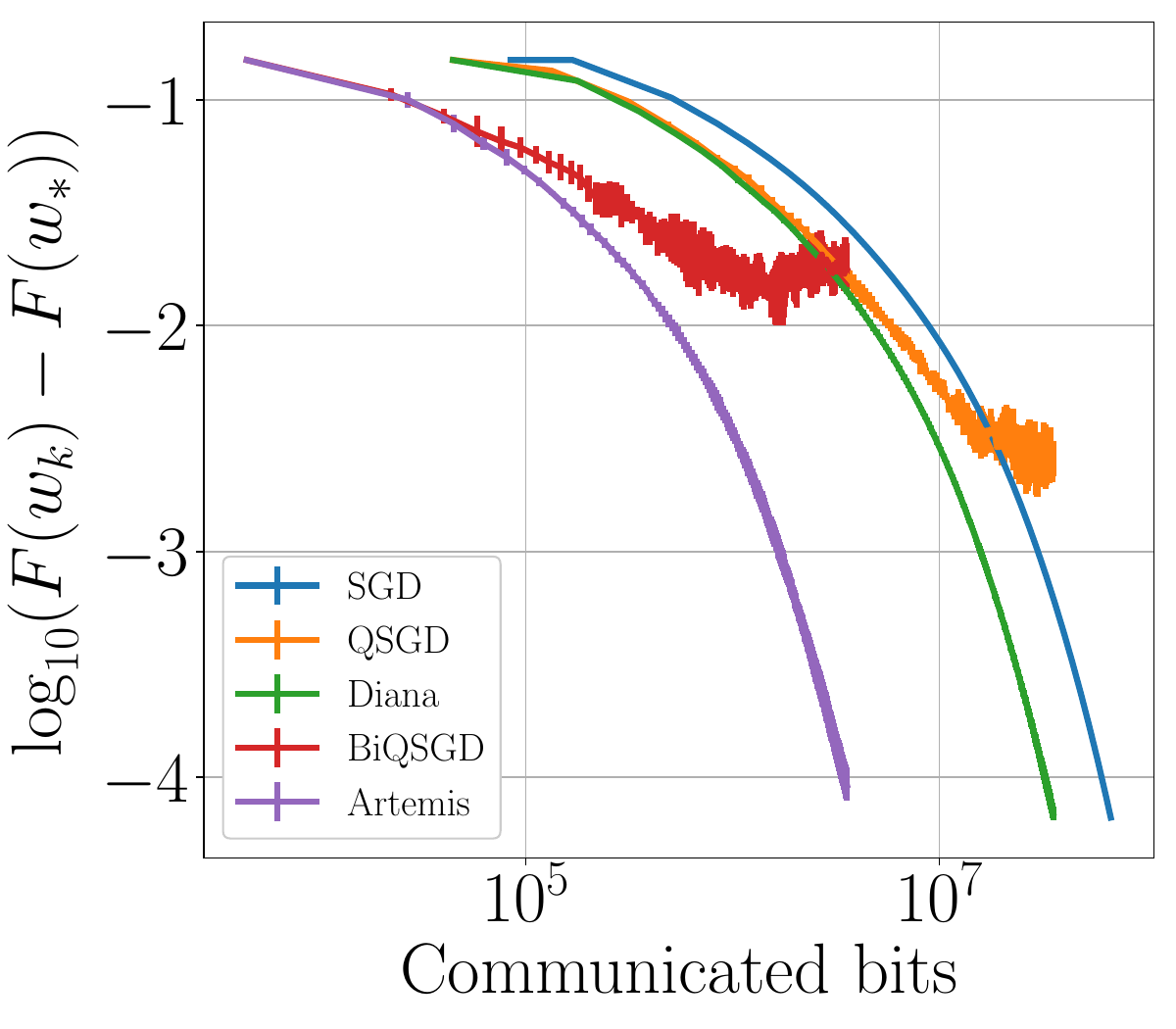}
        \caption{X-axis in \# bit}
        \label{fig:app:quantum_bits_full}
    \end{subfigure}\hfill
    \caption[Figures of the main]{\textbf{\textit{Quantum}}. least-square regression, $\sigmstar = 0$, $\gamma = 1/L$, $b=256$, non-i.i.d.. Best seen in colors.\gs\gs\gs}
    \label{fig:app:quantum_full}
\end{figure}

\begin{figure}
    \centering
    \begin{subfigure}{\sizefig \textwidth}
        \centering
        \includegraphics[width=1\textwidth]{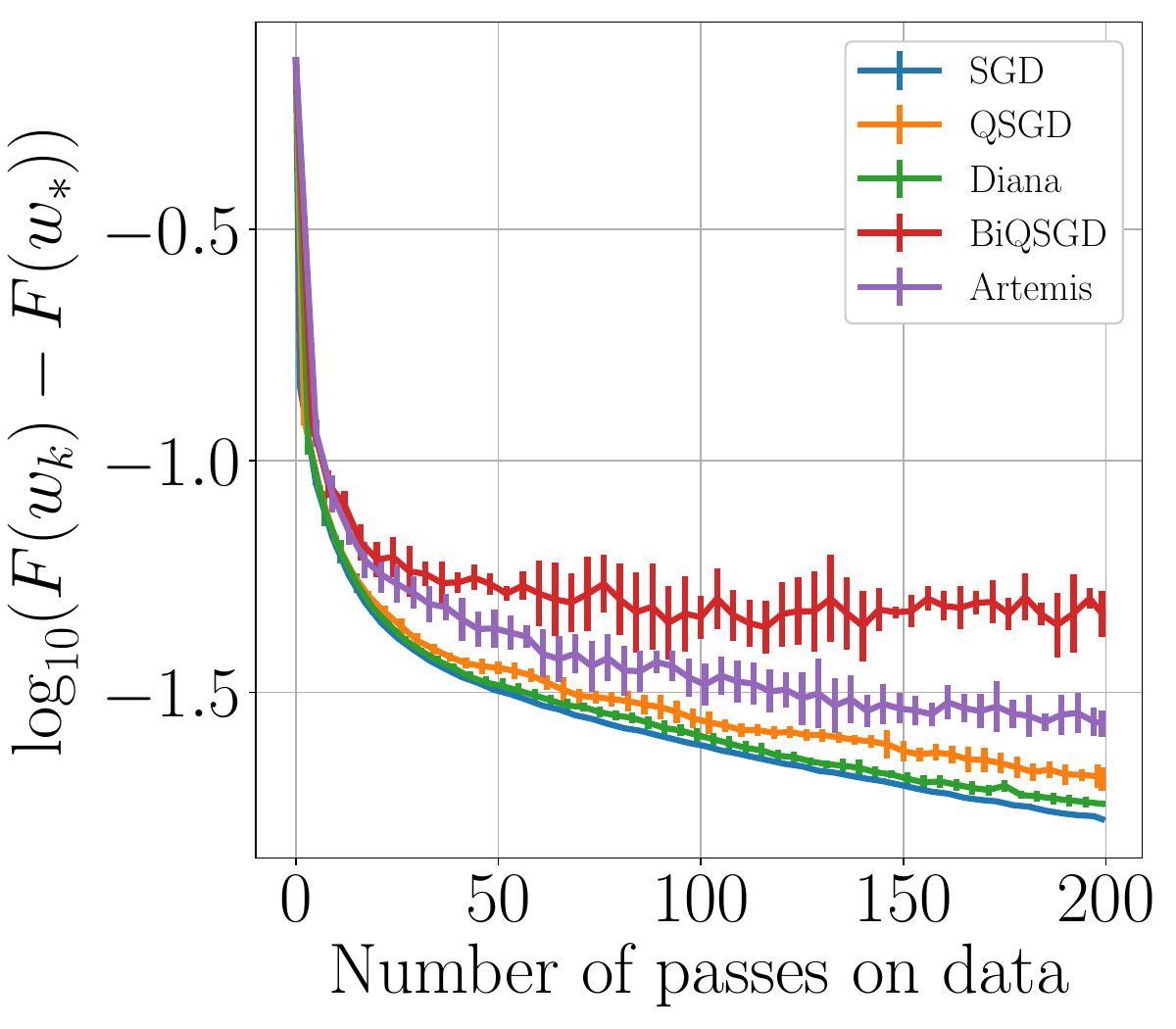}
        \caption{X-axis in \# epochs}
        \label{fig:app:superconduct_it_sto}
    \end{subfigure}
    \begin{subfigure}{\sizefig \textwidth}
        \centering
        \includegraphics[width=1\textwidth]{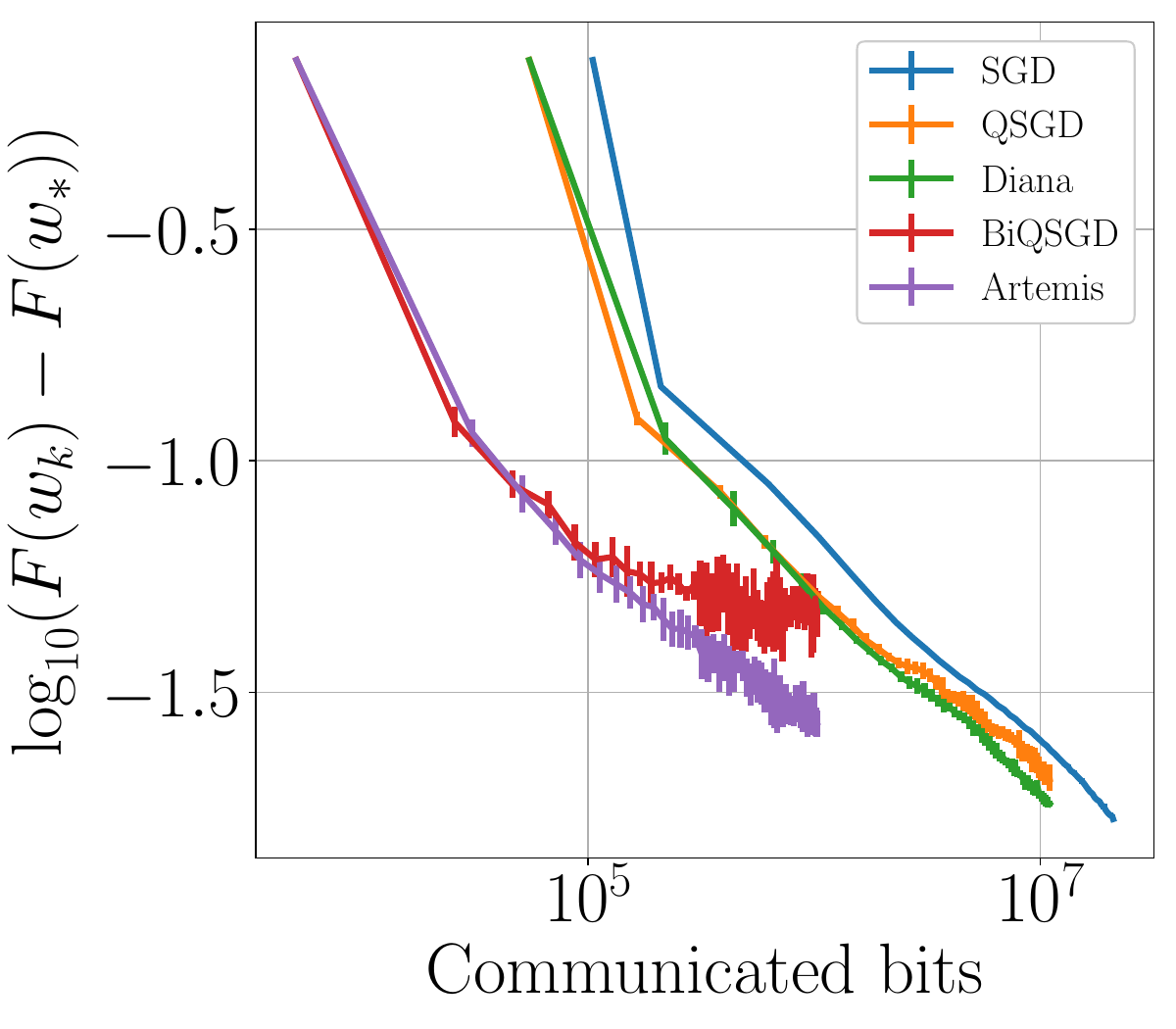} 
        \caption{X-axis in \# bit}
        \label{fig:app:superconduct_bits_sto}
    \end{subfigure}\hfill
    \caption[Figures of the main]{\textbf{\textit{Superconduct}}. least-square regression, $\sigmstar \neq 0$, $\gamma = 1/L$, $b=64$, non-i.i.d.. Best seen in colors. \gs\gs\gs}
    \label{fig:app:superconduct_sto}
\end{figure}

\begin{figure}
    \centering
    \begin{subfigure}{\sizefig \textwidth}
        \centering
        \includegraphics[width=1\textwidth]{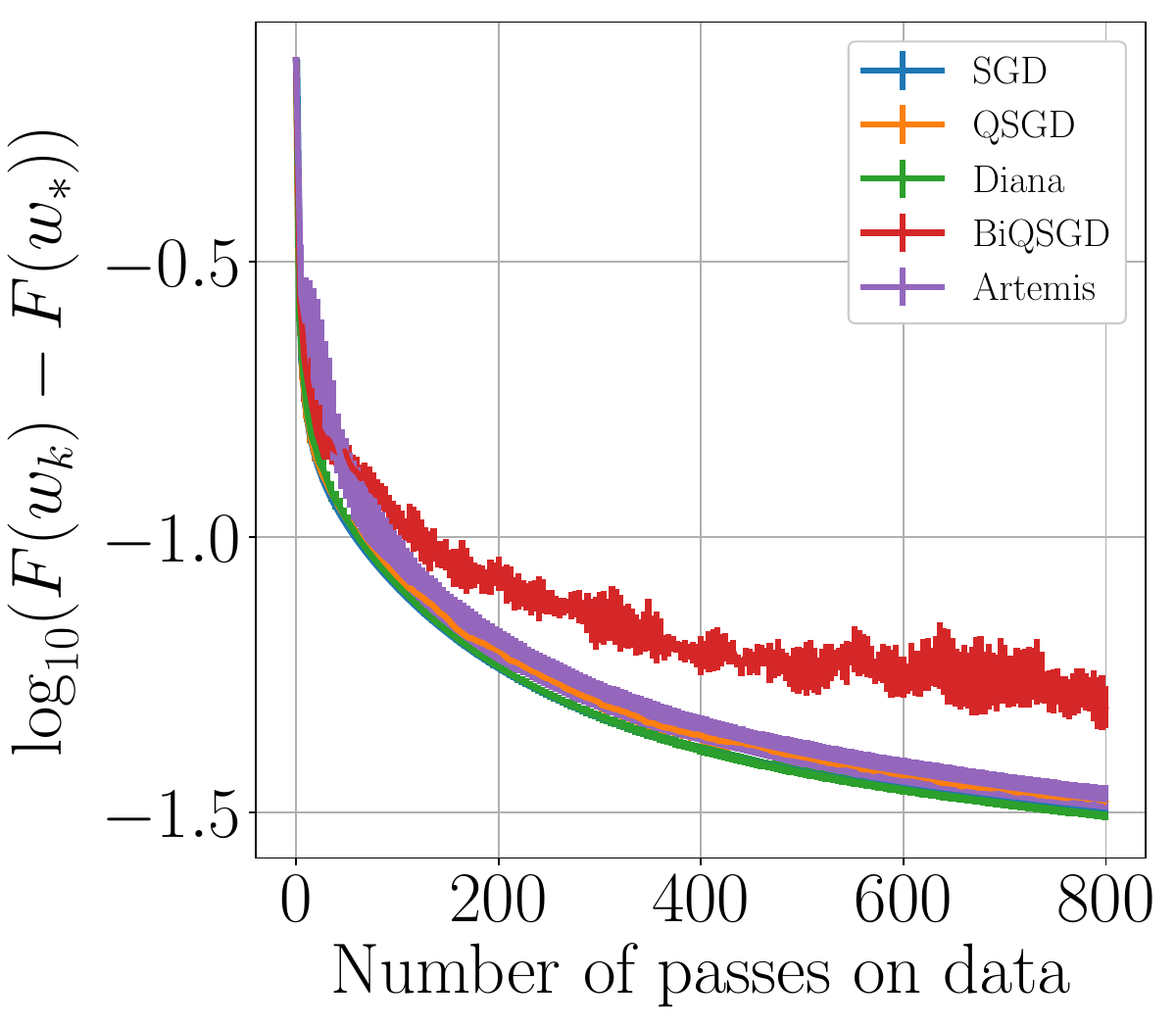}
        \caption{X-axis in \# epochs}
        \label{fig:app:superconduct_it_full}
    \end{subfigure}
    \begin{subfigure}{\sizefig \textwidth}
        \centering
        \includegraphics[width=1\textwidth]{pictures/superconduct/bits-noavg-Qtzd-full-eps-converted-to.pdf} 
        \caption{X-axis in \# bit}
        \label{fig:app:superconduct_bits_full}
    \end{subfigure}\hfill
    \caption[Figures of the main]{\textbf{\textit{Superconduct}}. least-square regression, $\sigmstar = 0$, $\gamma = 1/L$, $b=64$, non-i.i.d.. Best seen in colors. \gs\gs\gs}
    \label{fig:app:superconduct_full}
\end{figure}

\subsubsection{Partial participation}
\label{app:subsubsec:PP}

In this section we provide additional experiments on partial participation in the stochastic regime. Only half of the devices (randomly sampled) participate at each round, we use $2^1$-quantization.

\Cref{app:fig:real_dataset_PP1:sto} presents the first naive approach (\textbf{PP1}) to handle partial participation. This naive solution fails to properly converge. In the other hand, algorithms using \textbf{PP2} - \texttt{SGD} with memory i.e \Artemis~with $\omgC^{\up/\dwn} = 0$, \Artemis~with unidirectional compression i.e. $\omgC^\dwn = 0$ and \Artemis~with $\omgC^{\up/\dwn} \neq 0$ - presents much better convergence, see \Cref{app:fig:real_dataset_PP2:sto}. As an example, \texttt{SGD} with memory matches the results of \texttt{SGD} in the case of full participation (\Cref{fig:app:superconduct_sto,fig:app:quantum_sto}).
However, the convergence of \texttt{QSGD} and \texttt{Bi-QSGD} is unchanged as there is no difference between the two approaches in the absence of memory. 

The result in the full gradient regime is given in \Cref{sec:expemain}. On \Cref{fig:real_dataset_PP2,fig:real_dataset_PP1}, we can observe that our new algorithm \textbf{PP2} has a linear convergence unlike \textbf{PP1}.

As a conclusion on partial participation: with \textbf{PP2}, we observe the \textit{significant} impact of memory when using non-i.i.d.~data. Comparing \Cref{app:fig:real_dataset_PP2:sto} to \Cref{app:fig:real_dataset_PP1:sto}, the saturation of algorithms with $\alpha$ different to zero (i.e. using memory) is much lower; and classical \texttt{SGD} is outperformed by its variant using the memory mechanism.

\begin{figure}
    \centering
    \begin{subfigure}{\sizefig\textwidth}
        \centering
        \includegraphics[width=1\textwidth]{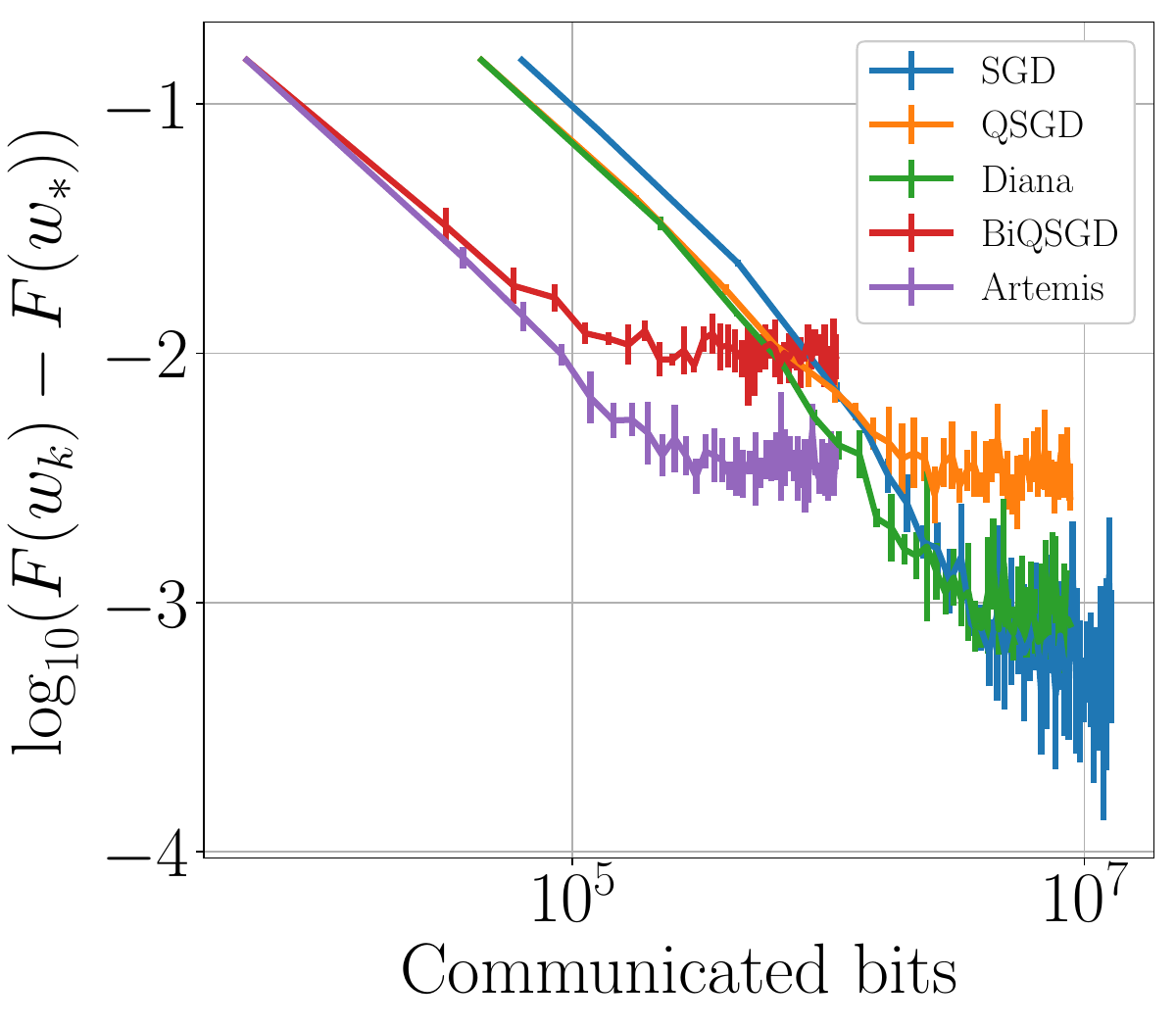}
        \caption{Quantum\gs}
    \end{subfigure}
    \begin{subfigure}{\sizefig\textwidth}
        \centering
        \includegraphics[width=1\textwidth]{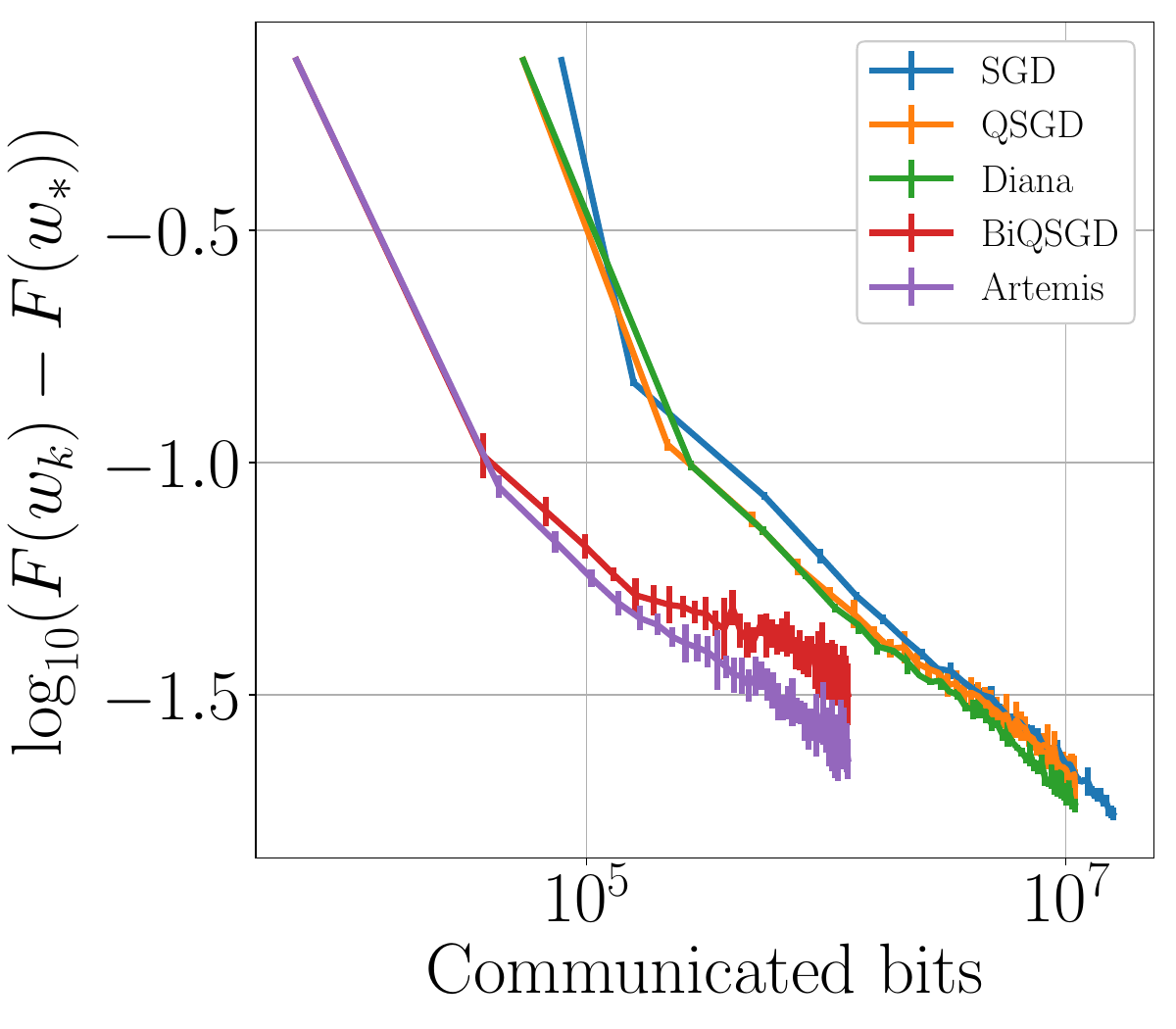}
        \caption{Superconduct\gs}
    \end{subfigure}
    \caption[Real dataset]{\textbf{Partial participation, stochastic regime - PP1.} $\sigmstar \neq 0$, $N=20$ workers, $p=0.5$,  $b >1$ ($150$ iter.) With this variant, all algorithms are saturating at a high level. 
    \vspace{-1em}}
    \label{app:fig:real_dataset_PP1:sto}
\end{figure}

\begin{figure}
    \centering
    \begin{subfigure}{\sizefig\textwidth}
        \centering
        \includegraphics[width=1\textwidth]{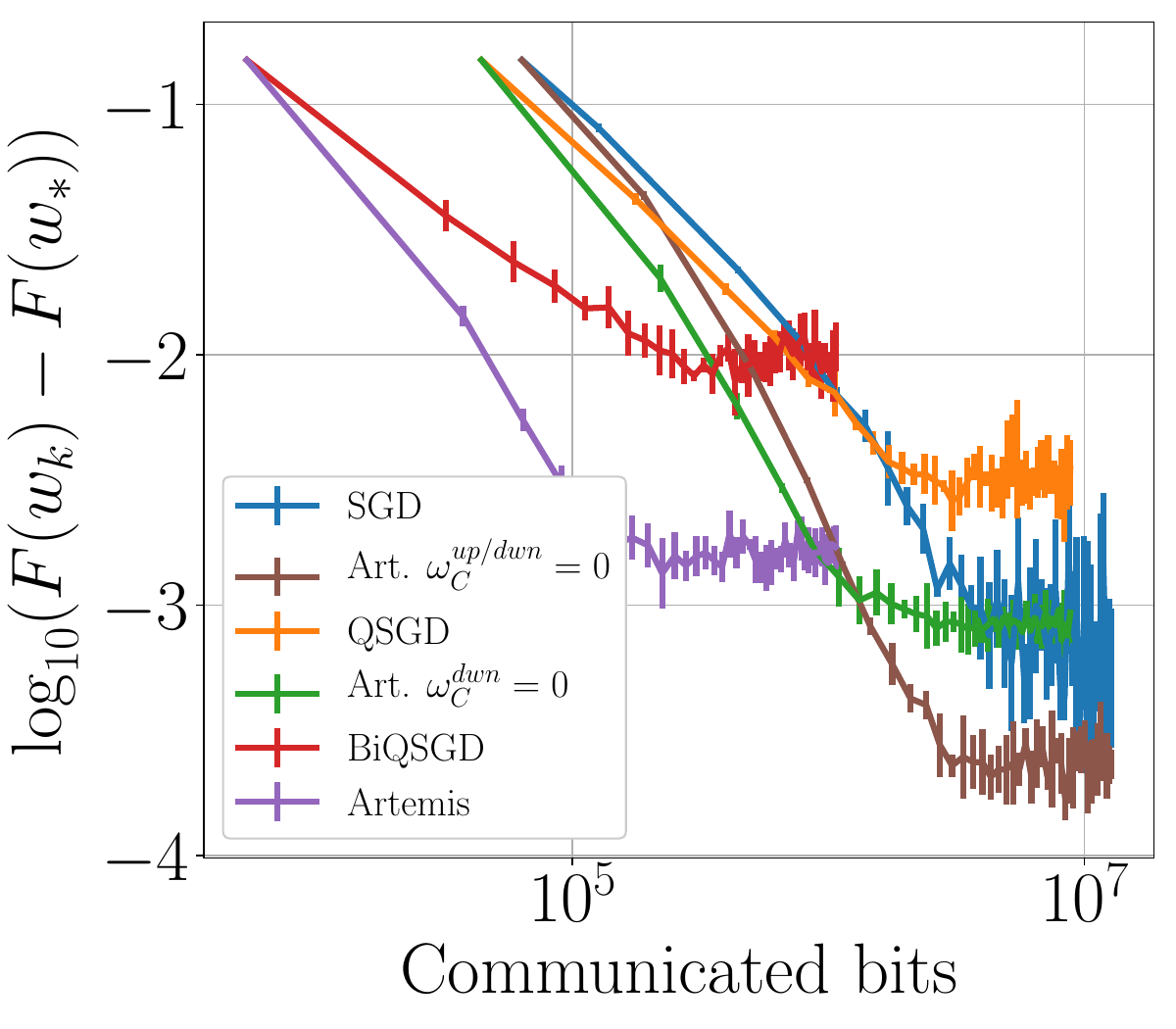}
        \caption{Qauntum\gs}
    \end{subfigure}
    \begin{subfigure}{\sizefig\textwidth}
        \centering
        \includegraphics[width=1\textwidth]{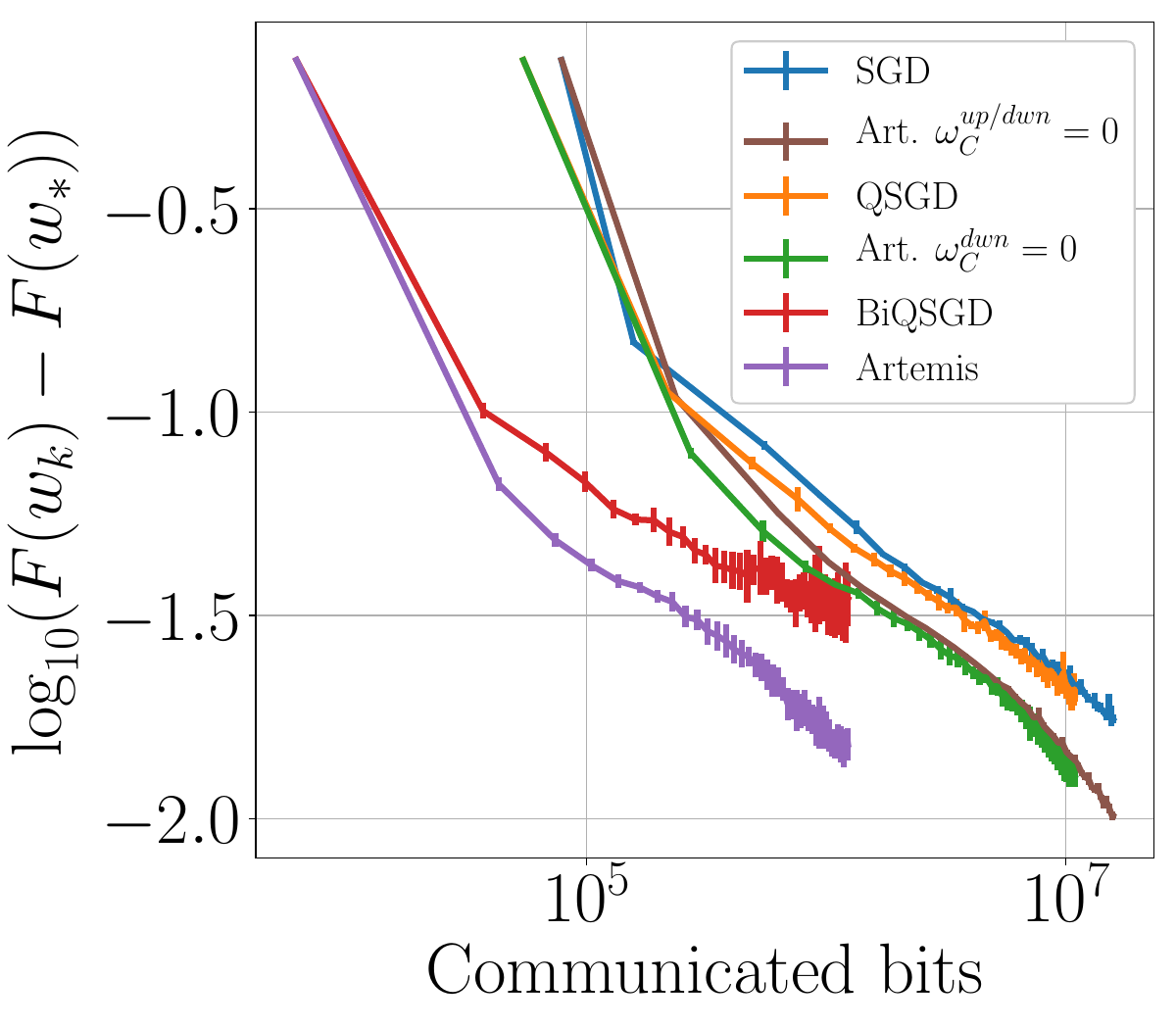}
        \caption{Superconduct\gs}
    \end{subfigure}
    \caption[Real dataset]{\textbf{Partial participation, stochastic regime - PP2.} $\sigmstar \neq 0$, $N=20$ workers, $p=0.5$,  $b >1$ ($150$ iter.). 
    \vspace{-1em}}
    \label{app:fig:real_dataset_PP2:sto}
\end{figure}

\subsubsection{Optimized step size}
\label{app:subsect:opt_step_size}

In this section, we want to address the issue of the optimal step size.
On \Cref{fig:app:gamma_opt} we plot the minimal loss after $150$ iterations for each of the $5$ algorithms. We can see that algorithms with memory clearly outperform those without. Then, on \Cref{fig:app:artemis_gamma} we present the loss of \Artemis~after $150$ iterations for various step size: $\frac{N = 20}{2L}$, $\frac{5}{L}$, $\frac{2}{L}$, $\frac{1}{L}$, $\frac{1}{2L}$, $\frac{1}{4L}$, $\frac{1}{8L}$ and $\frac{1}{16L}$. This helps to understand which step size should be taken to obtain the best accuracy after $k$ in $\llbracket 1, 150 \rrbracket$ iterations. Finally, on \Cref{fig:app:gamma_opt_for_each_algo}, we plot the loss obtained with the optimal step size $\gamma_{opt}$ of each algorithms (found with \Cref{fig:app:gamma_opt}) w.r.t the number of communicated bits. 

On \Cref{fig:app:gamma_opt}, it is interesting to note that the memory allows to increase the maximal step size. So, the optimal step size is $\gamma_{opt} = \frac{1}{L}$ for \Artemis~, but is $\gamma_{opt} = \frac{1}{2L}$ for \texttt{BiQSGD}. 

We plot the loss of \Artemis~after $150$ iterations for different step size on \Cref{fig:app:artemis_gamma}. As stressed by \Cref{fig:app:gamma_opt}, after $150$ iterations, the best accuracy for both datasets is indeed obtained with $\gamma_{opt} = \frac{1}{L}$. And we observe that (as for Vanilla \texttt{SGD}), the optimal step size of \Artemis~decreases with the number of iterations (e.g., for \textit{quantum}, it is $1/L$ before 50 iterations and $1/2L$ after). This is consistent with \Cref{thm:cvgce_artemis}.

\Cref{fig:app:gamma_opt_for_each_algo} plots the loss of each algorithm obtained with its optimal step size $\gamma$ i.e. the step size that attains the lowest error after $150$ iterations. For instance $\gamma = \frac{1}{L}$ for \Artemis, but $\gamma = \frac{2}{L}$ for \texttt{SGD}. For both \textit{superconduct} and \textit{quantum} datasets, taking the optimal step size leads \Artemis~to superior performance than other variants w.r.t. both accuracy and number of bits.

In conclusion of this subsection, \Cref{fig:app:artemis_gamma,fig:app:gamma_opt,fig:app:gamma_opt_for_each_algo} allow to conclude on the significant impact of memory in a non-i.i.d.~settings, and to claim that bidirectional compression with memory is by far superior (up to a threshold) to the four other algorithm: \texttt{SGD, QSGD, Diana} and \texttt{BiQSGD}. 

\begin{figure}
    \centering
    \begin{subfigure}{\sizefig \textwidth}
        \centering
        \includegraphics[width=1\textwidth]{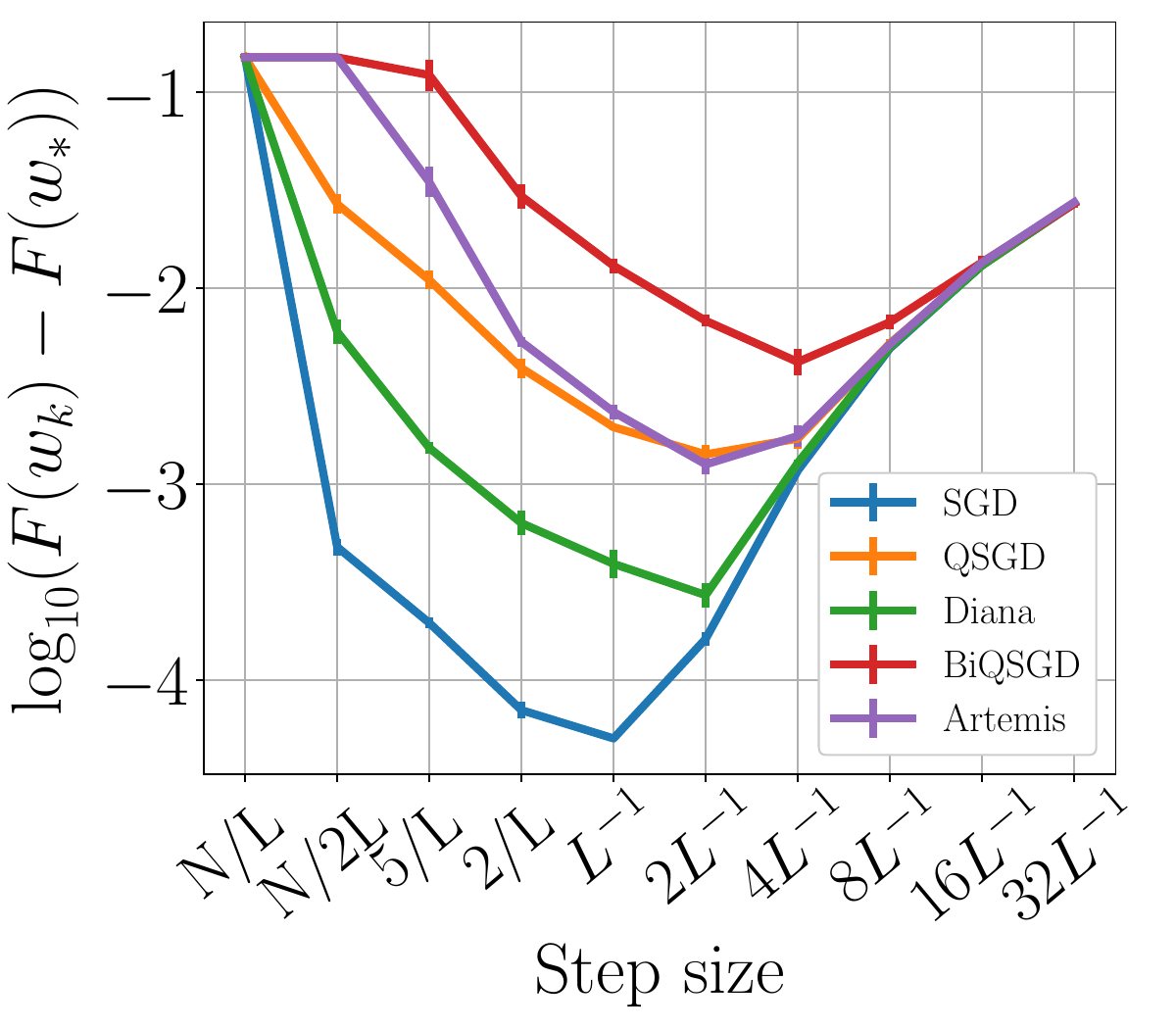}
        \caption{Quantum\gs}
    \end{subfigure}
    \begin{subfigure}{\sizefig \textwidth}
        \centering
        \includegraphics[width=1\textwidth]{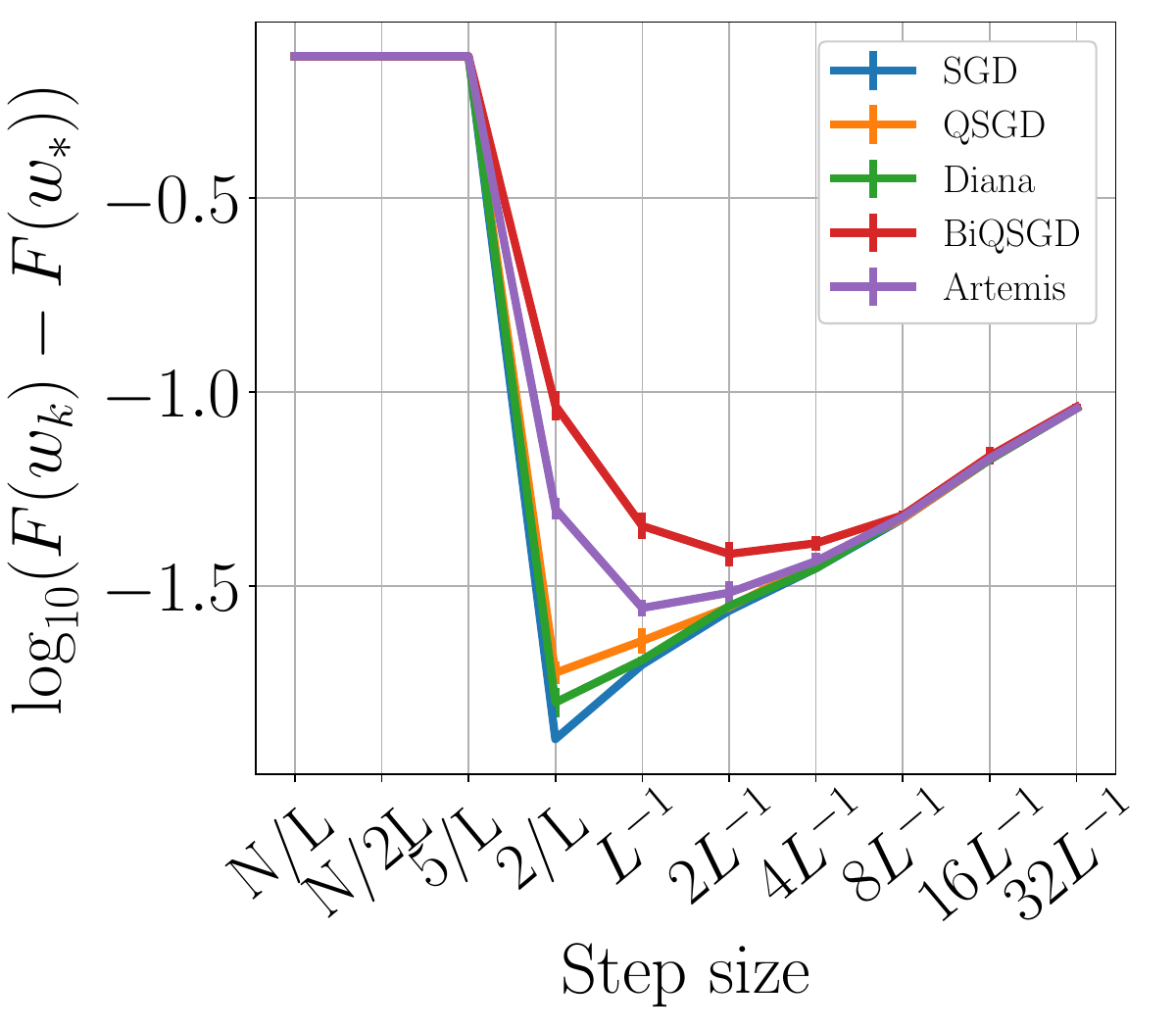} 
        \caption{Superconduct\gs}
    \end{subfigure}\hfill
    \caption[Figures of the main]{\textbf{Searching for the optimal step size $\gamma_{opt}$ for each algorithm.}  X-axis - value on step size, Y-axis - minimal loss after running $250$ iterations \gs\gs\gs}
    \label{fig:app:gamma_opt}
\end{figure}

\begin{figure}
    \centering
    \begin{subfigure}{\sizefig \textwidth}
        \centering
        \includegraphics[width=1\textwidth]{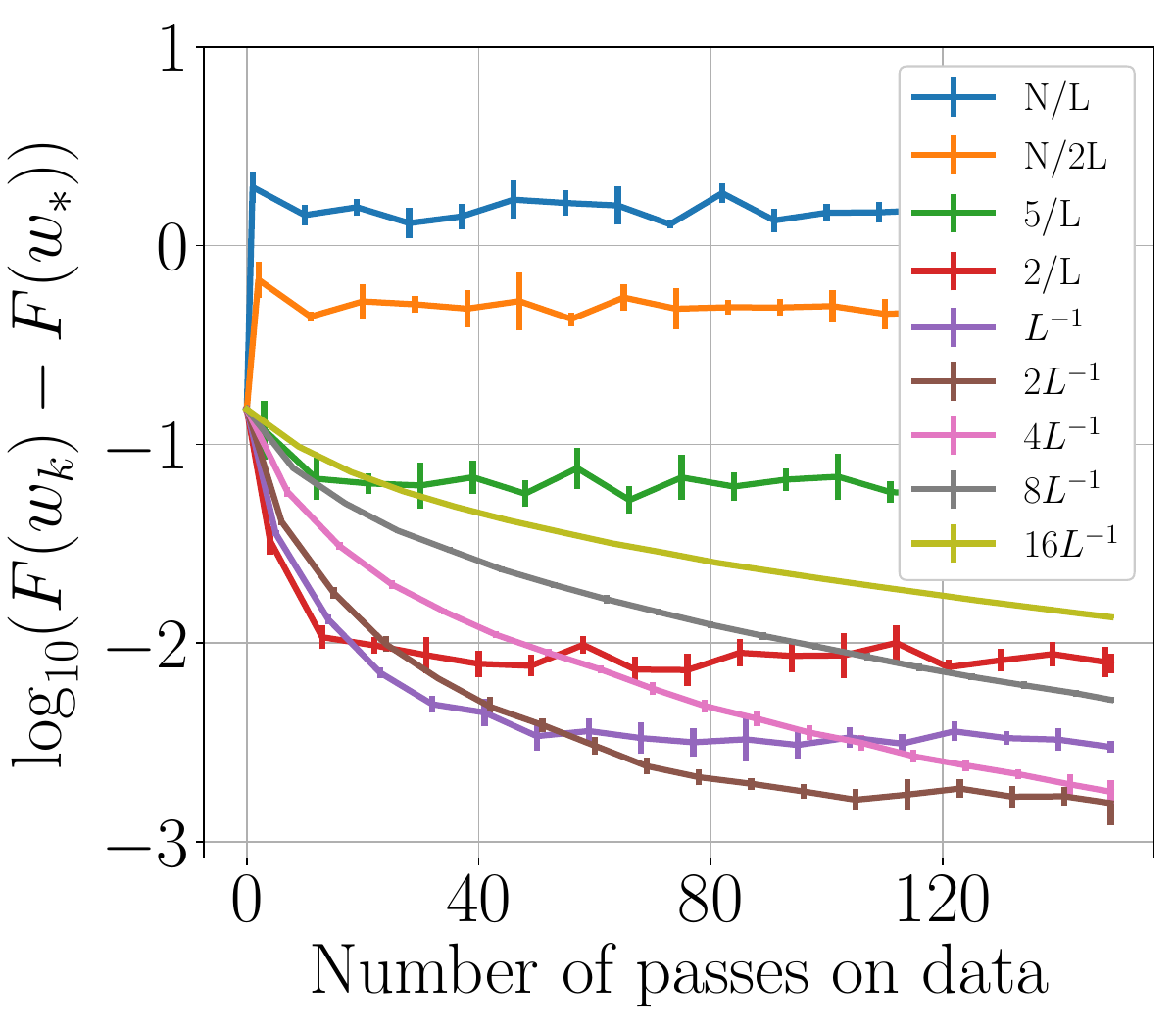}
        \caption{Quantum\gs}
    \end{subfigure}
    \begin{subfigure}{\sizefig \textwidth}
        \centering
        \includegraphics[width=1\textwidth]{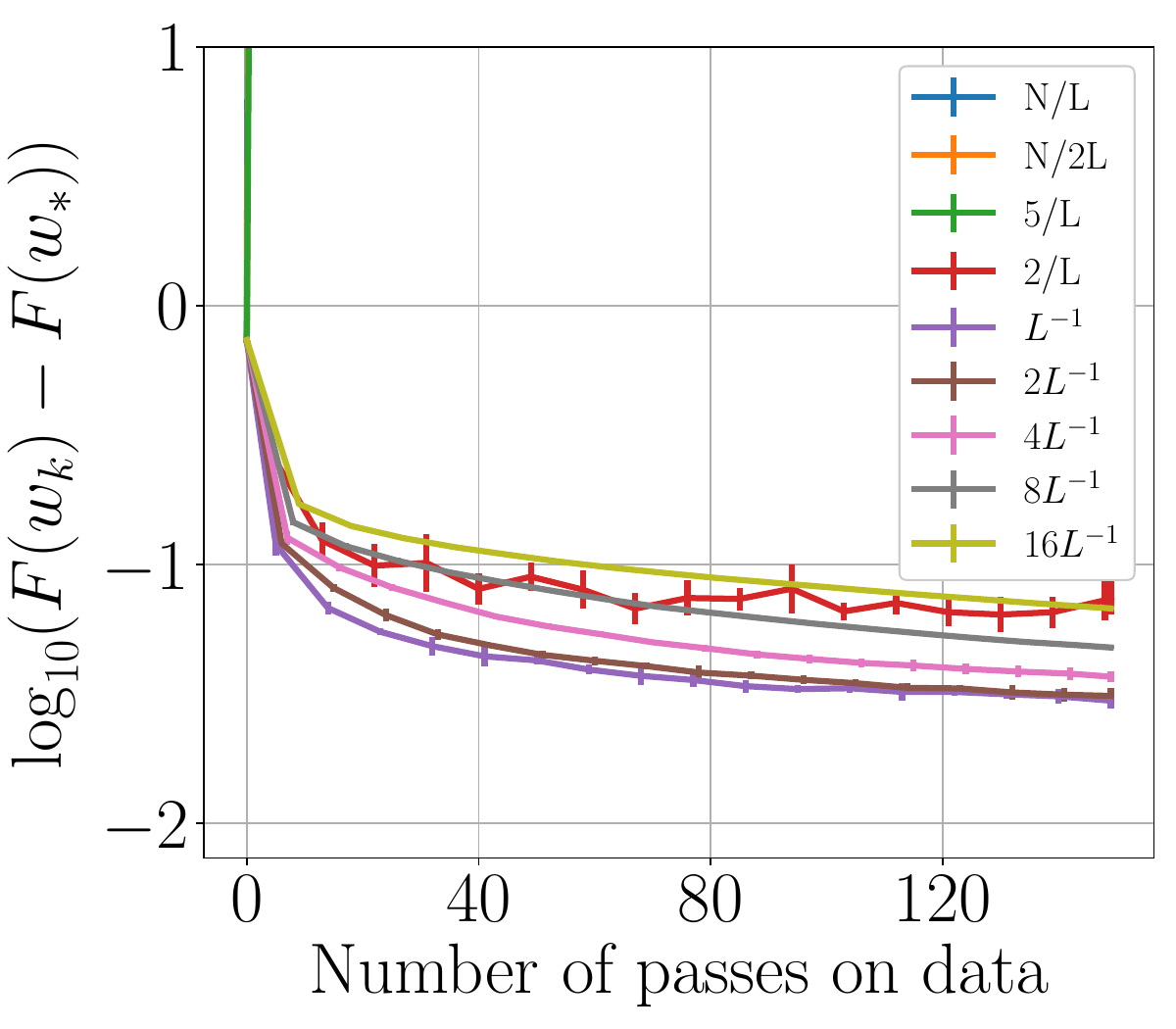}
        \caption{Superconduct\gs}
    \end{subfigure}\hfill
    \caption[Figures of the main]{Loss w.r.t.~step size $\gamma$. \gs\gs\gs}
    \label{fig:app:artemis_gamma}
\end{figure}

\begin{figure}
    \centering
    \begin{subfigure}{\sizefig \textwidth}
        \centering
        \includegraphics[width=1\textwidth]{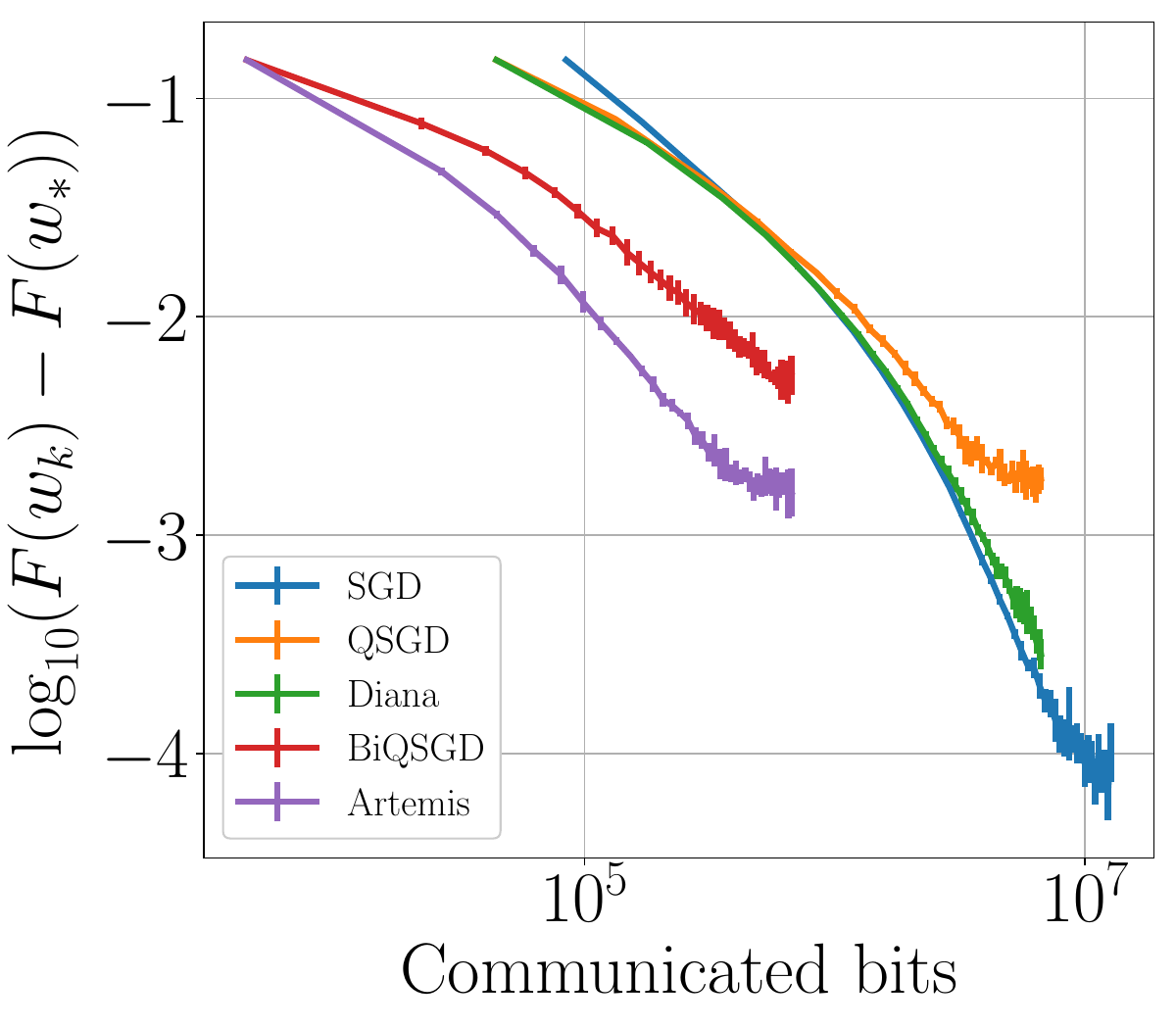}
        \caption{Quantum\gs}
    \end{subfigure}
    \begin{subfigure}{\sizefig \textwidth}
        \centering
        \includegraphics[width=1\textwidth]{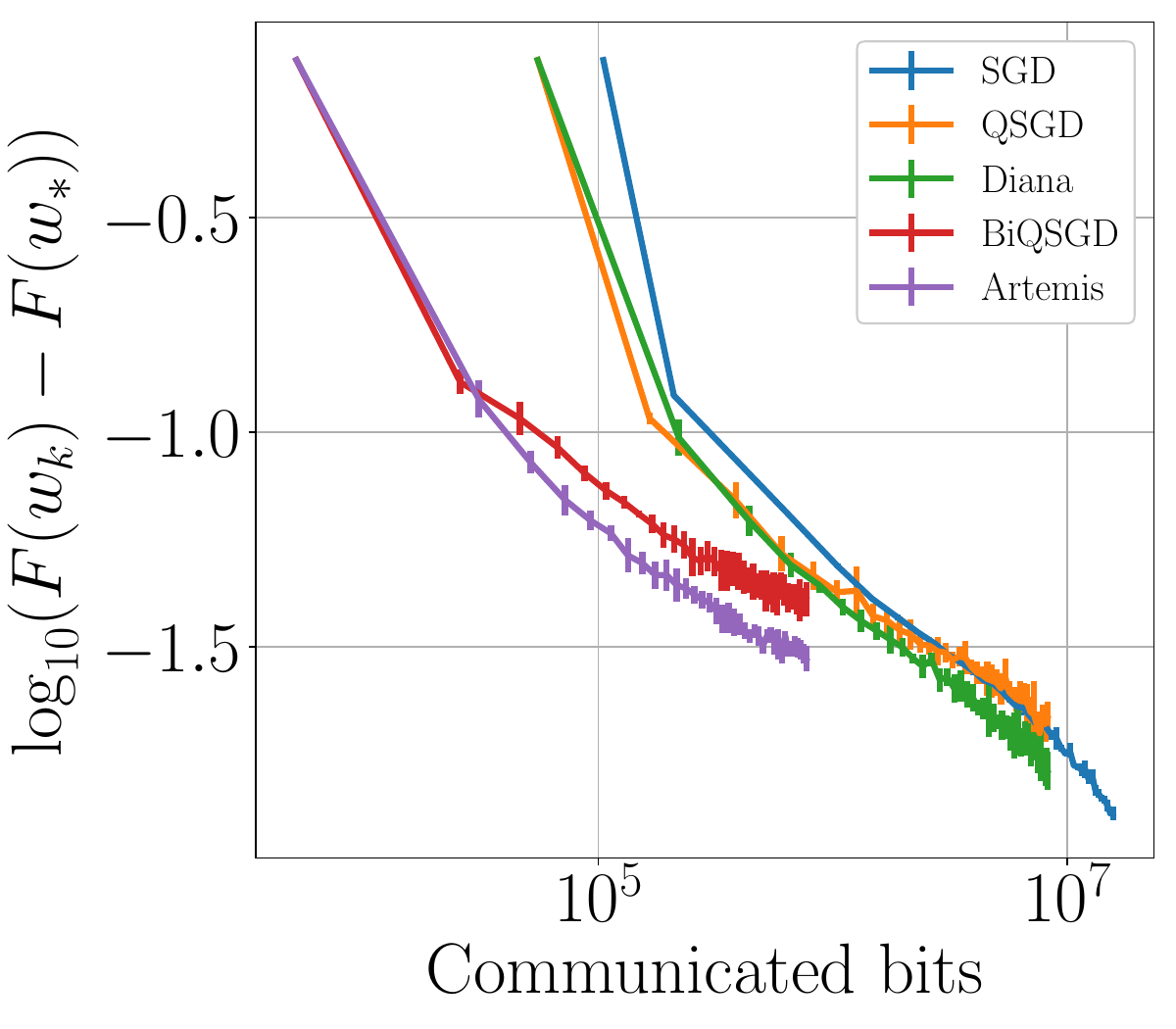} 
        \caption{Superconduct \gs\gs\gs}
    \end{subfigure}\hfill
    \caption[Figures of the main]{\textbf{Optimal step size for each of the algorithms.} X-axis in \# bits.}
    \label{fig:app:gamma_opt_for_each_algo}
\end{figure}

\subsubsection{Comparing \Artemis~with other existing algorithms}
\label{app:subsec:artemis_vs_existing}

On \Cref{fig:artemis_vs_existing} we compare \Artemis~with other existing algorithms: \texttt{FedSGD}, \texttt{FedPAQ}, \Diana, \Dore~and \texttt{Double-Squeeze}. We take $\gamma = 1/(2L)$ because otherwise \texttt{FedSGD} and \texttt{FedPAQ} diverge. These two algorithms present worse performance because they have not been designed for non-i.i.d.~datasets.

We can observe that \texttt{Double-Squeeze} (which only uses error-feedback) is outperformed by \Artemis. Besides, we observe that \texttt{Dore} (which combines this mechanism with memory) has identical rate of convergence than \Artemis. It underlines that for unbiased operators of compression, \textbf{the enhancement comes from the memory and not from the error-feedback}. 

\texttt{FedPAQ} (unidirectional compression) has a very fast convergence during first iterations, but then saturates at a level higher than for \Artemis-like algorithms. \texttt{FedSGD} (no compression) presents a convergence's rate worse that vanilla \texttt{SGD} because it does not correctly handle heterogeneous datasets.

\begin{figure}[H]
    \centering
    \begin{subfigure}{0.22 \textwidth}
        \centering
        \includegraphics[width=1\textwidth]{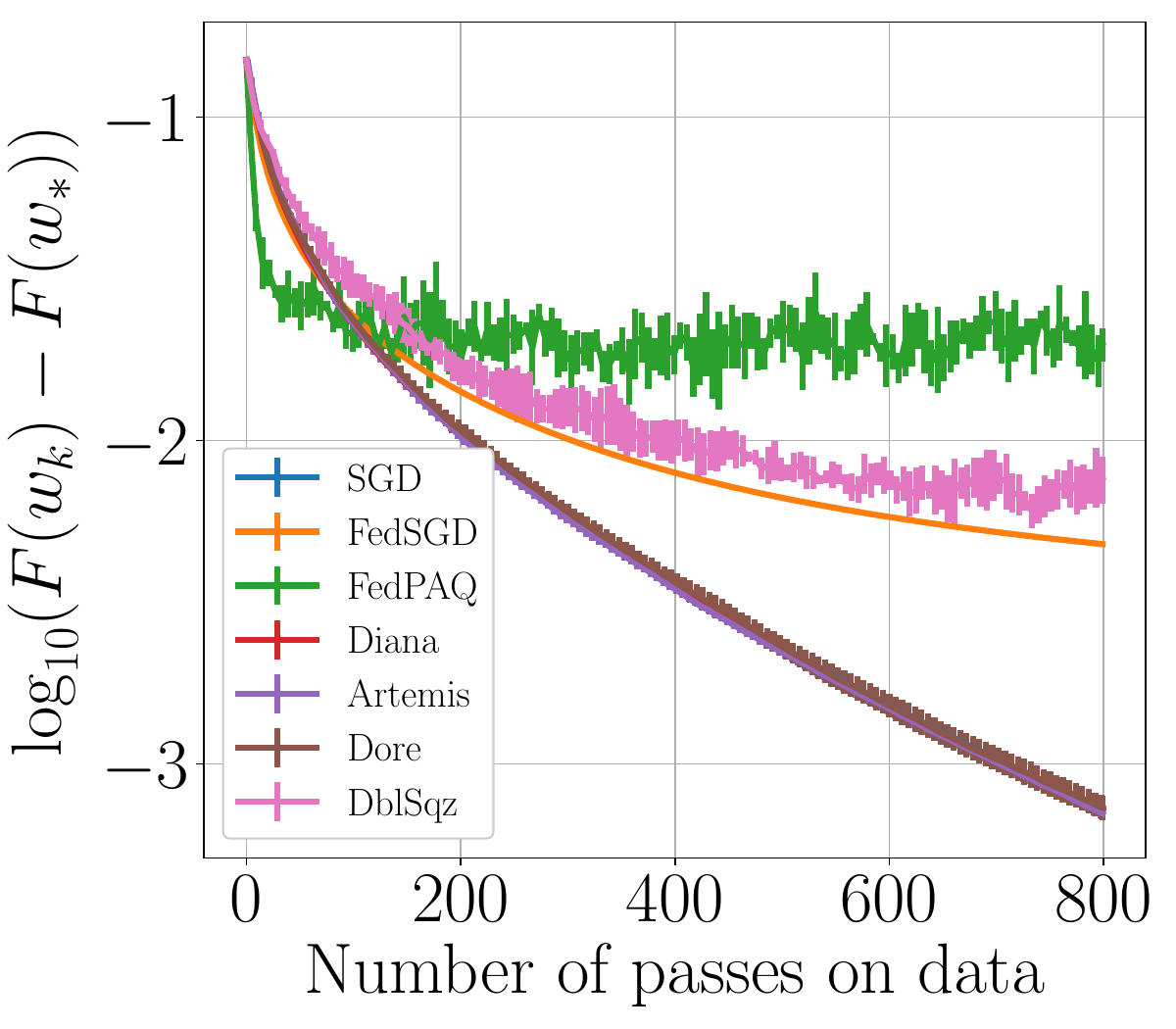}
        \caption{Superconduct\gs}
        \label{fig:artemis_vs_existing:quantum_it}
    \end{subfigure}
    \begin{subfigure}{0.22 \textwidth}
        \centering
        \includegraphics[width=1\textwidth]{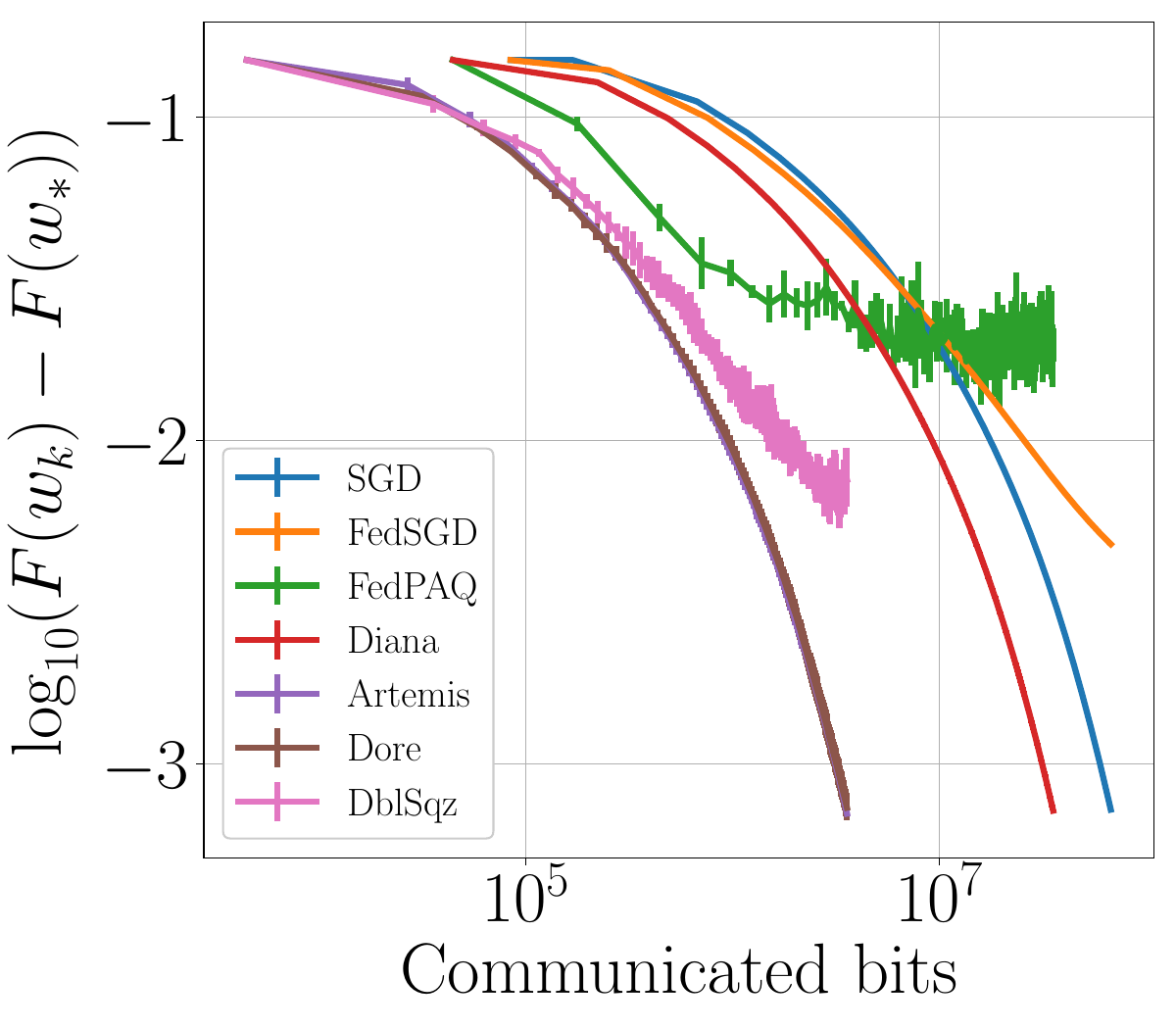}
        \caption{Superconduct\gs}
        \label{fig:artemis_vs_existing:quantum_bits}
    \end{subfigure}
    \begin{subfigure}{0.22 \textwidth}
        \centering
        \includegraphics[width=1\textwidth]{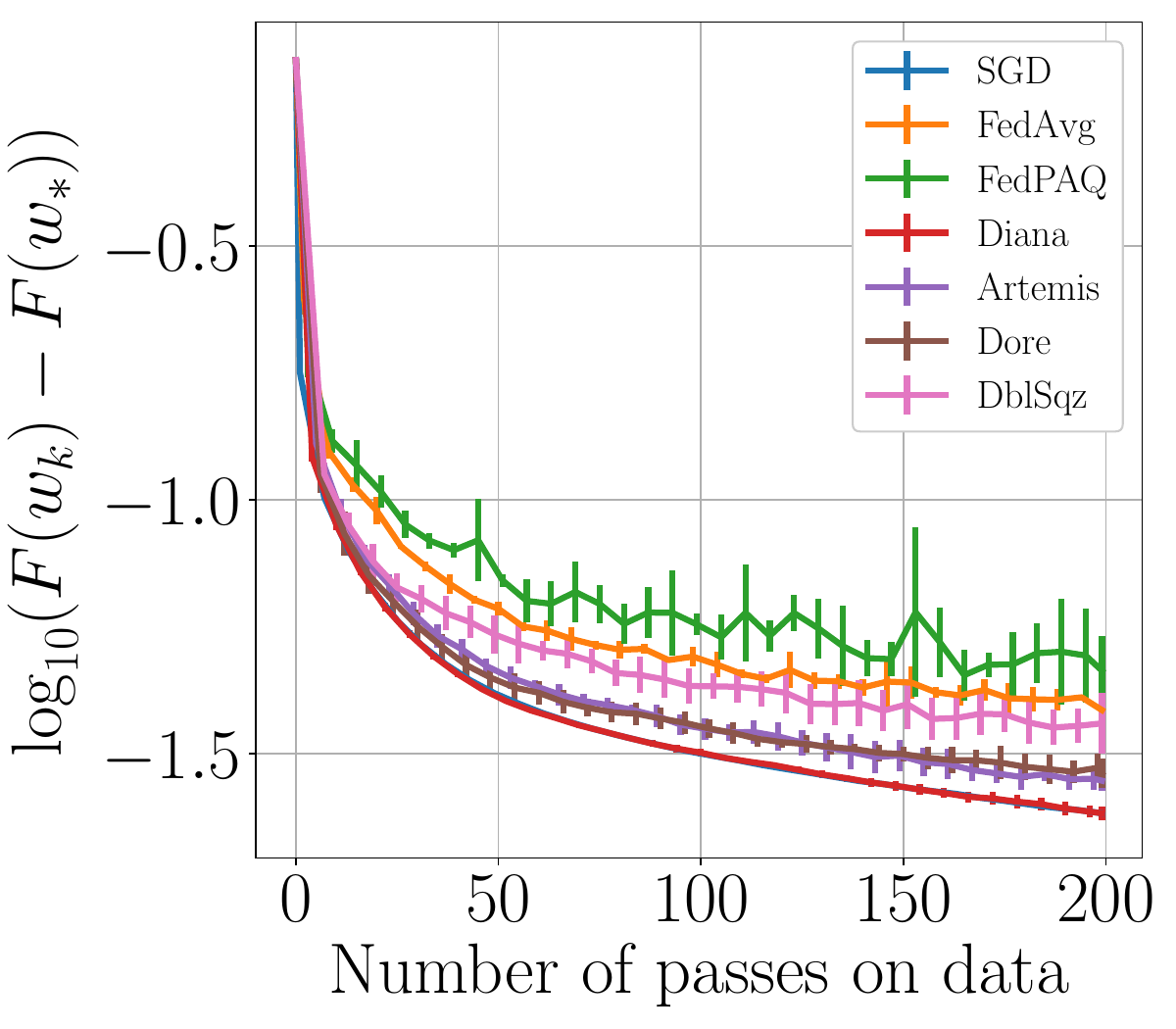}
        \caption{Quantum\gs}
        \label{fig:artemis_vs_existing:superconduct_it}
    \end{subfigure}
    \begin{subfigure}{0.22 \textwidth}
        \centering
        \includegraphics[width=1\textwidth]{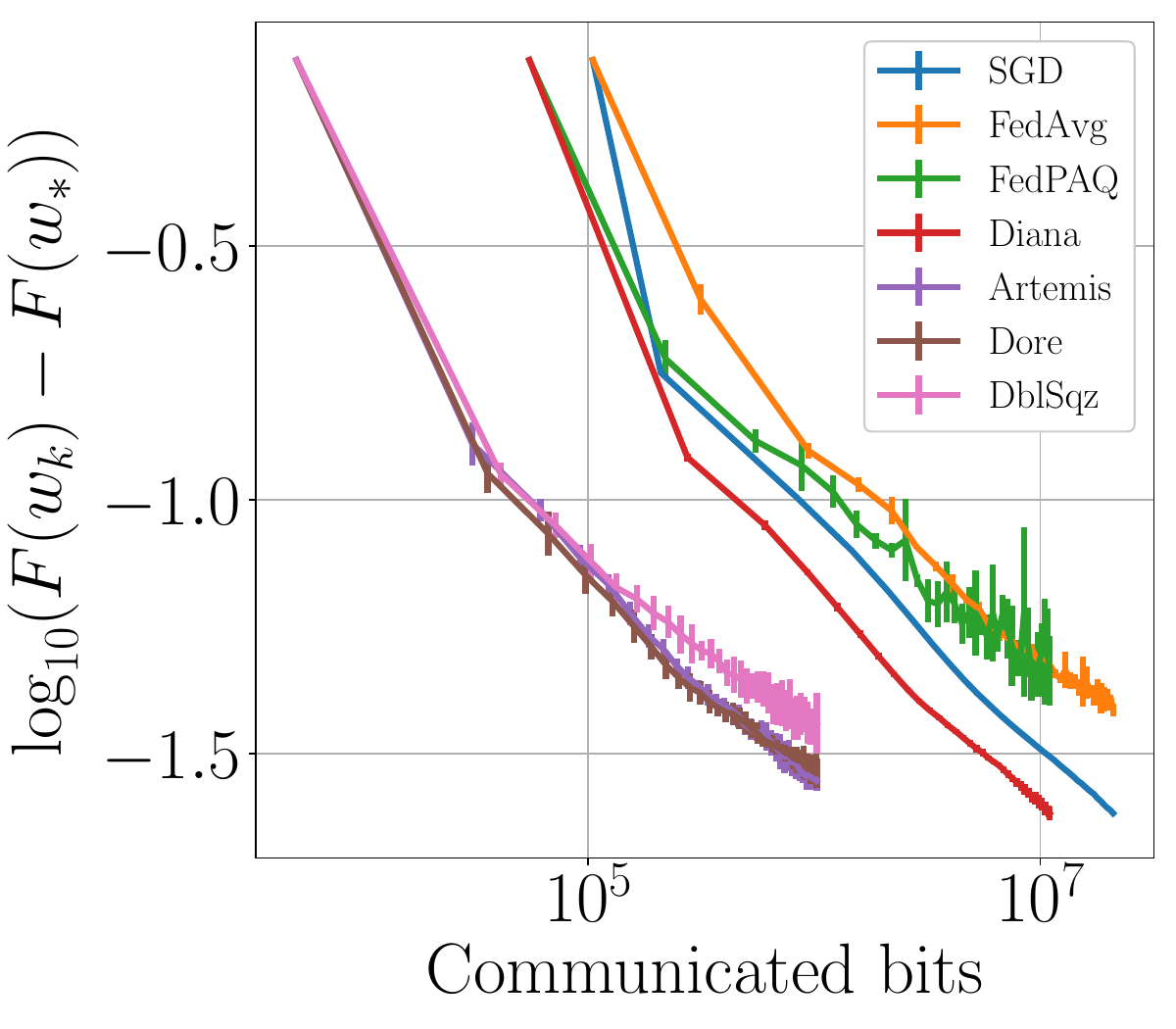}
        \caption{Quantum\gs}
        \label{fig:artemis_vs_existing:superconduct_bits}
    \end{subfigure}
    \caption[Real dataset]{\textbf{Artemis compared to other existing algorithms.} $\gamma = 1/(2L)$, X-axis in \# epoch or in \# bits. }
    \label{fig:artemis_vs_existing}
\end{figure}

\subsection{CPU usage and Carbon footprint}
\label{app:subsec:cpu_usage}

As part as a community effort to report the amount of experiments that were performed, we estimated that overall our experiments ran for 220 to 270 hours end to end. We used an Intel(R) Xeon(R) CPU E5-2667 processor with 16 cores. 

 The carbon emissions caused by this work were subsequently evaluated with \texttt{Green Algorithm} built by \citet{lannelongue_green_2020}. It estimates  our computations to generate $30$ to $35$ kg of CO2, requiring $100$ to $125$ kWh. To compare, it corresponds to about $160$ to $200$km by car. This is a relatively moderate impact, matching the goal to keep the experiments for an illustrative purpose.

\section{Technical Results}\label{app:technical}
In this section, we introduce a few technical lemmas that will be used in the proofs. In \Cref{sect:useful_identies}, we give four simple lemmas, while in \Cref{app:lemma_for_proof} we present a lemma which will be invoked in \Cref{app:all_proofs} to demonstrate \Cref{app:thm:without_mem,app:thm:with_mem,thm:doublecompress_avg}. 

\paragraph{Notation.} Let $k^* \in \N$ and $(a_{k+1}^i)_{i=0}^N \in (\WW)^N$ random variables independent of each other and $\Fartif_{k+1}$-measurable, for a $\sigma$-field $\Fartif_{k+1}$ s.t.~$(B_{k}^i)^N_{i=1}$ are independent of $\Fartif_{k+1}$.. Then in all the following demonstration, we note $a_{k+1} = \frac{1}{N} \sum_{i=1}^N a_{k+1}^i$ and $a_{k+1, S_k} = \frac{1}{pN} \sum_{i \in S_k} a_{k+1}^i$. 

The vector $a_{k+1}$ (resp. the $\sigma$-field $\Fartif_{k+1}$) may represent various objects, for instance : $\gwk$, $\gwkhat$, $G_{k+1}$ (resp. $\Fsto_{k+1}$, $\Fupcomp_{k+1}$, $\Fdwncomp_{k+1}$).

\begin{remark}We can add the following remarks on the assumptions.
\begin{itemize}
    \item \cref{asu:partial_participation} can be extended to probabilities depending on the worker $(p_i)_{i\in\llbracket1, N\rrbracket }$.
    \item \Cref{asu:expec_quantization_operator} requires in fact to access a \emph{sequence of i.i.d.~compression operators} $\C_{\up\slash\dwn, k}$ for $k\in \N$ -- but for simplicity, we generally omit the $k$ index. 
    \item \Cref{asu:bounded_noises_across_devices} in fact only requires that for any $i\in\llbracket1, N\rrbracket $,  $\Expec{\|\gwkstar^i - \nabla F_i(w_*)\|^2}{\Fdwncomp_{k}} \leq \frac{(\sigmstar)_i^2}{b}$, and the results then hold for $\sigmstar = \frac{1}{N} \sum_{i=1}^N (\sigmstar)_i^2$. In other words, the bounds does not need to be uniform over workers, only the average truly matters.
\end{itemize}
\end{remark}

\subsection{Useful identities and inequalities}
\label{sect:useful_identies}

\begin{lemma}
\label{lem:two_inequalities}
Let $N \in \N$ and $d \in \N$. For any sequence of vector $(a_i)_{i=1}^N \in \mathbb{R}^d$, we have the following inequalities:
\[
\lnrm \sum_{i=1}^N a_i \rnrm^2 \leq \left( \sum_{i=1}^N \lnrm a_i \rnrm \right)^2 \leq N \sum_{i=1}^N \lnrm a_i \rnrm^2 \,.
\]
\end{lemma}
The first part of the inequality corresponds to the triangular inequality, while the second part is Cauchy's inequality.

\begin{lemma}
\label{lem:inequalities_with_alpha}
Let $\alpha \in [0, 1]$ and $x, y \in (\R^d)^2$, then:
\[
\lnrm \alpha x + (1- \alpha) y \rnrm^2 = \alpha \lnrm x \rnrm^2 + (1- \alpha) \lnrm y \rnrm^2 - \alpha (1-\alpha) \lnrm x - y \rnrm^2 \,.
\]
\end{lemma}

This is a norm's decomposition of a convex combination.

\begin{lemma}
\label{lem:expectation_decomposition}
Let $X$ be a random vector of $\R^d$, then for any vector $x \in \R^d$:
\[
\E \lnrm X - \E X \rnrm^2 = \E \lnrm X - x \rnrm^2 - \lnrm \E X - x \rnrm^2 \,.
\]
\end{lemma}

This equality is a generalization of the well know decomposition of the variance (with $x = 0$).

\begin{lemma}
\label{lem:strongly-cvxe}
If $F: \mathcal{X} \subset  \mathbb{R}^d \rightarrow \mathbb{R}$ is strongly convex, then the following inequality holds:
\[ \forall (x,y) \in \WW, \PdtScl{\nabla F(x) - \nabla F(y)}{x - y} \geq \mu \lnrm x - y \rnrm^2 \,. \]
\end{lemma}

This inequality is a consequence of strong convexity and can be found in \cite[][equation $2.1.22$]{nesterov_introductory_2004}.

\subsection{Lemmas for proof of convergence}
\label{app:lemma_for_proof}

Below are presented technical lemmas needed to prove the contraction of the Lyapunov function for \Cref{app:thm:without_mem,app:thm:with_mem}. In this section we assume that \Cref{asu:partial_participation,asu:expec_quantization_operator,asu:strongcvx,asu:cocoercivity,asu:noise_sto_grad,asu:bounded_noises_across_devices} are verified.
In \Cref{app:subsubsec:lem_with_mem,app:subsubsec:lem_without_mem} we separate lemmas that required only for the case with memory or without.

The first lemma is very simple and straightforward from the definition of $\Delta_k^i$. We remind that $\Delta_k^i$ is the difference between the computed gradient and the memory hold on device $i$. It corresponds to the information which will be compressed and sent from device $i$ to the central server. 

\begin{lemma}[Bounding the compressed term]
\label{lem:bounding_compressed_term}
The squared norm of the compressed term sent by each node to the central server can be bounded as following:
\begin{align*}
    \forall k \in \N\,, \,\forall i \in \llbracket 1 , N \rrbracket\,, \quad \SqrdNrm{ \Delta_k^i } \leq 2 \left( \SqrdNrm{ \gwk^i - h_*^i } + \SqrdNrm{ h_k^i - h_*^i } \right)\,. \\
\end{align*}
\end{lemma}

\begin{proof}

 Let $k \in \N$ and  $i \in \llbracket1, N\rrbracket$, we have by definition:
\begin{align*}
    &\SqrdNrm{ \Delta_k^i } = \SqrdNrm{ \gwk^i - h_k^i } = \SqrdNrm{ (\gwk^i - h_*^i) + (h_*^i - h_k^i) } \,.\\
\end{align*}
Applying \Cref{lem:two_inequalities} gives the expected result.
\end{proof}

Below, we show up a recursion over the memory term $h_k^i$ involving the stochastic gradients. This recursion will be used in \Cref{lem:recursive_inequalities_over_memory}. The existence of recursion has been first shed into light by \citet{mishchenko_distributed_2019}.

\begin{lemma}[Expectation of memory term]
\label{lem:expectation_h_l}
The memory term $h_{k+1}^i$ can be expressed using a recursion involving the stochastic gradient $\gwk^i$:

\begin{align*}
    \forall k \in \N\,,\, \forall i \in \llbracket 1 , N \rrbracket\,, \quad\Expec{h_{k+1}^i}{\Fsto_{k+1}} = (1-\alpha) h_k^i + \alpha \gwk^i  \,.
\end{align*}
\end{lemma}

\begin{proof}
 Let $k \in \N$ and $ i \in \llbracket1, N\rrbracket$. We just need to decompose $h^i_k$ using its definition:

\[
    h_{k+1}^i = h_k^i + \alpha \widehat{\Delta}_k^i
    = h_k^i + \alpha (\gwkhat^i - h_k^i) 
    = (1 - \alpha) h_k^i + \alpha \gwkhat^i \,,
\]

and considering that $\Expec{\gwkhat^i}{\Fsto_{k+1}} = \gwk^i$ (\Cref{prop:uplink_expectation}), the proof is completed.

\end{proof}

In \Cref{lem:before_using_coco}, we rewrite $\SqrdNrm{\gwk}$ and $\SqrdNrm{\gwk - h_*^i}$ to make appears:

\begin{enumerate}
    \item the noise over stochasticity
    \item $\SqrdNrm{\gwk - \gwkstar}$ which is the term on which will later be applied cocoercivity (see \Cref{asu:cocoercivity}).
\end{enumerate}

\Cref{lem:before_using_coco} is required to correctly apply cocoercivity in \Cref{lem:applying_coco}.
\begin{lemma}[Before using co-coercivity]
\label{lem:before_using_coco}
Let $k \in \llbracket 0, K \rrbracket$ and $i \in \llbracket 1, N \rrbracket$. The noise in the stochastic gradients as defined in \Cref{asu:noise_sto_grad,asu:bounded_noises_across_devices} can be controlled as following:
\begin{align}
    &\frac{1}{N} \sum_{i=1}^N\Expec{\SqrdNrm{\gwk^i}}{\Fdwncomp_{k}} \leq \frac{2}{N} \sum_{i=1}^N \left( \Expec{\SqrdNrm{ \gwk^i - \gwkstar^i}}{\Fdwncomp_{k}} + (\ffrac{\sigmstar^2}{b} +B^2) \right) \label{eq:g_i_k} \,,\\
    &\frac{1}{N} \sum_{i=1}^N \Expec{\SqrdNrm{\gwk^i - h_*^i}}{ \Fdwncomp_{k}} \leq \frac{2}{N} \sum_{i=1}^N \left( \Expec{\SqrdNrm{\gwk^i - \gwkstar^i}}{\Fdwncomp_{k}} + \ffrac{\sigmstar^2}{b} \right) \,. \label{eq:g_i_k_h_i_k}
\end{align}

\end{lemma}

\begin{proof}
 Let $k \in \N$.
For \cref{eq:g_i_k}:
\begin{align*}
\SqrdNrm{ \gwk^i } &= \SqrdNrm{ \gwk^i -  \gwkstar^i +  \gwkstar^i} \\
&\quad\leq 2 \left( \SqrdNrm{ \gwk^i -  \gwkstar^i } + \SqrdNrm{ \gwkstar^i} \right) \text{using inequality of \Cref{lem:two_inequalities}.}\\
\end{align*}
Taking expectation with regards to filtration $\Fdwncomp_k$ and using \Cref{asu:noise_sto_grad,asu:bounded_noises_across_devices} gives the first result.

For \cref{eq:g_i_k_h_i_k}, we use \Cref{lem:two_inequalities} and we write:
\begin{align*}
\SqrdNrm{ \gwk^i - h_*^i } &= \SqrdNrm{ \left(\gwk^i - \gwkstar^i \right) + \left(\gwkstar^i - \nabla F_i(w_*) \right) } \\
&\leq  2(\SqrdNrm{ \gwk^i - \gwkstar^i } + \SqrdNrm{ \gwkstar^i - \nabla F_i(w_*) }) \,.\\
\end{align*}

Taking expectation, we have:
\begin{align*}
    \Expec{\SqrdNrm{ \gwk^i - h_*^i }}{\Fdwncomp_{k}} &\leq 2 \left(\Expec{\SqrdNrm{ \gwk^i - \gwkstar^i }}{\Fdwncomp_{k}} + \Expec{\SqrdNrm{ \gwkstar^i - \nabla F_i(w_*) }}{\Fdwncomp_{k}} \right)\\
    &\leq 2\left(\Expec{\SqrdNrm{ \gwk^i - \gwkstar^i}}{\Fdwncomp_{k}} + \ffrac{\sigmstar^2}{b}\right) \text{\quad using \Cref{asu:noise_sto_grad}.}
\end{align*}
\end{proof}

Demonstrating that the Lyapunov function is a contraction requires to bound $\SqrdNrm{ \gwks }$ which needs to control each term $(\SqrdNrm{ \gwks^i })_{i=1}^N$ of the sum. This leads to invoke smoothness of $F$ (consequence of \Cref{asu:cocoercivity}).
\begin{lemma}
\label{lem:bounding_g_k+1}
Regardless if we use memory, we have the following bound on the squared norm of the gradient, for all $k$ in $ \N$:
\begin{align*}
     \quad \Expec{ \SqrdNrm{ \gwks } }{\Ftotal}&\leq \ffrac{1}{pN^2}\sum_{i=1}^N\Expec{\SqrdNrm{\gwk^i -h_*^i}}{\Fdwncomp_{k}} + L \PdtScl{\nabla F(w_k)}{w_k - w_*}\,.
 \end{align*}
 \end{lemma}
 
 \begin{proof}
 
Let $k \in \N$,
\begin{align*}
    \SqrdNrm{ \gwks } &=  \SqrdNrm{  \frac{1}{pN} \sum_{i \in S_k} \gwk^i }  \\
    &= \SqrdNrm{  \frac{1}{pN} \sum_{i \in S_k} \left(\gwk^i - \nabla F_i(w_k) \right) + \frac{1}{pN} \sum_{i \in S_k} \nabla F_i(w_k) } \,.
\end{align*}

Now taking conditional expectation w.r.t $\Fdwncompnotsamp$ (including $\Fsamp_k$ in the $\sigma$-field allows to not consider the randomness associated to the device sampling):
 \begin{align*}
    \Expec{ \SqrdNrm{ \gwks }}{\Fdwncompnotsamp} &= \Expec{ \SqrdNrm{ \frac{1}{pN} \sum_{i \in S_k} \gwk^i - \nabla F_i(w_k) + \frac{1}{pN} \sum_{i \in S_k} \nabla F_i(w_k) }}{\Fdwncompnotsamp} \,.\\
\end{align*}
Expanding this squared norm:
\begin{align*}
    \Expec{ \SqrdNrm{ \gwks }}{\Fdwncompnotsamp} &= \Expec{ \SqrdNrm{ \frac{1}{pN} \sum_{i \in S_k} \gwk^i - \nabla F_i(w_k) }}{\Fdwncompnotsamp} \\
    &\qquad+ 2\Expec{\PdtScl{\frac{1}{pN} \sum_{i \in S_k}  \gwk^i - \nabla F_i(w_k)}{ \frac{1}{pN} \sum_{i \in S_k} \nabla F_j(w_k)}}{\Fdwncompnotsamp} \\
    &\qquad+ \Expec{\SqrdNrm{ \frac{1}{pN} \sum_{i \in S_k} \nabla F_i(w_k) }}{\Fdwncompnotsamp} \,.
\end{align*}
Moreover,  $\forall i, j \in \llbracket1, N\rrbracket^2, \Expec{ \PdtScl{\gwk^i - \nabla F_i(w_k)}{ \nabla F_j(w_k)}}{\Fdwncompnotsamp} = 0$ and $\nabla F(w_k)$ is $\Fdwncompnotsamp$-measurable:
 \begin{align}
 \label{eq:before_double_decompo}
     \Expec{ \SqrdNrm{ \gwks }}{\Fdwncompnotsamp} &\leq \Expec{ \SqrdNrm{ \frac{1}{pN} \sum_{i \in S_k} \gwk^i - \nabla F_i(w_k) }}{\Fdwncompnotsamp} + \SqrdNrm{ \nabla F_{s_k}(w_k) } \,.
\end{align}
To compute $\SqrdNrm{\nabla F_{S_k} (w_k)}$, we apply cocoercivity (see \Cref{asu:cocoercivity}) and next we take expectation w.r.t $\sigma$-algebra $\Ftotal$:
\[
\Expec{\SqrdNrm{\nabla F_{S_k} (w_k)}}{\Ftotal} = L \PdtScl{\Expec{\nabla F_{S_k} (w_k)}{\Ftotal}}{w_k - w_*} = L \PdtScl{\nabla F (w_k)}{w_k - w_*}
\]
Now, for sake of clarity we note $\Pi = \SqrdNrm{ \frac{1}{pN} \sum_{i \in S_k} \gwk^i - \nabla F_i(w_k) }$, then:\\
\begin{align*}
    \Expec{ \Pi }{\Fdwncompnotsamp} &= \frac{1}{p^2 N^2} \sum_{i \in S_k} \Expec{\SqrdNrm{ \gwk^i - \nabla F_i(w_k) }}{\Fdwncompnotsamp} \\
    &\qquad+ \frac{1}{p^2 N^2} \sum_{i,j \in S_k / i \neq j} \underbrace{\Expec{\PdtScl{\gwk^i - \nabla F_i(w_k) }{\gwk^j - \nabla F_j(w_k)}}{\Fdwncompnotsamp}}_{=0 \text{ by independence of } (\gwk^i)_{i=0}^N}  \\
    &= \frac{1}{p^2 N^2} \sum_{i \in S_k} \Expec{\SqrdNrm{(\gwk^i - \nabla F_i(w_*)) + (\nabla F_i(w_*) - \nabla F_i(w_k) )}}{\Fdwncompnotsamp} \,. \\
\end{align*}
Developing the squared norm a second time:
\begin{align*}
    \Expec{\Pi}{\Fdwncompnotsamp} &= \frac{1}{p^2 N^2} \sum_{i \in S_k} \Expec{\SqrdNrm{\gwk^i - \nabla F_i(w_*)}}{\Fdwncompnotsamp} \\
    &\qquad + \frac{2}{p^2 N^2} \sum_{i \in S_k}  \Expec{\PdtScl{\gwk^i - \nabla F_i(w_*)}{\nabla F_i(w_*) - \nabla F_i (w_k)}}{\Fdwncompnotsamp} \\
    &\qquad + \frac{1}{p^2 N^2}\sum_{i \in S_k}\SqrdNrm{\nabla F_i (w_k) - \nabla F_i(w_*),}
\end{align*}
Then,
\begin{align*}
    \Expec{\Pi}{\Fdwncompnotsamp} &= \frac{1}{p^2 N^2} \sum_{i \in S_k} \Expec{\SqrdNrm{\gwk^i - \nabla F_i(w_*)}}{\Fdwncompnotsamp} \\
    &\qquad - \frac{2}{p^2 N^2} \sum_{i \in S_k} \PdtScl{\nabla F_i (w_k) - \nabla F_i(w_*)}{\nabla F_i (w_k) - \nabla F_i(w_*) } \\
    &\qquad+ \frac{1}{p^2 N^2}\sum_{i \in S_k}\SqrdNrm{\nabla F_i (w_k) - \nabla F_i(w_*)} \\
    &= \frac{1}{p^2 N^2} \sum_{i \in S_k} \Expec{\SqrdNrm{\gwk^i - \nabla F_i(w_*)}}{\Fdwncompnotsamp} - \SqrdNrm{\nabla F_i (w_k) - \nabla F_i(w_*)} \,
\end{align*}

applying cocoercivity (\Cref{asu:cocoercivity}):
\begin{align}
\label{lem:eq:pi}
    \Expec{\Pi}{\Fdwncompnotsamp} &\leq \frac{1}{p^2 N^2} \sum_{i \in S_k} \Expec{\SqrdNrm{\gwk^i - \nabla F_i(w_*)}}{\Fdwncompnotsamp} - L \PdtScl{\nabla F_i(w_k) - \nabla F_i(w_*)}{w_k - w_*} \,.
\end{align}

Now we consider the randomness associated to device sampling.  Remember that because $\Ftotal \subset \Fdwncompnotsamp$, we have $\Expec{\Pi}{\Ftotal} = \Expec{\Expec{\Pi}{\Fdwncompnotsamp}}{\Ftotal}$. Thus, we consider now $\Pi$ w.r.t. the $\sigma$-field $\Ftotal$

\textit{Handling first term of \cref{lem:eq:pi}:}
\begin{align*}
    \frac{1}{p^2 N^2} \sum_{i \in S_k} \Expec{\SqrdNrm{\gwk^i - \nabla F_i(w_*)}}{\Ftotal} &= \frac{1}{p^2 N^2} \sum_{i=1}^N B_k^i \Expec{\SqrdNrm{\gwk^i - \nabla F_i(w_*)}}{\Ftotal} \\
    &= \frac{1}{p N^2} \sum_{i=1}^N \Expec{\SqrdNrm{\gwk^i - \nabla F_i(w_*)}}{\Ftotal} \,.
\end{align*}
 
\textit{Handling second term of \cref{lem:eq:pi}}:
\begin{align*}
    L \PdtScl{\frac{1}{p^2 N^2}\sum_{i \in S_k} \nabla F_i(w_k) - \nabla F_i(w_*)}{w_k - w_*} = L \PdtScl{\frac{1}{p^2 N^2}\sum_{i=0}^N \left(\nabla F_i(w_k) - \nabla F_i(w_*) \right) B_k^i}{w_k - w_*} \,.
\end{align*}

Taking expectation w.r.t $\sigma$-algebra $\Ftotal$:
\begin{align*}
    &\Expec{L \PdtScl{\frac{1}{p^2 N^2}\sum_{i \in S_k} \nabla F_i(w_k) - \nabla F_i(w_*)}{w_k - w_*}}{\Ftotal} \\
    &\qquad= L \PdtScl{\frac{1}{p N^2} \sum_{i=0}^N \nabla F_i(w_k) - \nabla F_i(w_*)}{w_k - w_*} \\
    &\qquad= \frac{L}{pN} \PdtScl{\nabla F(w_k) - \nabla F_i(w_*)}{w_k - w_*} \,.
\end{align*}

Injecting this in \cref{lem:eq:pi}:
\begin{align*}
    \Expec{\Pi}{\Ftotal} &\leq \frac{1}{p N^2} \sum_{i=1}^N \Expec{\SqrdNrm{\gwk^i - \nabla F_i(w_*)}}{\Ftotal} - \frac{L}{pN} \PdtScl{\nabla F(w_k)}{w_k - w_*} \,.\\
\end{align*}

Recall that we note $h_*^i = \nabla F_i(w_*)$, returning to \cref{eq:before_double_decompo} and invoking again cocoercivity:
\begin{align*}
     \Expec{ \SqrdNrm{ \gwk }}{\Ftotal} &\leq \frac{1}{pN^2} \sum_{i=1}^N \Expec{\SqrdNrm{\gwk^i - h_*^i}}{\Ftotal} + (1 - \frac{1}{pN}) L \PdtScl{\nabla F(w_k)}{w_k - w_*} \,, 
\end{align*}
which we simplify by considering that:
\begin{align*}
     \Expec{ \SqrdNrm{ \gwk }}{\Ftotal} &\leq \frac{1}{pN^2} \sum_{i=1}^N \Expec{\SqrdNrm{\gwk^i-h_*^i}}{\Ftotal} + L \PdtScl{\nabla F(w_k)}{w_k - w_*} \,.
\end{align*}
 \end{proof}

\subsubsection{Lemmas for the case without memory}
\label{app:subsubsec:lem_without_mem}

In this subsection, we give lemmas that are used only to demonstrate \Cref{app:thm:without_mem} (i.e. without memory).

\Cref{lem:gkhat_Sk_to_gkhat} helps to pass for $k$ in $\N$ from $\gwkhats$ to $(\gwkhat^i)_{i=1}^N$, this lemma will allow to invoke \Cref{lem:expec_gwkhat}.
\begin{lemma}
\label{lem:gkhat_Sk_to_gkhat}
In the case without memory, we have the following bound on the squared norm of the compressed gradient (randomly sampled), for all $k$ in $\N$:
\[
\Expec{\SqrdNrm{\gwkhats}}{\Fupcomp_{k+1}} = \ffrac{(1-p)}{p N^2} \sum_{i=0}^N \SqrdNrm{\gwkhat^i} + \SqrdNrm{\gwkhat} \,.
\]
\end{lemma}
\begin{proof}
The proof is quite straightforward using the expectation and the variance of $\gwkhats$ computed in \Cref{prop:exp_samp,prop:variance_samp}. We just need to decompose as following to easily obtain the result:
\begin{align*}
    \Expec{\SqrdNrm{\gwkhats}}{\Fupcomp_{k+1}} = \Var{\gwkhats}{\Fupcomp_{k+1}} + \SqrdNrm{\Expec{\gwkhats}{\Fupcomp_{k+1}}}
\end{align*}
\end{proof}

\Cref{lem:expec_gwkhat} is used to remove the uplink compression noise straight after \Cref{lem:gkhat_Sk_to_gkhat} has been applied.
\begin{lemma}[Expectation of the squared norm of the compressed gradient when no memory]
\label{lem:expec_gwkhat}
In the case without memory, we have the following bound on the squared norm of the compressed gradient (randomly sampled), for all $k$ in $\N$:
\[
\Expec{\SqrdNrm{\gwkhat}}{\Fdwncomp_k} \leq \ffrac{\omgC^\up}{ N^2} \sum_{i=0}^N \Expec{\SqrdNrm{\gwk^i}}{\Fdwncomp_k} + \ffrac{1}{N^2} \sum_{i=0}^N \Expec{\SqrdNrm{\gwk^i - h_*^i}}{\Fdwncomp_k} + L \PdtScl{\nabla F(w_k)}{w_k - w_*} \,.
\]
\end{lemma}

\begin{proof}
Let $k$ in $\N$, first, we write as following:
\begin{align*}
    \lnrm\gwkhat\rnrm^2 &= \lnrm\gwkhat - \gwk + \gwk \rnrm^2 \nonumber \\ 
    &= \lnrm \gwkhat - \gwk \rnrm^2 + 2 \PdtScl{\gwkhat - \gwk}{\gwk } + \lnrm \gwk \rnrm^2 \,. 
\end{align*}

Taking stochastic expectation (recall that $\gwk$ is $\Fsto_{k+1}$-measurable and that $\Fdwncomp_{k} \subset   \Fsto_{k+1}$):
\begin{align}
\begin{split}
    \Expec{ \Expec{ \lnrm\gwkhat\rnrm^2}{ \Fsto_{k+1}}}{ \Fdwncomp_{k}} &= \Expec{\Expec{ \lnrm \gwkhat - \gwk \rnrm ^2}{ \Fsto_{k+1}}}{ \Fdwncomp_{k}}  \\
    &\qquad+ 2 \times \Expec{\Expec{\PdtScl{\gwkhat - \gwk}{\gwk }}{\Fsto_{k+1}}}{\Fdwncomp_k} \\
    &\qquad+ \Expec{\lnrm \gwk \rnrm^2}{\Fdwncomp_{k}} \,. 
\end{split}
\label{eq:decomposition_gk_hat}
\end{align}

We need to find a bound for each of the terms of above \cref{eq:decomposition_gk_hat}. The last term is handled in \Cref{lem:bounding_g_k+1}, narrowing it down to the case $p=1$.

It follows that we just need to bound $\lnrm \gwkhat - \gwk \rnrm^2$:
\begin{align*}
    \Expec{ \lnrm \gwkhat - \gwk \rnrm ^2}{ \Fsto_{k+1}} &= \Expec{ \lnrm\gwkhat - \Expec{  \gwkhat }{ \Fsto_{k+1}} \rnrm^2}{ \Fsto_{k+1}} \\
    &= \Expec{ \lnrm \frac{1}{N} \sum_{i=1}^N \gwkhat^i - \Expec{  \gwkhat^i}{ \Fsto_{k+1}} \rnrm^2}{  \Fsto_{k+1}} \\
    &=\frac{1}{N^2} \sum_{i=0}^N \Expec{\lnrm \gwkhat^i - \gwk^i \rnrm^2}{\Fsto_{k+1}}  \\
    &\qquad+ \underbrace{\frac{1}{N} \sum_{i \neq j} \Expec{\PdtScl{\gwkhat^i - \gwk^i}{\gwkhat^j - \gwk^j}}{\Fsto_{k+1}}}_{= 0 \text{ because } (\gwkhat^i)_{i=1}^N \text{ are independents}} \\
     &= \frac{1}{N^2} \sum_{i=1}^N \Expec{ \lnrm \gwkhat^i - \gwk^i \rnrm^2}{\Fsto_{k+1}} \,.\\
\end{align*}

Combining with \Cref{prop:variance_uplink}, we hold that:
\begin{align*}
    \Expec{ \lnrm\gwkhat - \gwk \rnrm^2}{ \Fsto_{k+1}} &\leq \frac{\omgC^\up}{N^2}\sum_{i=1}^N \lnrm \gwk^i \rnrm^2 \,.
\end{align*}

Now, we proved that:
\[
\left\{
    \begin{array}{ll}
    	\Expec{ \lnrm\gwkhat - \gwk \rnrm^2}{ \Fsto_{k+1}} &\leq \ffrac{\omgC^\up}{N^2} \sum_{i=1}^N \Expec{\SqrdNrm{\gwk^i}}{\Fsto_{k+1}} \\
    	\Expec{ \PdtScl{\gwkhat - \gwk}{\gwk }}{ \Fsto_{k+1}} &= 0 \text{\qquad(\Cref{prop:uplink_expectation})}\\
    	\Expec{ \lnrm \gwk \rnrm^2}{ \Fdwncomp_{k}} &\leq \ffrac{1}{N^2}\sum_{i=0}^N \Expec{\lnrm \gwk^i - h_*^i \rnrm^2}{\Fdwncomp_{k}} \\
    	&\qquad+ L \PdtScl{\nabla F(w_k)}{w_k - w_*} \text{\quad ( \Cref{lem:bounding_g_k+1}, with $p=1$)}\,.
    \end{array}
\right.
\]

Thus, we obtain from \cref{eq:decomposition_gk_hat}:
\begin{align*}
    \Expec{ \lnrm\gwkhat\rnrm^2}{ \Fdwncomp_{k}} &\leq \ffrac{\omgC^\up}{N^2} \sum_{i=1}^N \Expec{\lnrm \gwk^i \rnrm^2}{\Fdwncomp_{k}} +\ffrac{1}{N^2} \sum_{i=1}^N \Expec{\lnrm \gwk^i- h_*^i \rnrm^2}{\Fdwncomp_{k}} + L \PdtScl{\nabla F(w_k)}{w_k - w_*} \,. \\
\end{align*}
\end{proof}

\begin{lemma}
\label{lem:nomem_gkhati_to_gki}
In the case without memory, we have the following bound on the squared norm of the local compressed gradient, for all $k$ in $\N$, for all $i$ in $\llbracket 1, N \rrbracket$: $\Expec{\SqrdNrm{\gwkhat^i}}{\Fsto_{k+1}} \leq (\omgC^\up + 1) \SqrdNrm{\gwk^i}$
\end{lemma}

\begin{proof}
Let $k$ in $\N$ and $i$ in $\llbracket 1, N \rrbracket$:
\begin{align*}
\Expec{\SqrdNrm{\gwkhat^i}}{\Fsto_{k+1}} &= \Expec{\SqrdNrm{\gwkhat^i - \gwk^i + \gwk^i}}{\Fsto_{k+1}} \\
&= \Expec{\SqrdNrm{\gwkhat^i - \gwk^i}}{\Fsto_{k+1}} + 2 \underbrace{\Expec{\PdtScl{\gwkhat^i - \gwk^i}{\gwk^i}}{\Fsto_{k+1}}}_{=0} + \Expec{\SqrdNrm{\gwk^i}}{\Fsto_{k+1}}
\end{align*}
To obtain the result, we need to recall that $\SqrdNrm{\gwk}$ is $\Fsto{k+1}$-measurable, and then to use \Cref{prop:variance_downlink}.
\end{proof}

\subsubsection{Lemmas for the case with memory}
\label{app:subsubsec:lem_with_mem}

In this subsection, we give lemmas that are used only to demonstrate \Cref{thm:doublecompress_avg,app:thm:with_mem} (i.e. with memory).

In order to derive an upper bound  on the squared norm of $\SqrdNrm{ w_{k+1} - w_* }$, for $k $ in $ \N$, we need to control  $\SqrdNrm{ \gwkhats }$. This term is decomposed as a sum of three terms depending on:
 \begin{enumerate}
     \item the recursion over the memory term ($h_k^i$)
     \item the difference between the stochastic gradient at the current point and at the optimal point (later controlled by co-coercivity)
     \item the noise over stochasticity. 
 \end{enumerate}
\begin{lemma}
\label{lem:bound_sum_of_compressed_gradients}
In the case with memory, we have the following upper bound on the squared norm of the compressed gradient, for all $k$ in $\N$:
\begin{align*}
    \Expec{\SqrdNrm{ \gwkhats }}{\Ftotal} &\leq 2 \bigpar{\frac{2(\omgC^\up + 1) }{p} - 1} \ffrac{1}{N^2} \sum_{i=1}^N \Expec{\SqrdNrm{ \gwk^i - \gwkstar^i}}{\Ftotal} \\
    & \qquad+\bigpar{\frac{2(\omgC^\up + 1)}{p} - 2} \ffrac{1}{N^2} \sum_{i=1}^N \Expec{\SqrdNrm{ h_k^i - h_*^i }}{\Ftotal} \\
    &\qquad + L \PdtScl{\nabla F(w_k)}{w_k - w_*} \\
    &\qquad + \ffrac{2\sigmstar}{Nb} \times \bigpar{\frac{2(\omgC^\up + 1)}{p} - 1}\,.
\end{align*}

\end{lemma}

\begin{proof}

We take the expectation w.r.t. the $\sigma$-algebra $\Flast_{k+1}$:
Doing a bias-variance decomposition: $\Expec{\SqrdNrm{\gwkhats}}{\Flast_{k+1}} = \Expec{\SqrdNrm{\gwkhats - \gwkhat}}{\Flast_{k+1}} + \Expec{\SqrdNrm{\gwkhat}}{\Flast_{k+1}}$.

\textbf{First term.} 
\begin{align*}
    \Expec{\SqrdNrm{\gwkhats - \gwkhat}}{\Flast_{k+1}} &= \Expec{\SqrdNrm{\ffrac{1}{Np} \sum_\iS \widehat{\Delta}_k^i + h_k - \ffrac{1}{N} \sum_\iN \widehat{\Delta}_k^i + h_k }}{\Flast_{k+1}} \\
    &= \Expec{\SqrdNrm{\ffrac{1}{Np} \sum_\iN \widehat{\Delta}_k^i (B_i - p)}}{\Flast_{k+1}} \\
    &= \ffrac{1}{N^2 p^2} \sum_\iN \Expec{(B_i - p)^2}{\Flast_{k+1}} \SqrdNrm{\widehat{\Delta}_k^i} \,,
\end{align*}

by independence of device sampling and because $(\Delta_k^i)_\iN$ are $\Flast_{k+1}$-measurable. Next:
\begin{align*}
    \Expec{\SqrdNrm{\gwkhats - \gwkhat}}{\Flast_{k+1}} &= \ffrac{1}{N^2 p^2} \sum_\iN p(1-p) \Expec{\SqrdNrm{\widehat{\Delta}_k^i} }{\Flast_{k+1}}\,.
\end{align*}

Now we take expectation w.r.t. the $\sigma$-algebra $\Ftotal$. Because $\Ftotal \subset \Fsto_{k+1}$,~we have for all $i$ in $\{1, \cdots, N \}$, $\Expec{\SqrdNrm{\Delta_k^i}}{\Ftotal} = \Expec{\Expec{\SqrdNrm{\Delta_k^i}}{\Fsto_{k+1}}}{\Ftotal}$ and we use \Cref{prop:variance_uplink}:
\begin{align*}
    \Expec{\SqrdNrm{\gwkhats - \gwkhat}}{\Ftotal} &= \ffrac{(\omgC^\up + 1)(1-p)}{N^2 p} \sum_\iN \Expec{\SqrdNrm{\Delta_k^i}}{\Ftotal} \,, \text{~with  \Cref{lem:bounding_compressed_term}} \\
    &= \ffrac{2(\omgC^\up + 1)(1-p)}{N^2 p} \sum_\iN \Expec{\SqrdNrm{\gwk^i - h_*^i}}{\Ftotal} + \Expec{\SqrdNrm{h_k^i - h_*^i}}{\Ftotal} \,.
\end{align*}

\textbf{Second term.} Again, with a bias-variance decomposition: 
\[
\Expec{\SqrdNrm{\gwkhat}}{\Flast_{k+1}} = \Expec{\SqrdNrm{\gwk}}{\Flast_{k+1}} + \Expec{\SqrdNrm{\gwkhat - \gwk}}{\Flast_{k+1}}\,.
\]

We take expectation w.r.t. the $\sigma$-algebra $\Ftotal$, thus the first term is handled with \Cref{lem:bounding_g_k+1}: 
\[
\Expec{\SqrdNrm{\gwk}}{\Ftotal} \leq \ffrac{1}{N^2} \sum_\iN \Expec{\SqrdNrm{\gwk^i - h_*^i}}{\Ftotal} + L \PdtScl{\nabla F(w_k)}{w_k - w_*} \,.
\]

Considering the second term, by independence of the ``N'' compressions and using as previously \Cref{prop:variance_uplink} (because $\Ftotal \subset \Fsto_{k+1}$), we have:
\begin{align*}
\Expec{\SqrdNrm{\gwkhat - \gwk}}{\Ftotal} &= \ffrac{1}{N^2} \sum_\iN \Expec{\SqrdNrm{\widehat{\Delta}_k^i - \Delta_k^i}}{\Ftotal} \leq \ffrac{\omgC^\up}{N^2} \sum_\iN \SqrdNrm{\Delta_k^i} \,, \text{~and with \cref{lem:bounding_compressed_term}} \\
&\leq \ffrac{2\omgC^\up}{N^2} \sum_\iN \Expec{\SqrdNrm{\gwk^i - h_*^i}}{\Ftotal} + \Expec{\SqrdNrm{h_k^i - h_*^i}}{\Ftotal} \,.
\end{align*}

At the end:
\begin{align*}
    \Expec{\SqrdNrm{\gwkhat}}{\Ftotal} &= \ffrac{2\omgC^\up + 1}{N^2} \sum_\iN \Expec{\SqrdNrm{\gwk^i - h_*^i}}{\Ftotal} + \ffrac{2\omgC^\up}{N^2} \sum_\iN \Expec{\SqrdNrm{h_k^i - h_*^i}}{\Ftotal} \\
    &\qquad + L\PdtScl{\nabla F(w_k)}{w_k - w_*} \,.
\end{align*}

We can combine the first and second term, and it follows that:
\begin{align*}
    \Expec{\SqrdNrm{ \gwkhats }}{\Ftotal} &\leq \bigpar{\frac{2(\omgC^\up + 1) (1-p)}{p} + 2 \omgC^\up + 1} \ffrac{1}{N^2} \sum_{i=1}^N \Expec{\SqrdNrm{ \gwk^i - h_*^i}}{\Ftotal} \\
    & \qquad+\bigpar{\frac{2(\omgC^\up + 1) (1-p)}{p} + 2 \omgC^\up} \ffrac{1}{N^2} \sum_{i=1}^N \Expec{\SqrdNrm{ h_k^i - h_*^i }}{\Ftotal} \\
    &\qquad + L \PdtScl{\nabla F(w_k)}{w_k - w_*} \,, 
\end{align*}

we can now apply \Cref{lem:before_using_coco} to conclude the proof:
\begin{align*}
    \Expec{\SqrdNrm{ \gwkhats }}{\Ftotal} &\leq 2 \bigpar{\frac{2(\omgC^\up + 1) (1-p)}{p} + 2 \omgC^\up + 1} \ffrac{1}{N^2} \sum_{i=1}^N \Expec{\SqrdNrm{ \gwk^i - \gwkstar^i}}{\Ftotal} \\
    & \qquad+\bigpar{\frac{2(\omgC^\up + 1) (1-p)}{p} + 2 \omgC^\up} \ffrac{1}{N^2} \sum_{i=1}^N \Expec{\SqrdNrm{ h_k^i - h_*^i }}{\Ftotal} \\
    &\qquad + L \PdtScl{\nabla F(w_k)}{w_k - w_*} \\
    &\qquad + \ffrac{2\sigmstar}{Nb} \times \bigpar{\frac{2(\omgC^\up + 1) (1-p)}{p} + 2 \omgC^\up + 1}\,,
\end{align*}

and simplifying each coefficient gives the result.
\end{proof}

To show that the Lyapunov function is a contraction, we need to find a bound for each terms. Bounding $\SqrdNrm{w_{k+1} - w_*}$, for $k$ in $\N$, flows from update schema (see \cref{eq:update_schema}) decomposition. However the memory term $\SqrdNrm{h_{k+1}^i - h_*^i}$ involved in the Lyapunov function doesn't show up naturally. 

The aim of \Cref{lem:recursive_inequalities_over_memory} is precisely to provide a recursive bound over the memory term to highlight the contraction. Like \Cref{lem:expectation_h_l}, the following lemma comes from \citet{mishchenko_distributed_2019}.

\begin{lemma}[Recursive inequalities over memory term]
\label{lem:recursive_inequalities_over_memory}
Let $k \in \N$ and let $i \in \llbracket 1 , N \rrbracket$. The memory term used in the uplink broadcasting can be bounded using a recursion:
\begin{align*}
    \Expec{\SqrdNrm{ h_{k+1}^i - h_*^i }}{\Ftotal} &\leq \left(1 + p(2 \alpha^2 \omgC^\up + 2 \alpha^2 - 3 \alpha ) \right)) \Expec{\SqrdNrm{ h_{k}^i - h_*^i }}{\Ftotal} \\
    &\qquad+ 2p(2 \alpha^2 \omgC^\up + 2 \alpha^2 - \alpha) \Expec{\SqrdNrm{ \gwk - \gwkstar }}{\Ftotal} \\
    &\qquad+ 2p\ffrac{\sigmstar^2}{b} \left(2\alpha^2 (\omgC^\up + 1) - \alpha \right) \,.
\end{align*}

\end{lemma}

\begin{proof}
The proof is done in two steps:
\begin{enumerate}
    \item First, we povide a recursive bound on memory when the device is used for the update (i.e such that for $k$ in $\N$, for $i$ in $\{ 1, ... N \}$, $B_k^i = 1$)
    \item Then, we generalize to the case with all $i$ in $\llbracket 1, N, \rrbracket$ regardless to if they are used at the round $k$.
\end{enumerate}

\textit{First part.}
Let $k \in \N$ and let $i \in \llbracket 1 , N \rrbracket$ such that $B_k^i = 1$
\begin{align*}
    \Expec{\SqrdNrm{ h_{k+1}^i - h_*^i }}{\Fsto_{k+1}} &= \SqrdNrm{ \Expec{h^i_{k+1}}{\Fsto_{k+1}} - h_*^i  } \\
    &\qquad+ \Expec{\SqrdNrm{ h_{k+1}^i - \Expec{h_{k+1}^i}{\Fsto_{k+1}} }}{\Fsto_{k+1}} \text{\quad using \Cref{lem:expectation_decomposition}\,,}
\end{align*}
and now with \Cref{lem:expectation_h_l}:
\begin{align*}
    \Expec{\SqrdNrm{ h_{k+1}^i - h_*^i }}{\Fsto_{k+1}} = \SqrdNrm{ (1 - \alpha) h_k^i + \alpha \gwk^i - h_*^i } + \Expec{\SqrdNrm{ h_{k+1}^i - \Expec{h_{k+1}^i}{\Fsto_{k+1}} }}{\Fsto_{k+1}} .
\end{align*}
Now recall that $h_{k+1}^i = h_k^i + \alpha \widehat{\Delta}_k^i$ and $\Expec{ \widehat{\Delta}_k^i}{\Fsto_{k+1}} = \Delta_k^i $:
\begin{align*}
    \Expec{\SqrdNrm{ h_{k+1}^i - h_*^i }}{\Fsto_{k+1}} &= \SqrdNrm{ (1- \alpha) (h_k^i - h_*^i) + \alpha (\gwk^i - h_*^i)} + \alpha^2 \Expec{ \SqrdNrm{ \widehat{\Delta}_k^i - \Delta_k^i } }{\Fsto_{k+1}} \,.
\end{align*}
Using \Cref{lem:inequalities_with_alpha} of \Cref{sect:useful_identies} and \Cref{prop:variance_uplink}:
\begin{align*}
    \Expec{\SqrdNrm{ h_{k+1}^i - h_*^i }}{\Fsto_{k+1}}&\leq (1-\alpha) \SqrdNrm{ h_k^i - h_*^i } + \alpha \SqrdNrm{ \gwk^i - h_*^i } \\
    &\qquad- \alpha (1-\alpha) \SqrdNrm{h_k^i - \gwk^i}+ \alpha^2 \omgC^\up \SqrdNrm{ \Delta_k^i }.
\end{align*}
Because $h_k^i - \gwk^i = \Delta_k^i$:
\begin{align*}
     \Expec{\SqrdNrm{ h_{k+1}^i - h_*^i }}{\Fsto_{k+1}} &\leq (1-\alpha) \SqrdNrm{ h_k^i - h_*^i } + \alpha \SqrdNrm{ \gwk^i - h_*^i } + \alpha \left(\alpha (\omgC^\up + 1 ) - 1 \right)\SqrdNrm{ \Delta_k^i } \,,
\end{align*}

and using \Cref{lem:bounding_compressed_term}:
\begin{align*}
    &\leq (1-\alpha) \SqrdNrm{ h_k^i - h_*^i } + \alpha \SqrdNrm{ \gwk^i - h_*^i } \\
    &\qquad+ 2\alpha \left(\alpha (\omgC^\up + 1 ) - 1 \right) \left( \SqrdNrm{h_k^i - h_*^i} + \SqrdNrm{ \gwk - h_*^i } \right) \\
    &\leq \left(1 + 2 \alpha^2 \omgC^\up + 2 \alpha^2 - 3 \alpha \right) \SqrdNrm{ h_k^i - h_*^i } \\
    &\qquad+ \alpha (2 \alpha \omgC^\up + 2 \alpha - 1) \SqrdNrm{ \gwk - h_*^i } \,.
\end{align*}

Finally using \cref{eq:g_i_k_h_i_k} of \Cref{lem:before_using_coco} and writing that:
\begin{align*}
    \Expec{\SqrdNrm{ \gwk - h_*^i }}{\Fdwncomp_k} = \Expec{\Expec{\SqrdNrm{ \gwk - h_*^i }}{\Fsto_{k+1}}}{\Fdwncomp_k} \text{(because $\Fdwncomp_k \subset \Fsto_{k+1}$) \,,}
\end{align*}

we have:
\begin{align*}
    \Expec{\SqrdNrm{ h_{k+1}^i - h_*^i }}{\Fdwncomp_k} &\leq (1 + \underbrace{2 \alpha^2 \omgC^\up + 2 \alpha^2 - 3 \alpha}_{=T_1}) \Expec{\SqrdNrm{ h_{k}^i - h_*^i }}{\Fdwncomp_k} \\
    &\qquad+ 2(\underbrace{2 \alpha^2 \omgC^\up + 2 \alpha^2 - \alpha}_{T_2}) \Expec{\SqrdNrm{ \gwk - \gwkstar }}{\Fdwncomp_k} \\
    &\qquad+ 2\underbrace{\ffrac{\sigmstar^2}{b} \left(2\alpha^2 (\omgC^\up + 1) - \alpha \right)}_{T_3} \,,
\end{align*}

which conclude the first part of the proof. Now we take the general case with $\forall i \in \llbracket 1, N \rrbracket, B_k^i = 0$ or $1$.

\textit{Second part.} Let $k \in \N$ and let $i \in \llbracket 1 , N \rrbracket$.

To resume, if the device participate to the iteration $k$, we have 
\[
\Expec{\SqrdNrm{ h_{k+1}^i - h_*^i }}{\Fdwncomp_k} \leq (1 + T_1) \Expec{\SqrdNrm{ h_{k}^i - h_*^i }}{\Fdwncomp_k} + 2 T_2 \Expec{\SqrdNrm{ \gwk - \gwkstar }}{\Fdwncomp_k} + 2T_3\,,
\]

otherwise:
\[
\Expec{\SqrdNrm{ h_{k+1}^i - h_*^i }}{\Fdwncomp_k} = \Expec{\SqrdNrm{ h_{k}^i - h_*^i }}{\Fdwncomp_k} \,.
\]

In other words, for all $i$ in $\llbracket 1, N \rrbracket$:
\begin{align*}
\Expec{\SqrdNrm{ h_{k+1}^i - h_*^i }}{\Fdwncomp_k} &\leq (1 + T_1)B_k^i \Expec{\SqrdNrm{ h_{k}^i - h_*^i }}{\Fdwncomp_k} + 2 T_2 B_k^i \Expec{\SqrdNrm{ \gwk - \gwkstar }}{\Fdwncomp_k} + 2T_3 B_k^i \\
&\qquad+ (1 - B_k^i) \Expec{\SqrdNrm{ h_{k}^i - h_*^i }}{\Fdwncomp_k} \\
&\leq (1 + T_1 B_k^i) \Expec{\SqrdNrm{ h_{k}^i - h_*^i }}{\Fdwncomp_k} + 2 T_2 B_k^i \Expec{\SqrdNrm{ \gwk - \gwkstar }}{\Fdwncomp_k} + 2T_3 B_k^i \\
\end{align*}

Taking expectation w.r.t $\sigma$-algebra $\Ftotal$ gives the result.

\end{proof}

After successfully invoking all previous lemmas, we will finally be able to use co-coercivity. \Cref{lem:applying_coco} shows how \Cref{asu:cocoercivity} is used to do it. After this stage, proof will be continued by applying strong-convexity of $F$.

\begin{lemma}[Applying co-coercivity]
\label{lem:applying_coco}
This lemma shows how to apply co-coercivity on stochastic gradients.
\begin{align*}
    \forall k \in \N\,, \quad \frac{1}{N} \sum_{i=1}^N \Expec{\SqrdNrm{ \gwk^i - \gwkstar }}{\Fdwncomp_{k}} \leq L \PdtScl{\nabla F(w_k)}{w_k - w_*} \,.
\end{align*}

\end{lemma}

\begin{proof}
Let $k \in \N$.
\begin{align*}
\frac{1}{N} \sum_{i=1}^N \Expec{\SqrdNrm{ \gwk^i - \gwkstar^i }}{\Fdwncomp_k} &\leq \frac{1}{N}  \sum_{i=1}^N L\PdtScl{ \Expec{\gwk^i - \gwkstar^i}{\Fdwncomp_k} }{w_k - w_*} \text{using \Cref{asu:cocoercivity},}\\
&\leq L\PdtScl{ \frac{1}{N}  \sum_{i=1}^N \Expec{\gwk^i - \gwkstar^i}{{\Fdwncomp_k} } }{w_k - w_*} \\
&\leq L\PdtScl{ \frac{1}{N}  \sum_{i=1}^N \nabla F_i(w_k) - \nabla F_i(w_*)}{w_k - w_*}  \,.
\end{align*} 
\end{proof}

\section{Proofs of Theorems}
\label{app:all_proofs}

In this section we give demonstrations of all our theorems, that is to say, first the proofs of \Cref{app:thm:with_mem,app:thm:without_mem} from which flow \Cref{thm:cvgce_artemis}.
Their demonstration sketch is drawn from \citet{mishchenko_distributed_2019}.
And in a second time, we give a complete demonstration of theorems stated in the main paper: \Cref{thm:main_PRave,thm:cvdist}.

For the sake of demonstration, we define a Lyapunov function $V_k$ \cite[as in][]{mishchenko_distributed_2019,liu_double_2020}, with $k$ in $\llbracket 1, K \rrbracket$:
\[
V_k = \SqrdNrm{w_k - w_*} + 2 \gamma^2 \cst \frac{1}{N} \sum_{i = 1}^N \SqrdNrm{h_k^i - h_*^i} \,.
\]

The Lyapunov function is defined combining two terms:
\begin{enumerate}
    \item the distance from parameter $w_k$ to optimal parameter $w_*$ 
    \item The memory term, the distance between the next element prediction $h_k^i$ and the true gradient $h_*^i = \nabla F_i(w_*)$. 
\end{enumerate}

The aim is to proof that this function is a $(1 - \gamma \mu)$ contraction for each variant of \Artemis, and also when using Polyak-Ruppert averaging.  To show that it's a contraction, we need three stages: 

\begin{enumerate}
    \item we develop the update schema defined in \cref{eq:update_schema} to get a first bound on $\lnrm w_k - w_* \rnrm^2$ 
    \item we find a recurrence over the memory term $\lnrm h_k^i- h_*^i \rnrm^2$
    \item and finally we combines the two equations to obtain the expected contraction using co-coercivity and strong convexity.
\end{enumerate} 

\subsection{Proof of main  Theorem for  \Artemis~- variant without memory}
\label{app:proof_singlecompressnomemory}

\begin{theorem}[Unidirectional or bidirectional compression without memory]
\label{app:thm:without_mem}

Considering that \Cref{asu:partial_participation,asu:expec_quantization_operator,asu:strongcvx,asu:cocoercivity,asu:noise_sto_grad,asu:bounded_noises_across_devices} hold. Taking $\gamma$ such that
\[
\gamma \leq \ffrac{pN}{L (\omgC^\dwn +1) \left( pN + 2(\omgC^\up + 1) \right)} \,,
\]

then running \Artemis~with $\alpha = 0$ (i.e without memory), we have for all $k$ in $\N$:
\begin{align*}
     \E \SqrdNrm{w_{k+1} - w_*} &\leq (1 - \gamma \mu)^{k+1} \SqrdNrm{w_{0} - w_*} + 2\gamma \frac{E_{}}{\mu p N} \,,
\end{align*}

with $E = (\omgC^\dwn + 1) \left( (\omgC^\up + 1) \ffrac{\sigmstar^2}{b} + (\omgC^\up + 1 - p) B^2 \right)$. In the case of unidirectional compression (resp. no compression), we have $\omgC^\dwn=0$ (resp. $\omgC^{\up/\dwn}=0$).
\end{theorem}

\begin{proof}

In the case of variant of \Artemis~with $\alpha = 0$, we don't have any memory term, thus $p=0$ and we don't need to use the Lyapunov function.   

Let $k$ in $\N$, we start by writing that by definition of \cref{eq:update_schema}:
\begin{align*}
    \lnrm w_{k+1} - w_* \rnrm^2 &= \lnrm w_{k} - \gamma \Omega_{k+1, S_k} - w_* \rnrm^2 \\
     &= \SqrdNrm{w_{k} - w_*} - 2 \gamma \PdtScl{\Omega_{k+1, S_k}}{w_k - w_* } + \gamma^2 \SqrdNrm{\Omega_{k+1, S_k}} \,,
\end{align*}
with $\Omega_{k+1, S_k} = \mathcal{C}_{\dwn} \left(\ffrac{1}{ N}\sum_{i=1}^N \gwkhat^i \right)$. First, we have $\Expec{\Omega_{k+1, S_k}}{\Fupcompnotsamp} = \gwkhats$; secondly considering that:
\begin{align*}
\Expec{\lnrm \Omega_{k+1, S_k} \rnrm^2}{\Fupcompnotsamp} = \V(\Omega_{k+1, S_k})  + \lnrm \Expec{\Omega_{k+1, S_k}}{\Fupcompnotsamp}\rnrm^2=  (\omgC^\dwn +1) \lnrm \gwkhats \rnrm^2 \,,
\end{align*}

it leads to:
\begin{align*}
    \Expec{\SqrdNrm{ w_{k+1} - w_*}}{\Fupcompnotsamp} &= \Expec{\SqrdNrm{ w_{k} - w_*}}{\Fupcompnotsamp} - 2 \gamma \PdtScl{\gwkhats}{w_k - w_* } \\
    &\qquad+ \gamma^2 (\omgC^\dwn +1) \SqrdNrm{\gwkhats} \,.
\end{align*}

Now, if we take expectation w.r.t $\sigma$-algebra $\Ftotal \subset \Fupcompnotsamp$ (the inclusion is true because $\Fupcomp_{k+1}$ contains $\Fsamp_{k-1}$); with use of \Cref{prop:stochastic_expectation,prop:uplink_expectation,lem:gkhat_Sk_to_gkhat,prop:exp_samp} we obtain : 
\begin{align}
\label{eq:singlenomem_decomposition}
     \Expec{ \lnrm w_{k+1} - w_* \rnrm^2}{ \Ftotal} &= \Expec{\SqrdNrm{w_{k} - w_*}}{\Ftotal}   \\
     &\qquad- 2 \gamma \PdtScl{\nabla F(w_k)}{w_k - w_* } \nonumber\\
    &\qquad+ \gamma^2 (\omgC^\dwn +1)  \ffrac{(1-p)}{p N^2} \sum_{i=0}^N \Expec{\SqrdNrm{\gwkhat^i}}{\Ftotal} \nonumber \\
    &\qquad+ \gamma^2 (\omgC^\dwn +1) \Expec{\SqrdNrm{\gwkhat}}{\Ftotal} \nonumber \,.
\end{align}

For sake of clarity we temporarily note: $\Pi = \ffrac{(1-p)}{p N^2} \sum_{i=0}^N \Expec{\SqrdNrm{\gwkhat^i}}{\Ftotal} + \Expec{\SqrdNrm{\gwkhat}}{\Ftotal}$.
Recall that in fact, $\Pi$ is equal to $\Expec{\SqrdNrm{\gwkhats}}{\Ftotal}$.

Using \Cref{lem:nomem_gkhati_to_gki,lem:expec_gwkhat}, we have:
\begin{align*}
    \Expec{ \Pi }{ \Ftotal} &\leq \ffrac{(1-p)(\omgC^\up + 1)}{p N^2} \sum_{i=0}^N \Expec{\SqrdNrm{\gwk^i}}{\Ftotal} \\
    &\qquad+ \ffrac{\omgC^\up}{N^2} \sum_{i=0}^N \Expec{\SqrdNrm{\gwk^i}}{\Ftotal} + \frac{1}{N} \sum_{i=0}^N \Expec{\SqrdNrm{\gwk^i - h_*^i}}{\Ftotal} \\
    &\qquad+ L \PdtScl{\nabla F(w_k)}{w_k - w_*} \,. \\
    &\leq \ffrac{(\omgC^\up + 1-  p)}{p N^2} \sum_{i=0}^N \Expec{\SqrdNrm{\gwk^i}}{\Ftotal}  + \frac{1}{N} \sum_{i=0}^N \Expec{\SqrdNrm{\gwk^i - h_*^i}}{\Ftotal} \\
    &\qquad + L \PdtScl{\nabla F(w_k)}{w_k - w_*} \,. \\
\end{align*}

Lets introducing the noise at optimal point $w_*$ with the two equations of \Cref{lem:before_using_coco}:
\begin{align*}
    \Expec{\Pi}{ \Ftotal} &\leq \ffrac{(\omgC^\up + 1 - p)}{p N^2} \sum_{i=1}^N 2 \left(\Expec{\SqrdNrm{ \gwk^i - \gwkstar^i}}{\Ftotal} + (\ffrac{\sigmstar^2}{b} +B^2) \right) \\
    &\qquad+\ffrac{1}{N^2} \sum_{i=1}^N 2 \left(\Expec{\SqrdNrm{ \gwk^i - \gwkstar^i}}{\Ftotal} + \ffrac{\sigmstar^2}{b} \right) \\
    &\qquad+ L \PdtScl{\nabla F(w_k)}{w_k - w_*} \,. \\
    &\leq \frac{2(\omgC^\up + 1)}{p N^2} \sum_{i=1}^N  \Expec{\SqrdNrm{ \gwk^i - \gwkstar^i}}{\Ftotal} \\
    &\qquad + L \PdtScl{\nabla F(w_k)}{w_k - w_*} + \frac{2 \times \left( (\omgC^\up + 1) \ffrac{\sigmstar^2}{b} + \omgC^\up + 1 - p) B^2 \right)}{p N} \,.
\end{align*}

Invoking cocoercivity (\Cref{asu:cocoercivity}):
\begin{align}
\label{app:eq:pi_coco_before_inject}
    \Expec{\Pi}{ \Ftotal} &\leq \frac{2(\omgC^\up + 1)}{p N^2} \sum_{i=1}^N  \Expec{L \PdtScl{\gwk^i - \gwkstar^i}{w_k - w_*}}{\Ftotal} \nonumber \nonumber \\
    &\qquad + L \PdtScl{\nabla F(w_k)}{w_k - w_*} + \frac{2 \times \left( (\omgC^\up + 1) \ffrac{\sigmstar^2}{b} + \omgC^\up + 1 - p) B^2 \right)}{p N} \nonumber\\
    &\leq \frac{2(\omgC^\up + 1)L}{p N} \PdtScl{\nabla F(w_k)}{w_k - w_*} \nonumber \\
    &\qquad + L \PdtScl{\nabla F(w_k)}{w_k - w_*} + \frac{2 \times \left( (\omgC^\up + 1) \ffrac{\sigmstar^2}{b} + \omgC^\up + 1 - p) B^2 \right)}{p N} \,.
\end{align}

Finally, we can inject \cref{app:eq:pi_coco_before_inject} in \cref{eq:singlenomem_decomposition} to obtain: 
\begin{align}
\begin{split}
    \Expec{\lnrm w_{k+1} - w_* \rnrm^2}{ \Ftotal} &\leq 
    \SqrdNrm{w_{k} - w_*} - 2 \gamma \left(1 -  \frac{\gamma L (\omgC^\dwn +1) (\omgC^\up + 1 )}{p N} -\frac{\gamma L (\omgC^\dwn +1)  }{2} \right) \PdtScl{\nabla F(w_k)}{w_k - w_* }\\
    &\qquad + \frac{2 \times \overbrace{(\omgC^\dwn + 1) \left( (\omgC^\up + 1) \ffrac{\sigmstar^2}{b} + (\omgC^\up + 1 - p) B^2 \right)}^{= E}}{p N} \,.
\end{split}
\label{prf:after_cocoercivity}
\end{align}

We need $1 -  \frac{\gamma L (\omgC^\dwn +1) (\omgC^\up + 1 )}{p N} -\frac{\gamma L (\omgC^\dwn +1)}{2} \geq 0$ in order to further apply strong convexity. This condition is equivalent to:
\[
\gamma \leq \ffrac{2pN}{L (\omgC^\dwn +1) \left( p N + 2(\omgC^\up + 1) \right)}\,.
\]

Finally, using strong convexity of $F$ (\Cref{asu:strongcvx}), we rewrite \Cref{prf:after_cocoercivity}:
\begin{align*}
\begin{split}
    \Expec{\lnrm w_{k+1} - w_* \rnrm^2}{\Ftotal} &\leq \SqrdNrm{w_{k} - w_*} - 2 \gamma \mu \left(1 - \frac{\gamma L (\omgC^\dwn +1) (\omgC^\up + 1)}{p N} -\frac{\gamma L (\omgC^\dwn +1)}{2} \right) \SqrdNrm{w_{k} - w_*} \\
    &\qquad+ 2\gamma^2 \frac{E}{p N} \,, \text{ equivalent to:}
\end{split}
\\[2ex]
\begin{split}
     &\leq \left(1 - 2 \gamma \mu \left(1 - \frac{\gamma L (\omgC^\dwn +1)  (\omgC^\up + 1)}{pN} -\frac{\gamma L(\omgC^\dwn +1) s}{2} \right) \right) \SqrdNrm{w_{k} - w_*} + 2\gamma^2 \frac{E}{p N}\,.
\end{split}
\end{align*}

To guarantee a convergence in $(1-\gamma \mu)$, we need:
\begin{align*}
    &\frac{1}{2} \geq \frac{\gamma L (\omgC^\dwn +1) (\omgC^\up + 1)}{pN} + \frac{\gamma L (\omgC^\dwn +1)}{2} \\
    \Longleftrightarrow\quad &\gamma \leq \ffrac{pN}{L (\omgC^\dwn +1) \left( pN + 2(\omgC^\up + 1) \right)} \,, 
\end{align*}

which is stronger than the condition obtained to correctly apply strong convexity. Then we are allowed to write:
\begin{align*}
    &\E \SqrdNrm{w_{k+1} - w_*} \leq (1 - \gamma \mu) \E \SqrdNrm{w_{k} - w_*} + 2\gamma^2 \frac{E}{p N} \\
    \Longleftrightarrow \qquad& \E \SqrdNrm{w_{k+1} - w_*} \leq (1-\gamma \mu)^{k+1} \E \SqrdNrm{w_0 - w_*} + 2 \gamma^2 \frac{E}{pN} \times \frac{1 - (1- \gamma \mu)^{k+1}}{\gamma \mu} \\ 
    \Longleftrightarrow \qquad& \E\SqrdNrm{w_{k+1} - w_*} \leq (1-\gamma \mu)^{k+1} \SqrdNrm{w_0 - w_*} + 2  \gamma \frac{E_{}}{\mu p N} \,,
\end{align*}
and the proof is complete.

\end{proof}

\subsection{Proof of main Theorem
for  \Artemis~- variant with memory}
\label{app:proof_doublecompressnomem}

\begin{theorem}[Unidirectional or bidirectional compression with memory]
\label{app:thm:with_mem}

Considering that \Cref{asu:partial_participation,asu:expec_quantization_operator,asu:strongcvx,asu:cocoercivity,asu:noise_sto_grad,asu:bounded_noises_across_devices} hold. We use $w_*$ to indicate the optimal parameter such that $\nabla F (w_*) = 0$, and we note $h_*^i = \nabla F_i(w_*)$. We define the Lyapunov function:
\[
V_k = \SqrdNrm{w_k - w_*} + 2 \gamma^2 \cst \frac{1}{N} \sum_{i = 1}^N \SqrdNrm{h_k^i - h_*^i} \,.
\]

We defined $\cst \in \mathbb{R}^*$, such that:
\begin{align}
\label{cond:cst_withmem}
\ffrac{(\omgC^\dwn + 1) \bigpar{\frac{\omgC^\up + 1}{p} - 1}}{\alpha p \left( 3 - 2 \alpha (\omgC^\up + 1) \right)} \leq \cst \leq \frac{N - \gamma L (\omgC^\dwn + 1) \left(N + \frac{4(\omgC^\up + 1)}{p} - 2 \right)}{4 \gamma L p \alpha \left( 2 \alpha (\omgC^\up +1) - 1 \right)} \,.
\end{align}

Then, using \Artemis~with a memory mechanism ($\alpha \neq 0$), the convergence of the algorithm is guaranteed if:
\begin{equation}
\left\{
    \begin{array}{l}
    	\ffrac{1}{2(\omgC^\up + 1)} \leq \alpha < \min \left( \ffrac{3}{2(\omgC^\up + 1)}, \quad  \ffrac{3N - \gamma L  (\omgC^\dwn + 1) \left( 3 N + \frac{8(\omgC^\up + 1)}{p} - 2 \right)}{2 (\omgC^\up +1)  (N - \gamma L  (\omgC^\dwn + 1) (N + 2) )} \right)\\
    	\gs
    	\gamma < \min
    	\left\{
        \begin{array}{l}
    	\ffrac{1}{(\omgC^\dwn + 1) \left( 1 + \frac{2}{Np}\right) L }, \quad \ffrac{3}{(\omgC^\dwn + 1) \left( 3 + \frac{8 (\omgC^\up - 1) - 2p}{Np} \right) L} ,  \\
    	\ffrac{N}{(\omgC^\dwn + 1)\left(N + \frac{4(\omgC^\up + 1)}{p} - 2 \right) L} \\
    	\end{array}
        \right\}\,.
    \end{array}
\right.      
\end{equation}

And we have a bound for the Lyapunov function:
\begin{align*}
    \E V_{k+1} \leq (1- \gamma \mu)^{k+1} \left(\SqrdNrm{w_0 - w_*} + 2 \cst \gamma^2 B^2  \right) + 2 \gamma \frac{E}{\mu N} \,,
\end{align*}

with 
\begin{align*}
E = \ffrac{\sigmstar^2}{b} \bigpar{(\omgC^\dwn + 1) \bigpar{\frac{2(\omgC^\up + 1)}{p} -1} + 2p C \bigpar{2\alpha^2 (\omgC^\up + 1) - \alpha }  } \,.
\end{align*}

In the case of unidirectional compression (resp. no compression), we have $\omgC^\dwn=0$ (resp. $\omgC^{\up/\dwn}=0$).

\end{theorem}

\begin{proof}

Let $k \in \llbracket 1, K \rrbracket$, by definition of the update schema in the partial-participation setting: $w_{k+1} = w_k - \gamma \Omega_{k+1, S_k}$, with~$\Omega_{k+1, S_k} = \mathcal{C}_{\dwn} \left(\ffrac{1}{pN}\sum_\iS (\widehat{\Delta}_k^i + h_k^i) \right)$, thus:
\begin{align*}
    \lnrm w_{k+1} - w_* \rnrm^2 &= \lnrm w_k - w_* + \gamma \Omega_{k+1, S_k} \rnrm^2 \\
    &= \lnrm w_k - w_* \rnrm - 2 \gamma \PdtScl{\Omega_{k+1, S_k}}{w_k - w_*} + \gamma^2 \lnrm \Omega_{k+1, S_k} \rnrm^2 \,.
\end{align*}

First, we have $\Expec{\Omega_{k+1, S_k}}{\Fupcompnotsamp} = \gwkhats$; secondly considering that:
\begin{align*}
\Expec{\lnrm \Omega_{k+1, S_k} \rnrm^2}{\Fupcompnotsamp} = \V(\Omega_{k+1, S_k})  + \lnrm \Expec{\Omega_{k+1, S_k}}{\Fupcompnotsamp}\rnrm^2=  (\omgC^\dwn +1) \lnrm \gwkhats \rnrm^2 \,,
\end{align*}

it leads to:
\begin{align*}
    \Expec{\SqrdNrm{ w_{k+1} - w_*}}{\Fupcompnotsamp} &= \Expec{\SqrdNrm{ w_{k} - w_*}}{\Fupcompnotsamp} - 2 \gamma \PdtScl{\gwkhats}{w_k - w_* } \\
    &\qquad+ \gamma^2 (\omgC^\dwn +1) \SqrdNrm{\gwkhats} \,.
\end{align*}
we can invoke \Cref{lem:bound_sum_of_compressed_gradients} with the $\sigma$-algebras $\Ftotal$:
    \begin{align}
    \label{eq:proof_end_decomposition_doublecompressnomem}
    \begin{split}
        \Expec{\lnrm w_{k+1} - w_* \rnrm^2}{\Ftotal} &\leq \SqrdNrm{w_{k} - w_*} - 2 \gamma \Expec{\PdtScl{\gwkhats}{w_k - w_*}}{\Ftotal} \\
        &\qquad+ 2 \gamma^2 (\omgC^\dwn + 1) \bigpar{\frac{2(\omgC^\up + 1)}{p} -1} \ffrac{1}{N^2} \sum_{i=1}^N \Expec{\SqrdNrm{ \gwk^i - \gwkstar^i}}{\Ftotal} \\
        & \qquad+ \gamma^2 (\omgC^\dwn + 1) \bigpar{\frac{2(\omgC^\up + 1)}{p} -2} \ffrac{1}{N^2} \sum_{i=1}^N \Expec{\SqrdNrm{ h_k^i - h_*^i }}{\Ftotal} \\
        &\qquad + \gamma^2 (\omgC^\dwn + 1) L \PdtScl{\nabla F(w_k)}{w_k - w_*} \\
        &\qquad + \ffrac{2\sigmstar}{Nb} \times \gamma^2 (\omgC^\dwn + 1) \bigpar{\frac{2(\omgC^\up + 1)}{p} - 1} \,.
    \end{split}
    \end{align}

Note that in the case of unidirectional compression, we have $\Omega_{k+1} = \widehat{g}_{k+1}$, and the steps above are more straightforward.
Recall that according to \Cref{lem:recursive_inequalities_over_memory} (and taking the sum), we have:
\begin{align}
\label{eq:proof_doublecompressnomem_combinaison_mem1}
\begin{split}
    &\frac{1}{N^2} \sum_{i=1}^N \Expec{\lnrm h_{k+1}^i - h_*^i \rnrm^2}{\Ftotal} \\
    &\qquad\leq \left(1 + p (2 \alpha^2 \omgC^\up + 2 \alpha^2 - 3 \alpha) \right) \frac{1}{N^2} \sum_{i=1}^N \Expec{\lnrm h_{k}^i - h_*^i \rnrm^2}{\Ftotal} \\
    &\qquad+ 2 p(2 \alpha^2 \omgC^\up + 2 \alpha^2 - \alpha) \frac{1}{N^2} \sum_{i=1}^N \Expec{\lnrm \gwk^i - \gwkstar^i \rnrm^2}{\Ftotal} \\
    &\qquad+  \frac{2p}{N} \ffrac{\sigmstar^2}{b} \bigpar{2\alpha^2 (\omgC^\up + 1) - \alpha } \,
\end{split}
\end{align}

With a linear combination (\ref{eq:proof_end_decomposition_doublecompressnomem}) + $2 \gamma^2 \cst$ (\ref{eq:proof_doublecompressnomem_combinaison_mem1}):
\begin{align*}
    &\Expec{\lnrm w_{k+1} - w_* \rnrm^2}{\Ftotal} + 2 \gamma^2 \cst \frac{1}{ N^2} \sum_{i=1}^N \Expec{\lnrm h_{k+1}^i - h_*^i \rnrm^2}{\Ftotal} \\ 
    &\qquad\leq \SqrdNrm{w_{k} - w_*} - 2 \gamma \Expec{\PdtScl{\gwkhats}{w_k - w_*}}{\Ftotal} \\
    &\qqquad+  2\gamma^2  \underbrace{\left((\omgC^\dwn + 1)\bigpar{\frac{2(\omgC^\up + 1)}{p} -1} + 2 p C (2 \alpha^2 \omgC^\up + 2 \alpha^2 - \alpha) \right)}_{:= A_\cst} \\
    &\qqquad\qqquad \times \frac{1}{N^2} \sum_{i=1}^N \Expec{\lnrm \gwk^i- \gwkstar^i \rnrm}{\Ftotal} \\
    &\qqquad+ 2\gamma^2 \cst  \underbrace{ \bigpar{\ffrac{\omgC^\dwn + 1}{2 \cst} \bigpar{\frac{2(\omgC^\up + 1)}{p} -2} + 1 + p(2 \alpha^2 \omgC^\up + 2\alpha^2 - 3\alpha) }}_{:= D_{\cst} }  \frac{1}{N^2} \sum_{i=1}^N \Expec{\lnrm h_k^i - h_*^i \rnrm^2}{\Ftotal} \\
    &\qqquad+ \gamma^2 (\omgC^\dwn +1 ) L \PdtScl{\nabla F(w_k)}{w_k - w_*} \\
    &\qqquad+  \frac{2 \gamma^2}{N} \underbrace{\left(\ffrac{\sigmstar^2}{b} \bigpar{(\omgC^\dwn + 1) \bigpar{\frac{2(\omgC^\up + 1)}{p} -1} + 2p C \bigpar{2\alpha^2 (\omgC^\up + 1) - \alpha }  } \right)}_{:= E_{}} \,.
\end{align*}

Because $\Fdwncomp_k \subset \Fsto_{k+1}$, because $\gwkhat$ is independent of $\Fsamp_k$, and with \Cref{prop:stochastic_expectation,prop:uplink_expectation}:
\[
    \Expec{\gwkhat}{\Ftotal} = \Expec{\Expec{\gwkhat}{\Fsto_{k+1}}}{\Fdwncomp_{k}} = \Expec{\gwk}{\Fdwncomp_k} = \nabla F(w_k) \,.
\]

We transform $\lnrm \gwk^i- \gwkstar^i \rnrm$ applying co-coercivity (\Cref{lem:applying_coco}):
\begin{align}
\label{eq:doublecompress_before_stgcvx}
\begin{split}
    \E V_{k+1} &\leq \SqrdNrm{w_{k} - w_*} -  2\gamma \left(1 - \gamma L \left(\frac{\omgC^\dwn +1}{2} + \frac{A_\cst}{N} \right) \right) \PdtScl{\nabla F(w_k)}{w_k - w_*} \\
    &\qquad+ 2\gamma^2 \cst D_{\cst} \frac{1}{N^2} \sum_{i=1}^N \Expec{\lnrm h_k^i - h_*^i \rnrm^2}{\Ftotal} +  \frac{2 \gamma^2}{N} E_{} \,.
\end{split}
\end{align}
Now, the goal is to apply strong convexity of $F$ (\Cref{asu:strongcvx}) using the inequality presented in \Cref{lem:strongly-cvxe}.
But then we must have:
\begin{align*}
    &\qquad 1 - \gamma L \left(\frac{\omgC^\dwn + 1}{2} + \frac{A_\cst}{N} \right) \geq 0 \,,\\
\end{align*}
However, in order to later obtain a convergence in $(1 - \gamma \mu)$, we will use a stronger condition and, instead, state that we need:
\begin{align*}
    &\qquad \gamma L \left(\frac{\omgC^\dwn + 1}{2} + \frac{A_\cst}{N} \right) \leq \frac{1}{2} \\
    \Longleftrightarrow&\qquad A_\cst \leq \frac{\left(1 - \gamma L (\omgC^\dwn + 1) \right)N}{2 \gamma L}\\
    \Longleftrightarrow&\qquad \left((\omgC^\dwn + 1)\bigpar{\frac{2(\omgC^\up + 1)}{p} -1} + 2 p C (2 \alpha^2 \omgC^\up + 2 \alpha^2 - \alpha) \right) \leq \frac{(1 - \gamma L (\omgC^\dwn + 1)) N}{2 \gamma L} \\
    \Longleftrightarrow&\qquad \cst \leq \frac{N - \gamma L (\omgC^\dwn + 1) \left(N + \frac{4(\omgC^\up + 1)}{p} - 2 \right)}{4 \gamma L p \alpha \left( 2 \alpha (\omgC^\up +1) - 1 \right)} \,.\\
\end{align*}
This holds only if the numerator and the denominator are positive:
\[
\left\{
    \begin{array}{ll}
        N - \gamma L (\omgC^\dwn + 1) \left(N + \frac{4(\omgC^\up + 1)}{p} - 2 \right) > 0 \Longleftrightarrow  \gamma < \ffrac{N}{(\omgC^\dwn + 1)\left(N + \frac{4(\omgC^\up + 1)}{p} - 2 \right) L}  \\
    	2 \alpha (\omgC^\up +1) - 1 \leq 0 \Longleftrightarrow \alpha \geq \ffrac{1}{2(\omgC^\up + 1)} \,.\\
    \end{array}
\right.
\]
Strong convexity is applied, and we obtain:
\begin{align}
\label{eq:doublecompress_after_applied_strongcv}
\begin{split}
    \E V_{k+1} &\leq \left(1 - 2 \gamma \mu \left(1 - \frac{\gamma L (\omgC^\dwn +1) }{2} - \frac{\gamma L A_\cst }{N} \right) \right) \SqrdNrm{w_{k} - w_*} \\
    &\qquad+ 2\gamma^2 \cst D_{\cst} \frac{1}{N} \sum_{i=1}^N \Expec{\lnrm h_k^i - h_*^i \rnrm^2}{\Ftotal} +  2 \gamma^2 \frac{E_{}}{N}\,.
\end{split}
\end{align}

To guarantee a $(1 - \gamma \mu)$ convergence, constants must verify:
\[
\left\{
    \begin{array}{ll}
    	\ffrac{\gamma L (\omgC^\dwn + 1)}{2} - \ffrac{\gamma L A_\cst }{N} \leq \frac{1}{2} \text{\quad (which is already verified)}\\
    	\begin{split}
    	D_{\cst} \leq 1 - \gamma \mu &\Longleftrightarrow \ffrac{\omgC^\dwn + 1}{\cst} \bigpar{\frac{\omgC^\up + 1}{p} - 1} \leq p (3 \alpha - 2 \alpha^2 \omgC^\up - 2 \alpha) - \gamma \mu \\
    	&\Longleftrightarrow \cst \geq \ffrac{(\omgC^\dwn + 1) \bigpar{\frac{\omgC^\up + 1}{p} - 1}}{\alpha p \left( 3 - 2 \alpha (\omgC^\up + 1) \right) - \gamma \mu}
    	\end{split}\,.\\
    \end{array}
\right.
\]

In the following we will consider that $\frac{\gamma \mu}{\alpha} = \underset{\mu \to 0}{o} (1)$  which is possible because $\alpha$ is independent of $\mu$ (it depends only of $\omgC^\up$ and $\omgC^\dwn$) and it result to:
\[
\alpha p \left( 3 - 2 \alpha (\omgC^\up + 1)\right) - \gamma \mu  \underset{\mu \to 0}{\sim} \alpha p \left( 3 - 2 \alpha (\omgC^\up + 1)\right)
\]

Thus, the condition on $\cst$ becomes:
\[
 \ffrac{(\omgC^\dwn + 1)\bigpar{\frac{\omgC^\up + 1}{p} - 1}}{\alpha p \left( 3 - 2 \alpha (\omgC^\up + 1)\right)} \leq \cst\,,
\]

which is correct only if $\alpha \leq \ffrac{3}{2(\omgC^\up +1)}$.

And we obtain the following conditions on $C$:
\[
\ffrac{(\omgC^\dwn + 1) \bigpar{\frac{\omgC^\up + 1}{p} - 1}}{\alpha p \left( 3 - 2 \alpha (\omgC^\up + 1) \right)} \leq \cst \leq \frac{N - \gamma L (\omgC^\dwn + 1) \left(N + \frac{4(\omgC^\up + 1)}{p} - 2 \right)}{4 \gamma L p \alpha \left( 2 \alpha (\omgC^\up +1) - 1 \right)} \,.
\]

It follows, that the above interval is not empty if:
\begin{align*}
    \ffrac{(\omgC^\dwn + 1) \bigpar{\frac{\omgC^\up + 1}{p} - 1}}{\alpha p \left( 3 - 2 \alpha (\omgC^\up + 1) \right)} \leq \frac{N - \gamma L (\omgC^\dwn + 1) \left(N + \frac{4(\omgC^\up + 1)}{p} - 2 \right)}{4 \gamma L p \alpha \left( 2 \alpha (\omgC^\up +1) - 1 \right)} \,.
\end{align*}

For sake of clarity we denote momentarily $\tilde{\gamma} = (\omgC^\dwn +1) \gamma L$ and $\Pi = \frac{\omgC^\up + 1}{p}$, hence the below condition becomes:
\begin{align*}
    &\qquad 8 \alpha (\omgC^\up +1) (\Pi - 1) \tilde{\gamma}   - 4 (\Pi - 1) \tilde{\gamma}  \leq 3N - 3 \tilde{\gamma} \left(N + 2  + 4\Pi - 2 ) \right) \\
    &\qqquad\qquad - 2 \alpha (\omgC^\up + 1)N + 2 \alpha \tilde{\gamma}  (\omgC^\up +1 ) \left(N + 4\Pi - 2 \right) \\
    \Longleftrightarrow&\qquad 2 \alpha (\omgC^\up +1)  (N - \tilde{\gamma}  (N + 2) ) \leq 3N - \tilde{\gamma}  \left( 3 N + 8 \Pi -2 \right)
\end{align*}

And at the end, we obtain:
\begin{align*}
    \alpha \leq \ffrac{3N - \gamma L  (\omgC^\dwn + 1) \left( 3 N + \frac{8(\omgC^\up + 1)}{p} - 2 \right)}{2 (\omgC^\up +1)  (N - \gamma L  (\omgC^\dwn + 1) (N + 2) )} \,.
\end{align*}

Again, this implies two conditions on gamma:
\[
\left\{
    \begin{array}{ll}
    	3N - \gamma L  (\omgC^\dwn + 1) \left( 3 N + 8 \frac{\omgC^\up + 1}{p} -2 \right) > 0 \Longleftrightarrow \gamma < \ffrac{3}{(\omgC^\dwn + 1) \left( 3 + \frac{8 (\omgC^\up - 1) - 2p}{Np} \right) L}\\
    	N - \gamma L  (\omgC^\dwn + 1) (N + 2) > 0 \Longleftrightarrow \gamma < \ffrac{1}{(\omgC^\dwn + 1) \left( 1 + \frac{2}{N}\right) L }\,. \\
    \end{array}
\right.
\]

The constant $\cst$ exists, and from \cref{eq:doublecompress_after_applied_strongcv} we are allowed to write:
\begin{align*}
    &\E V_{k+1} \leq (1- \gamma \mu) \E V_k + 2 \gamma^2 \frac{E_{}}{N} \\
    \Longleftrightarrow\qquad&  \E V_{k+1} \leq (1- \gamma \mu)^{k+1} \E V_0 + 2 \gamma^2 \frac{E_{}}{N} \times \frac{1 - (1- \gamma \mu)^{k+1}}{\gamma \mu} \\
    \Longrightarrow\qquad&  \E V_{k+1} \leq (1- \gamma \mu)^{k+1} V_0 + 2 \gamma \frac{E_{}}{\mu N} \,.
\end{align*}

Because $V_0 = \E \SqrdNrm{w_0 - w_*} + 2 \gamma^2 \cst \frac{1}{N} \sum_{i=0}^N \SqrdNrm{h_*^i} \leq \SqrdNrm{w_0 - w_*} + 2 \cst \gamma^2 B^2$ (using \Cref{asu:bounded_noises_across_devices}), we can write:
\begin{align*}
    \E V_{k+1} = (1- \gamma \mu)^{k+1} \left(\SqrdNrm{w_0 - w_*} + 2 \cst \gamma^2 B^2 \right) + 2 \gamma \frac{E}{\mu N} \,.
\end{align*}

Thus, we highlighted that the Lyapunov function 
\[V_k = \SqrdNrm{w_k - w_*} + 2 \gamma^2 \cst \frac{1}{N} \sum_{i=1}^N \SqrdNrm{h_k^i - h_*^i}
\]

is a $(1 - \gamma \mu)$ contraction if $\cst$ is taken in a given interval, with $\gamma$ and $\alpha$ satisfying some conditions. This guarantee the convergence of the \texttt{Artemis} using version $1$ or $2$ with $\alpha \neq 0$ (algorithm with uni-compression or bi-compression combined with a memory mechanism).

\end{proof}


\subsection{Proof of Theorem - Polyak-Ruppert averaging}
\label{app:doublecompress_avg}

\begin{theorem}[Unidirectional or bidirectional compression using memory and averaging]
\label{thm:doublecompress_avg}
We suppose now that $F$ is convex, thus $\mu = 0$ and we consider that \Cref{asu:partial_participation,asu:expec_quantization_operator,asu:cocoercivity,asu:noise_sto_grad,asu:bounded_noises_across_devices} hold. We use $w_*$ to indicate the optimal parameter such that $\nabla F (w_*) = 0$, and we note $h_*^i = \nabla F_i(w_*)$. We define the Lyapunov function:
\[
V_k = \SqrdNrm{w_k - w_*} + 2 \gamma^2 \cst \frac{1}{N} \sum_{i=1}^N \SqrdNrm{h_k^i - h_*^i}\,.
\]
We defined $\cst \in \mathbb{R}^*$, such that:
\begin{align}
\ffrac{(\omgC^\dwn + 1) \bigpar{\frac{\omgC^\up + 1}{p} - 1}}{\alpha p \left( 3 - 2 \alpha (\omgC^\up + 1) \right)} \leq \cst \leq \frac{N - \gamma L (\omgC^\dwn + 1) \left(N + \frac{4(\omgC^\up + 1)}{p} - 2 \right)}{4 \gamma L p \alpha \left( 2 \alpha (\omgC^\up +1) - 1 \right)} \,.
\end{align}

Then running variant of \Artemis~with $\alpha \neq 0$, hence with a memory mechanism, and using Polyak-Ruppert averaging, the convergence of the algorithm is guaranteed if:
\begin{equation}
\label{app:thm:avg:conditions_gamma_alpha}
\left\{
    \begin{array}{l}
    	\ffrac{1}{2(\omgC^\up + 1)} \leq \alpha < \min \left( \ffrac{3}{2(\omgC^\up + 1)}, \quad  \ffrac{3N - \gamma L  (\omgC^\dwn + 1) \left( 3 N + \frac{8(\omgC^\up + 1)}{p} - 2 \right)}{2 (\omgC^\up +1)  (N - \gamma L  (\omgC^\dwn + 1) (N + 2) )} \right)\\
    	\gs
    	\gamma < \min
    	\left\{
        \begin{array}{l}
    	\ffrac{1}{(\omgC^\dwn + 1) \left( 1 + \frac{2}{Np}\right) L }, \quad \ffrac{3}{(\omgC^\dwn + 1) \left( 3 + \frac{8 (\omgC^\up - 1) - 2p}{Np} \right) L} ,  \\
    	\ffrac{N}{(\omgC^\dwn + 1)\left(N + \frac{4(\omgC^\up + 1)}{p} - 2 \right) L} \\
    	\end{array}
        \right\}\,.
    \end{array}
\right.      
\end{equation}

And we have the following bound:
\begin{align}
\label{app:thm:rates_cvg_avg}
    F\left( \frac{1}{K}\sum_{k=0}^K w_k \right) - F(w_*) \leq \ffrac{\SqrdNrm{w_0 - w_*} + 2 \cst \gamma^2 B^2}{\gamma K} + 2 \gamma \frac{E}{N} \,,
\end{align}

with $E = (\omgC^\dwn + 1) \bigpar{\frac{2(\omgC^\up + 1)}{p} -1} \frac{\sigmstar^2}{b} + 2p C \bigpar{2\alpha^2 (\omgC^\up + 1) - \alpha } \frac{\sigmstar^2}{b}$.

\Cref{app:thm:rates_cvg_avg} can be written as in \Cref{thm:main_PRave} if we take $\gamma = \min \left( \sqrt{\frac{N \delta_0^2}{2E K}}; \gamma_{\max }\right)$, where $\gamma_{\max}$ is the maximal possible value of $\gamma$ as precised by \Cref{app:thm:avg:conditions_gamma_alpha}:
\begin{align*}
    F\left( \frac{1}{K}\sum_{k=0}^K w_k \right) - F(w_*) \leq 2 \max \left(\sqrt{\frac{2 \delta_0^2  E}{N K}}; \frac{\delta_0^2}{\gamma_{\max} K} \right) +  \frac{ 2 \gamma_{\max} \cst B^2}{K}
\end{align*}

\end{theorem}


\begin{proof}

Starting from \cref{eq:doublecompress_before_stgcvx} from the proof of \Cref{app:thm:with_mem}:
\begin{align*}
\begin{split}
    \E V_{k+1} &\leq \SqrdNrm{w_{k} - w_*} -  2\gamma \left(1 - \gamma L \left(\frac{\omgC^\dwn +1}{2} + \frac{A_\cst}{N} \right) \right) \PdtScl{\nabla F(w_k)}{w_k - w_*} \\
    &\qquad+ 2\gamma^2 \cst D_{\cst} \frac{1}{N^2} \sum_{i=1}^N \Expec{\lnrm h_k^i - h_*^i \rnrm^2}{\Ftotal} +  \frac{2 \gamma^2}{N} E_{}  \,.
\end{split}
\end{align*}

But this time, instead of applying strong convexity of $F$, we apply convexity (\Cref{asu:strongcvx} but with $\mu=0$):
\begin{align}
\label{eq:avg_after_cvx}
\begin{split}
    \E V_{k+1} &\leq \SqrdNrm{w_{k} - w_*} - 2\gamma \left(1 - \gamma L \left(\frac{\omgC^\dwn +1}{2} + \frac{A_\cst}{N} \right) \right) \left(F(w_k) - F(w_*) \right) \\
    &\qquad+ 2\gamma^2 \cst D_{\cst} \frac{1}{N^2} \sum_{i=1}^N \Expec{\lnrm h_k^i - h_*^i \rnrm^2}{\Ftotal} +  \frac{2 \gamma^2}{N} E_{} \,.
\end{split}
\end{align}

As in \Cref{app:thm:with_mem}, we want:
\begin{align}
\label{eq:condition_avg_un_demi}
\begin{split}
    &\qquad \gamma L \left(\frac{\omgC^\dwn + 1}{2} + \frac{A_\cst}{N} \right) \leq \frac{1}{2} \\
    \Longleftrightarrow&\qquad \cst \leq \frac{N - \gamma L (\omgC^\dwn + 1) \left(N + \frac{4(\omgC^\up + 1)}{p} - 2 \right)}{4 \gamma L p \alpha \left( 2 \alpha (\omgC^\up +1) - 1 \right)} \,.
\end{split}
\end{align}

which holds only if the numerator and the denominator are positive:
\[
\left\{
    \begin{array}{ll}
        N - \gamma L (\omgC^\dwn + 1) \left(N + \frac{4(\omgC^\up + 1)}{p} - 2 \right) > 0 \Longleftrightarrow  \gamma < \ffrac{N}{(\omgC^\dwn + 1)\left(N + \frac{4(\omgC^\up + 1)}{p} - 2 \right) L}  \\
    	2 \alpha (\omgC^\up +1) - 1 \leq 0 \Longleftrightarrow \alpha \geq \ffrac{1}{2(\omgC^\up + 1)} \,.\\
    \end{array}
\right.
\]

Returning to \cref{eq:avg_after_cvx}, taking benefit of \cref{eq:condition_avg_un_demi} and passing $F(w_k) - F(w_*)$ on the left side gives:
\begin{align*}
    \gamma (F(w_k) - F(w_*)) &\leq \SqrdNrm{w_{k} - w_*} + 2\gamma^2 \cst D_{\cst} \frac{1}{N^2} \sum_{i=1}^N \Expec{\lnrm h_k^i - h_*^i \rnrm^2}{\Ftotal} - \E V_{k+1} +  \frac{2 \gamma^2}{N} E_{} \,.
\end{align*}

If $D_{\cst} \leq 1$, then:
\begin{align*}
\begin{split}
    \gamma ( F(w_k) - F(w_*)) &\leq \E V_k - \E V_{k+1} +  \frac{2 \gamma^2}{N} E_{} \,,
\end{split}
\end{align*}

summing over all $K$ iterations:
\begin{align*}
\begin{split}
    \gamma \left( \frac{1}{K} \sum_{k=0}^K F(w_k)  - F(w_*) \right) &\leq \frac{1}{K} \sum_{k=0}^K \left( \E V_k - \E V_{k+1} + 2 \gamma^2 \frac{E}{N} \right) \\
    &\leq \frac{\E V_0 - \E V_{K+1}}{K} + 2 \gamma^2 \frac{E}{N} \text{\quad because $E$ is independent of $K$.}
\end{split}
\end{align*}

Thus, by convexity:
\begin{align*}
    F\left( \frac{1}{K}\sum_{k=0}^K w_k \right) - F(w_*) \leq \frac{1}{K} \sum_{k=0}^K F(w_k) - F(w_*) \leq  \frac{V_0}{\gamma K} + 2\gamma \frac{E}{N} \,.
\end{align*}

Last step is to extract conditions over $\gamma$ and $\alpha$ from requirement $D_{\cst} \leq 1$:
\[
    D_{\cst} < 1 \Longleftrightarrow \ffrac{\omgC^\dwn + 1}{2 \cst} \bigpar{\frac{2(\omgC^\up + 1)}{p} -2}  < 3 \alpha - 2 \alpha^2 \omgC^\up - 2 \alpha\Longleftrightarrow \cst > \ffrac{(\omgC^\dwn + 1) \bigpar{\frac{\omgC^\up + 1}{p} - 1}}{\alpha p \left( 3 - 2 \alpha (\omgC^\up + 1) \right)} \,,\\
\]

and the second inequality is correct only if $\alpha \leq \ffrac{3}{2(\omgC^\up +1)}$.

From this development follows the following conditions on $p$, which are equivalent to those obtain in \Cref{app:thm:with_mem}
\[
\ffrac{(\omgC^\dwn + 1) \bigpar{\frac{\omgC^\up + 1}{p} - 1}}{\alpha p \left( 3 - 2 \alpha (\omgC^\up + 1) \right)} \leq \cst \leq \frac{N - \gamma L (\omgC^\dwn + 1) \left(N + \frac{4(\omgC^\up + 1)}{p} - 2 \right)}{4 \gamma L p \alpha \left( 2 \alpha (\omgC^\up +1) - 1 \right)} \,.
\]

This interval is not empty:
\begin{align*}
    &\qquad \ffrac{(\omgC^\dwn + 1) \bigpar{\frac{\omgC^\up + 1}{p} - 1}}{\alpha p \left( 3 - 2 \alpha (\omgC^\up + 1) \right)} \leq \frac{N - \gamma L (\omgC^\dwn + 1) \left(N + \frac{4(\omgC^\up + 1)}{p} - 2 \right)}{4 \gamma L p \alpha \left( 2 \alpha (\omgC^\up +1) - 1 \right)}  \\
    \Longleftrightarrow&\qquad \alpha \leq \ffrac{3N - \gamma L  (\omgC^\dwn + 1) \left( 3 N + \ffrac{8(\omgC^\up + 1)}{p} - 2 \right)}{2 (\omgC^\up +1)  (N - \gamma L  (\omgC^\dwn + 1) (N + 2) )}  \,.
\end{align*}

Again, this implies two conditions on gamma:
\[
\left\{
    \begin{array}{ll}
    	3N - \gamma L  (\omgC^\dwn + 1) \left( 3 N + 8 \frac{\omgC^\up + 1}{p} -2 \right) > 0 \Longleftrightarrow \gamma < \ffrac{3}{(\omgC^\dwn + 1) \left( 3 + \frac{8 (\omgC^\up - 1) - 2p}{Np} \right) L}\\
    	N - \gamma L  (\omgC^\dwn + 1) (N + 2) > 0 \Longleftrightarrow \gamma < \ffrac{1}{(\omgC^\dwn + 1) \left( 1 + \frac{2}{N}\right) L }\,. \\
    \end{array}
\right.
\]

which guarantees the existence of $\cst$ and thus the validity of the above development.

As a conclusion:
\begin{align*}
    F\left( \frac{1}{K}\sum_{k=0}^K w_k \right) - F(w_*) &\leq  \frac{V_0}{\gamma K} + 2\gamma \frac{E}{N} \leq \ffrac{\SqrdNrm{w_0 - w_*} + 2 \cst \gamma^2 B^2}{\gamma K} + 2 \gamma \frac{E}{N} \,.\\
    &\leq   \frac{\SqrdNrm{w_0 - w_*}}{\gamma K} + 2 \gamma \left(\frac{E}{N} +  \frac{\cst B^2}{K} \right) \,.
\end{align*}

Next, our goal is to define the optimal step size $\gamma_{opt}$. With this aim, we bound $2 \gamma \ffrac{C B^2}{K}$ by $2 \gamma_{\max} \ffrac{C B^2}{ K}$. This leads to ignore this term when optimizing the step size and thus to obtain a simpler expression of $\gamma_{opt}$. This approximation is relevant, because $\frac{B^2}{K}$ is ``small''. And we obtain: 

\begin{align*}
    F\left( \frac{1}{K}\sum_{k=0}^K w_k \right) - F(w_*) &\leq  \frac{\SqrdNrm{w_0 - w_*}}{\gamma K} + 2 \gamma \frac{E}{N} + 2 \gamma_{\max} \ffrac{C B^2}{ K}  .
\end{align*}
This is valid for all variants of \Artemis, with step-size in \cref{table:summary_convergence} and $E,p$ in \Cref{thm:cvgce_artemis}. Subsequently, the ``optimal'' step size (at least the one minimizing the upper bound) is $$\gamma_{opt} = \sqrt{\ffrac{\SqrdNrm{w_0 - w_*} N }{2 E K}} \,,$$ 

resulting in a convergence rate as $2 \sqrt{\frac{2 \SqrdNrm{w_0 - w_*}  E}{N K}} +\frac{ 2 \gamma_{\max} \cst B^2}{K}$, if this step size is allowed. If $\sqrt{\frac{\SqrdNrm{w_0 - w_*} N}{2 E K}}\geq \gamma_{\max}$ $\left( \Longrightarrow \frac{2 \gamma_{\max} E}{N} \leq \frac{\SqrdNrm{w_0 - w_*}}{\gamma_{\max} K} \right)$, then the bias term dominates and the upper bound is $ 2\frac{\SqrdNrm{w_0 - w_*}}{\gamma_{\max} K} +\frac{ 2 \gamma_{\max} \cst B^2}{K} $.
Overall, the convergence rate is given by:
\begin{equation*}
       F\left( \frac{1}{K}\sum_{k=0}^K w_k \right) - F(w_*) \leq 2 \max \left(  \sqrt{\frac{2 \SqrdNrm{w_0 - w_*}  E}{N K}}; \frac{\SqrdNrm{w_0 - w_*}}{\gamma_{\max} K} \right) +  \frac{ 2 \gamma_{max} \cst B^2}{K}  \,.
\end{equation*}

\end{proof}

\subsection{Proof of Theorem S4 - convergence in distribution}
\label{app:distrib_convergence}

In this section, we give the proof of \Cref{thm:cvdist} in a full participation setting. The theorem is decomposed into two main points, that are respectively derived from \Cref{prop:cvdistpoint1,prop:cvdistpoint2}, given in \Cref{subsec:point1,subsec:point2}. We first introduce a few notations in \Cref{subsec:backgroundMC}.

We consider in this section the \emph{Stochastic Sparsification} compression operator $\C_\qbern$, which is defined as follows: for any $x\in \rset^d$, $\C_\qbern(x) \overset{dist}{=} \frac{1}{\qbern}(x_1 B_1, \dots , x_d B_d)$, with $(B_1, \dots, B_d)\sim \mathcal B(\qbern)^{\otimes n}$ i.i.d.~Bernoullis with mean $\qbern$. That is, each coordinate is independently assigned to 0 with probability $1-\qbern$ or rescaled by a factor $\qbern^{-1}$ in order to get an unbiased operator. 

\begin{lemma}
This compression operator satisfies \Cref{asu:expec_quantization_operator} with $\omgC = q^{-1}-1$.

Moreover, if I consider a random variable $(B_1, \dots, B_d)\sim \mathcal B(\qbern)^{\otimes n}$ and define \emph{almost surely} $\C_\qbern(x) \overset{a.s.}{=} \frac{1}{\qbern}(x_1 B_1, \dots , x_d B_d)$, then we also have that for any $x,y\in \rset ^d$, $\C_\qbern(x)-\C_\qbern(y) = \C_\qbern(x-y)$.
\end{lemma}

\subsubsection{Background on distributions and Markov Chains}\label{subsec:backgroundMC}
We consider \Artemis~iterates  $(w_k, (h_k^i)_{i\in \llbracket1, N\rrbracket})_{k \in \N} \in \rset^{d(1+N)}$ with the following update equation:
\begin{equation}
 \left\lbrace   \begin{array}{lll}
    & w_{k+1} &= w_k - \gamma \mathcal{C}_\dwn \left(\ffrac{1}{N}\sum_\iN \mathcal{C}_\up\left(\gwk^i - h_k^{i} \right)+ h_k^{i} \right) \,  \\
 \forall i \in    \llbracket1, N\rrbracket, \qquad & h_{k+1}^{i} &= h_k^{i} + \alpha \mathcal{C}_\up\left(\gwk^i - h_k^{i} \right) 
\end{array} \right.
\end{equation}

We see the iterates, for a constant step size $\gamma$, as a homogeneous Markov chain, and denote $R_{\gamma,v}$ the \emph{Markov kernel}, which is the equivalent for continuous spaces of
the \emph{transition matrix} in finite state spaces.  Let $R_{\gamma,v}$ be the Markov kernel
on $(\rset^{d(1+N)}, \mathcal{B}(\rset^{d(1+N)}))$ associated with the SGD iterates
$(w_k, \tau (h^i_k)_{i \in \llbracket1, N\rrbracket})_{k \geq 0}$ for a variant $v$ of \Artemis, as defined in \Cref{algo} and  with $\tau$ a constant specified afterwards, where $\mathcal{B}(\rset^{d(1+N)})$ is the Borel
$\sigma$-field of $\rset^{d(1+N)}$. \citet{meyn:tweedie:2009} provide
an introduction to Markov chain theory. 
For readability, we now denote $(h^i_k)_{i}$ for $(h^i_k)_{i \in \llbracket1, N\rrbracket}$.

\begin{definition}
For any initial distribution
$\nu_0$ on $\mathcal{B}(\rset^{d(1+N)})$ and $k \in \N$, $\nu_0
R_{\gamma, v}^k$ denotes the distribution of $(w_{k}, \tau (h^i_k)_{i   })$ starting at $(w_0, \tau (h^i_0)_{i   })$
distributed according to $\nu_0$. 
\end{definition}

We can make the following comments:
\begin{enumerate}
\item \textbf{Initial distribution.} We  consider deterministic initial points, i.e.,  $(w_0, \tau (h^i_0)_{i})$ follows a Dirac at point $(w_0, \tau (h^i_0)_{i})$. We denote this Dirac $\delta_{w_0}\otimes \otimes_{i=1}^N \delta_{\tau h^i_0} \overset{\text{not}}{=} \delta_{w_0}\otimes \delta_{\tau h^1_0} \otimes \dots \otimes \delta_{\tau h^N_0}$.
\item \textbf{Notation in the main text:} In the main text, for simplicity, we used $\Theta_k$ to denote the distribution of $w_k$ when launched from  $(w_0, \tau (h^i_0)_{i})$. Thus $\Theta_k$ corresponds to the distribution of the projection on first $d$ coordinates of  $((\delta_{w_0}\otimes \otimes_{i=1}^N \delta_{\tau h^i_0}) R_{\gamma}^k)$.
\item \textbf{Case without memory:} In the memory-less case, we have $(h^i_k)_{k\in \N} \equiv 0$, and could restrict ourselves to a Markov kernel on $(\rset^d, \mathcal{B}(\rset^d))$.
\end{enumerate}

For any variant $v$ of \Artemis, we  prove that $(w_k, (h_k^i)_i)_{k \geq 0}$ admits a limit stationary
distribution \begin{equation}
    \label{eq:notation_PIgv}
    \Pi_{\gamma, v} = \pi_{\gamma, v, w} \otimes \pi_{\gamma, v, (h)}
\end{equation} and quantify the convergence of $((\delta_{w_0}\otimes \otimes_{i=1}^N \delta_{\tau h^i_0})
R_{\gamma}^k)_{k \geq 0}$ to $\pigv$, in terms of Wasserstein metric $\mathcal W_2$. 

\begin{definition}\label{def:wassdist}
For all probability measures $\nu$ and $\lambda$ on $\mathcal{B}(\rset^d)$, 
such that $\int_{\rset^d} \SqrdNrm{w} \rmd \nu (w) < +\infty$ and $\int_{\rset^d} \SqrdNrm{w} \rmd \lambda (w) \leq +\infty $, define the squared Wasserstein distance of order $2$ between $\lambda$ and $\nu$ by 
\begin{equation}
  \W_2(\lambda, \nu):= \inf
_{\xi \in  \Gamma(\lambda , \nu)} \int \| x-y\|^{2 } \xi(dx,
  dy) , \label{eq:defwass}
\end{equation}  
where $ \Gamma(\lambda, \nu)$ is the set of
probability measures $\xi$ on $\mathcal{B}(\R^{d} \times \R^{d})$ satisfying for all $\mathsf{A} \in \mathcal{B}(\rset^d)$, $\xi(\mathsf{A} \times \rset^d) = \nu(\mathsf{A})$, $\xi( \rset^d \times \mathsf{A} ) = \lambda(\mathsf{A})$. 
\end{definition}

\subsubsection{Proof of the first point in \Cref{thm:cvdist}} \label{subsec:point1}
 We prove the following proposition:
\begin{proposition}\label{prop:cvdistpoint1}
Under \Cref{asu:strongcvx,asu:cocoercivity,asu:noise_sto_grad,asu:bounded_noises_across_devices,asu:expec_quantization_operator}, for $\C_p$ the \emph{Stochastic Sparsification} compression operator, for any variant $v$ of the algorithm, there exists a limit distribution $\pigv$, which is stationary, such that for any $k$ in $\N$, for any $\gamma$ satisfying conditions given in \Cref{app:thm:without_mem,app:thm:with_mem}:
\begin{align*}
    \W_2((\delta_{w_0}\otimes \otimes_{i=1}^N \delta_{\tau h^i_0})
R_{\gamma}^k, \pigv) &\le
\\& \hspace{-2.5cm} (1-\gamma \mu )^k \int_{(w',h')\in \rset^{d(1+N)}} \SqrdNrm{(w_0,\tau(h^i_0)_{i})-(w',\tau (h^i)'_i)} \rmd \pigv (w',(h^i)'_i).
\end{align*}
\end{proposition}

Point 1 in \Cref{thm:cvdist} is derived from the proposition above using $\pi_{\gamma,v} = \pi_{\gamma,v,w}$, with $\pi_{\gamma,v,w}$ as in \Cref{eq:notation_PIgv}, the limit distribution of the main  iterates $(w_k)_{k\in \N}$ and the observation that:
\begin{align*}
\W_2(\Theta_k, \pi_{\gamma,v}) &\le   \W_2((\delta_{w_0}\otimes \otimes_{i=1}^N \delta_{\tau h^i_0}) \Rgv^k, \pigv)\\ \hspace{-1cm}&\le (1-\gamma \mu )^k \int_{(w',h')\in \rset^{d(1+N)}} \SqrdNrm{(w_0,\tau(h^i_0)_{i })-(w',\tau (h^i)'_i)}  \rmd \pigv (w',(h^i)'_i)\\ & =     (1-\gamma \mu )^k  C_0.
\end{align*}

The sketch of the proof is simple:
\begin{itemize}
    \item We introduce a \emph{coupling of random variables} following respectively $\nu_0^a \Rgv^k$ and $\nu_0^b \Rgv^k$, and show that under the assumptions given in the proposition:
    \begin{align*}
        \W_2(\nu_0^a \Rgv^{k+1}, \nu_0^b \Rgv^{k+1}) &\le (1-\gamma \mu)  \W_2(\nu_0^a \Rgv^{k}, \nu_0^b \Rgv^{k}).
    \end{align*}
    This proof follows the same line as the proof of \Cref{app:thm:without_mem,app:thm:with_mem}.
    \item We deduce that $((\delta_{w_0}\otimes \otimes_{i=1}^N \delta_{\tau h^i_0})) \Rgv^k)$ is a Cauchy sequence in a Polish space, thus the existence and stability of the limit, we show that this limit is independent from $(\delta_{w_0}\otimes \otimes_{i=1}^N \delta_{\tau h^i_0}))$ and conclude.
\end{itemize}

\begin{proof}
We consider two initial distributions $\nu_0^a $ and $\nu_0^b$  for $(w_0, \tau (h^i_0)_{i   })$ with finite second moment and $\gamma>0$. Let $(w^a_0,\tau (h^{i,a}_0)_i)$ and $(w^b_0, \tau(h^{i,b}_0)_i)$ be respectively distributed according to $\nu_0^a $ and $\nu_0^b$. Let $(w^a_k, \tau(h^{i,a}_k)_i)_{k\geq 0}$ and $(w^b_k, \tau(h^{i,b}_k)_i)_{k\geq 0}$ the \Artemis~iterates, respectively starting from $(w^a_0, \tau(h^{i,a}_0)_i)$ and $(w^b_0, \tau(h^{i,b}_0)_i)$, and \emph{sharing the same sequence of noises}, i.e.,
\begin{itemize}
    \item built with the same gradient oracles $\g_{k+1}^{i,a} =\g_{k+1}^{i,b}$ for all $k \in \N, i \in \llbracket1, N\rrbracket$.
    \item the compression operator used for both recursions is  almost surely the same, for any iteration $k$, and both uplink and downlink compression.  We denote these operators $\mathcal{C}_{\dwn,k}$ and $\mathcal{C}_{\up,k}$ the compression operators at iteration $k$ for respectively the uplink compression and downlink compression. 
\end{itemize}

We thus have the following updates, for any $u \in \lbrace a, b\rbrace$: 
\begin{align}
\left \lbrace\begin{array}{lll}
&w^u_{k+1} &= w^u_k - \gamma \mathcal{C}_{\dwn,k} \left(\ffrac{1}{N}\sum_\iN\mathcal{C}_{\up,k} \left(\gwk^i - h_k^{i, u} \right)+ h_k^{i,u} \right) \, \label{eq:updatewab} \\
\forall i \in \llbracket1; n \rrbracket \qquad & h_{k+1}^{i,u} &= h_k^{i,u} + \alpha \mathcal{C}_{\up,k} \left(\gwk^i - h_k^{i, u} \right)
\end{array}
     \right.
\end{align}

The proof is obtained by induction. For a $k $ in $ \N$, let  $\left((w^a_k, \tau(h^{i,a}_k)_i), (w^b_k, \tau(h^{i,b}_k)_i)\right)$ be a coupling of random variable in $ \Gamma(\nu_0^a \Rgv^k, \nu_0^b \Rgv^k)$ -- as in \Cref{def:wassdist} --, that achieve the equality in the definition, i.e., \begin{equation}\label{eq:wasswitheq}
    \W_2(\nu_0^a \Rgv^k, \nu_0^b \Rgv^k) = \E\left[\left\|(w^a_k,\tau (h^{i,a}_k)_i)-(w^b_k, \tau(h^{i,b}_k)_i)\right\|^2\right].
\end{equation} Existence of such a couple is given by \citep[theorem 4.1]{Vil_2009}. 

Then $\left((w^a_{k+1}, \tau(h^{i,a}_{k+1})_i), (w^b_{k+1}, \tau(h^{i,b}_{k+1})_i)\right)$ obtained after one update from \Cref{eq:updatewab} belongs to  $ \Gamma(\nu_0^a \Rgv^{k+1}, \nu_0^b \Rgv^{k+1})$, and as a consequence:
\begin{align*}
    \W_2(\nu_0^a \Rgv^{k+1}, \nu_0^b \Rgv^{k+1})&\le \E \left[ \SqrdNrm{(w^a_{k+1}, \tau(h^{i,a}_{k+1})_i)- (w^b_{k+1}, \tau(h^{i,b}_{k+1})_i))}\right] \\
    & = \E \left[ \SqrdNrm{w^a_{k+1} - w^b_{k+1}}\right]  + \tau^2 \sum_{i=1}^N \E \left[ \SqrdNrm{h^{i,a}_{k+1}-h^{i,b}_{k+1}}\right]\\
     & = \E \left[ \SqrdNrm{w^a_{k+1} - w^b_{k+1}}\right]  + 2\gamma^2 \frac{\cst}{N} \sum_{i=1}^N \E \left[ \SqrdNrm{h^{i,a}_{k+1}-h^{i,b}_{k+1}}\right], 
\end{align*}

with $\tau^2 = 2\gamma^2 \frac{\cst}{N}$, where $\cst$ depends on the variant as in \Cref{thm:cvgce_artemis}.

We now follow the proof of the previous theorems to control respectively
$ \E \left[ \SqrdNrm{w^a_{k+1} - w^b_{k+1}}\right] $ and $ \E \left[ \SqrdNrm{h^{i,a}_{k+1}-h^{i,b}_{k+1}}\right] $.

First, following the proof of \Cref{eq:proof_end_decomposition_doublecompressnomem}, we  get, using the fact that the compression operator is random sparsification, thus that $\C(x)- \C(y)= \C(x-y)$: 
\begin{align*}
    \E \left[ \SqrdNrm{w^a_{k+1} - w^b_{k+1}} |\Fdwncomp_k \right]  & \leq \SqrdNrm{w^a_{k} - w^b_k} - 2 \gamma {\PdtScl{\nabla F(w_k^a)-\nabla F(w_k^b)}{w^a_k - w^b_k}} \\
    &\qquad+ \frac{2 (2\omgC^\up + 1)(\omgC^\dwn + 1)\gamma^2}{N^2} \sum_{i=1}^N \Expec{\lnrm \g^i_{k+1}(w_k^a)- \g^i_{k+1}(w_k^b) \rnrm^2}{\Fdwncomp_k} \\
    &\qquad+ \frac{2\omgC^\up (\omgC^\dwn+ 1) \gamma^2}{N^2} \sum_{i=1}^N \Expec{\lnrm h_k^{i,a} - h_k^{i,b} \rnrm^2}{\Fdwncomp_k} \\
    &\qquad+ \gamma^2 (\omgC^\dwn + 1)  L {\PdtScl{\nabla F(w_k^a)-\nabla F(w_k^b)}{w^a_k - w^b_k}}.
\end{align*}
This expression is nearly the same as in \Cref{eq:proof_end_decomposition_doublecompressnomem}, apart from the constant term depending on $\sigmstar^2$ that disappears.

Note that with a more general compression operator, for example for quantization, it is not possible to derive such a result.

Similarly, we control $ \E \left[ \SqrdNrm{h^{i,a}_{k+1}-h^{i,b}_{k+1}}\right] $ using the same line of proof as for \Cref{eq:proof_doublecompressnomem_combinaison_mem1}, resulting in: 
\begin{align*}
        \frac{1}{N^2} \sum_{i=0}^N \Expec{\lnrm h_{k+1}^{a,i} - h_{k+1}^{b,i} \rnrm^2}{\Fdwncomp_k} & \leq (1 + p \left(2 \alpha^2 \omgC^\up + 2 \alpha^2 - 3 \alpha) \right) \frac{1}{N^2} \sum_{i=0}^N \Expec{\lnrm h_{k}^{a,i} - h_{k}^{b,i} \rnrm^2}{\Fdwncomp_k} \\
    &\hspace{-1.3cm}+ 2(2 \alpha^2 \omgC^\up + 2 \alpha^2 - \alpha) \frac{1}{N^2} \sum_{i=0}^N \Expec{\lnrm \g^i_{k+1}(w_k^a)- \g^i_{k+1}(w_k^b) \rnrm^2}{\Fdwncomp_k} .
\end{align*}

Combining both equations, and using \Cref{asu:strongcvx,asu:cocoercivity} and \Cref{eq:wasswitheq} we get, under conditions on the learning rates $\alpha, \gamma$ similar to the ones in \Cref{app:thm:without_mem,app:thm:with_mem}, that 
$$  \W_2(\nu_0^a \Rgv^{k+1}, \nu_0^b \Rgv^{k+1}) \le (1-\gamma \mu)  \W_2(\nu_0^a \Rgv^{k}, \nu_0^b \Rgv^{k}).$$

And by induction:
$$ \W_2(\nu_0^a \Rgv^{k+1}, \nu_0^b \Rgv^{k+1}) \le (1-\gamma \mu)^{k+1}  \W_2(\nu_0^a, \nu_0^b).$$
\end{proof}

From the contraction above, it is easy to derive the existence of a unique stationnary limit distribution: we use Picard fixed point theorem, as in \cite{dieuleveut_bridging_2018}. This concludes the proof of \Cref{prop:cvdistpoint1}.

\subsubsection{Proof of the second point of \Cref{thm:cvdist}}\label{subsec:point2}
To prove the second point, we first detail the complementary assumptions  mentioned in the text, then show the convergence to the mean squared distance under the limit distribution, and finally give a lower bound on this quantity.

\paragraph{Complementary assumptions.}\ \\
To prove the lower bound given by  the second point,  we need to assume that the constants given in the assumptions are tight, in other words, that corresponding lower bounds exist in \Cref{asu:bounded_noises_across_devices,asu:expec_quantization_operator,asu:noise_sto_grad}.

\begin{assumption}[Lower bound on noise over stochastic gradients computation]
\label{asu:noise_sto_gradLB}
The noise over stochastic gradients at optimal global point for a mini-batch of size $b$ is \emph{lower bounded}. In other words, there exists a constant $\sigmstar \in \mathbb{R}$, such that for all $k$ in $\N$, for all $i$ in $\llbracket 1, N \rrbracket\,$, we have a.s:
\[
 \quad\Expec{\|\gwkstar^i - \nabla F_i(w_*)\|^2}{\Fdwncomp_{k}} \geq \frac{\sigmstar^2}{b}.
\]
\end{assumption}

\begin{assumption}[Lower bound on local gradient at $w_*$]
\label{asu:bounded_noises_across_devicesLB}
There exists a constant $B \in \mathbb{R}$, s.t.:
\[
\frac{1}{N}\sum_{i=1}^N \| \nabla F_i(w_*)\|^2 \geq B^2.
\]
\end{assumption}

\begin{assumption}[Lower bound on the compression operator's variance]
\label{asu:expec_quantization_operatorLB}
There exists a constant $\omgC \in \R^*$ such that the compression operators $\mathcal{C}_{\up}$ and $\mathcal{C}_{\dwn}$ verify the  following property: 
\[
\forall \Delta \in \R^d\,, \E [ \SqrdNrm{\mathcal{C}_{\up\,,\dwn}(\Delta) - \Delta}] = \omgC^{\up\,,\dwn} \SqrdNrm{\Delta} \,.
\]
\end{assumption}
This last assumption is valid for Stochastic Sparsification.

Moreover, we also assume some extra regularity on the function. This restricts the regularity of the function beyond \Cref{asu:cocoercivity} and is a purely technical assumption in order to conduct the detailed asymptotic analysis. It is valid in practice for Least Squares or Logistic regression.
\begin{assumption}[Regularity of the functions]
\label{asu:cocoercivityREG}
The function $F$ is also times continuously differentiable with second to fifth uniformly bounded derivatives: for all $k \in \{2,\dots,5\}$, $\sup_{w\in \rset^d} \|{F^{(k)}(w)}\| < \infty$.
\end{assumption}

\paragraph{Convergence of moments.} \ \\
We first prove that $\E[\|w_k-w_*\|^2] $ converges to $ \E_{w\sim \pi_{\gamma,v}}[\|w-w_*\|^2] $ as $k$ increases to $\infty$.

We have that the difference satisfies, for random variables $w_k$ and $w$ following distributions $\delta_{w_0} \Rgv^k$ and $\pi_{\gamma,v}$, and coupled such that they achieve the equality in \Cref{eq:defwass}:
\begin{align*}
   \Delta_{\E, k} :&= \E[\|w_k-w_*\|^2] - \E_{w\sim \pi_{\gamma,v}}[\|w-w_*\|^2] \\&= \E_{w_k, w\sim \pi_{\gamma,v}}\left[\|w_k-w_*\|^2- \|w-w_*\|^2\right]\\
    &= \E_{w_k, w\sim \pi_{\gamma,v}}\left[(\|w_k-w_*\| - \|w-w_*\|)(\|w_k-w_*\| + \|w-w_*\|)\right]\\
        &\overset{\text{C.S}}{\le} \left(\E_{w_k, w\sim \pi_{\gamma,v}}\left[(\|w_k-w_*\| - \|w-w_*\|)^2\right]\E_{w_k, w}\left[(\|w_k-w_*\| + \|w-w_*\|)^2\right]\right)^{1/2}\\
&\overset{\text{T.I.}}{\le} \left(\E_{w_k, w\sim \pi_{\gamma,v}}\left[(\|w_k- w\|)^2\right]\E_{w_k, w\sim \pi_{\gamma,v}}\left[(\|w_k-w_*\| + \|w-w_*\|)^2\right]\right)^{1/2}\\
&\overset{\text{(i)}}{\le} \left(\E_{w_k, w\sim \pi_{\gamma,v}}\left[(\|w_k- w\|)^2\right] 2 L\right)^{1/2}\\
&\overset{\text{(ii)}}{\le} \left(\W_2(\delta_{w_0} \Rgv^k , \pi_{\gamma,v}) 2 L\right)^{1/2}\\
& \overset{\text{(iii)}}{\to} 0 .
\end{align*}
Where we have used Cauchy-Schwarz inequality at line C.S., triangular inequality at line T.I., the fact that the moments are bounded by a constant $L$ at line (i), the fact that the distributions are coupled such that they achieve the equality in \Cref{eq:defwass} at line (ii), and finally \Cref{prop:cvdistpoint1} for the conclusion at line (iii).

Overall, this shows that the mean squared distance (i.e., saturation level) converges to the mean squared distance under the limit distribution.

\paragraph{Evaluation of $ \E_{w\sim \pi_{\gamma,v}}[\|w-w_*\|^2] $.}\ \\
In this section, we denote $\Xi_{k+1}(w_k, h_k)$ the \emph{global} noise, defined by $$\Xi_{k+1}(w_k, h_k) = \nabla F(w_k) - \C_\dwn \left(\frac{1}{N} \sum_\iN \C_\up(\g_{k+1}^i(w_k) - h^i_k)+ h_k^i \right),$$
such that $w_{k+1} = w_k - \gamma \nabla F(w_k) + \gamma \Xi_{k+1}(w_k, h_k)$.

In the following, we denote $a^\otimes 2:= a a^T$ the second order moment of $a$. We define $\tr$ the trace operator and $\cov$ the covariance operator such that $\cov(\Xi(w, h)) = \E\left[(\Xi(w, h))^{\otimes 2}\right] $, where the expectation is taken on the randomness of both compressions and the gradient oracle.  We make a final technical assumption on the regularity of the covariance matrix.
\begin{assumption} \label{asu:regnoise}
We assume that:
\begin{enumerate}
\item $\cov(\Xi(w, h))$ is continuously differentiable, and there exists constants $C$ and $C'$ such that for all $w,h\in \rset^{d(1+N)}$, $\max_{o=1, 2,3} \cov^{(o)}(w,h) \le C + C' ||(w,h)-  (w_*,h_*)||^2$. 
\item $(\Xi(w_*, h_*))$ has finite order moments up to order 8.
\end{enumerate}
\end{assumption}
\textbf{Remark:} with the \emph{Stochastic Sparsification} operator, this assumption can directly be translated into an assumption on the moments and regularity of $\g_k^i$. Note that Point 2 in \Cref{asu:regnoise} is an extension of \Cref{asu:noise_sto_grad} to higher order moments, but \textbf{still at the optimal point}.
Under this assumption, we have the following lemma:
\begin{lemma}\label{lem:sqrdnorm_asympt}
Under \Cref{asu:cocoercivityREG,asu:strongcvx,asu:cocoercivity,asu:noise_sto_grad,asu:bounded_noises_across_devices,asu:expec_quantization_operator,asu:regnoise}, we have that \begin{equation}
    \E_{\pi_{\gamma,v}}\left[ \SqrdNrm{w-w_*}\right] \underset{\gamma \to 0}{=} \gamma  \tr(A \ \cov(\Xi(w_*, h_*))) + O(\gamma^2),
\end{equation}
with $A:= (F''(w_*) \otimes I + I \otimes F''(w_*))^{-1}$. 
\end{lemma}

The intuition of the proof is natural: using the stability of the limit distribution, we have that if we start from the stationary distribution, i.e., $(w_0, h_0) \sim \pigv$, then $(w_1, h_1) \sim \pigv$.

We can thus write:
\begin{align*}
    \E_{\pi_{\gamma,v}}\left[(w-w_*)^{\otimes 2}\right] &= \E\left[(w_1-w_*)^{\otimes 2}\right] \\
    &= \E\left[(w_0-w_* -\gamma \nabla F(w_0) + \gamma \Xi(w_0,h_0))^{\otimes 2}\right] .
\end{align*}
Then, expanding the right hand side and using the fact that $\E[\Xi(w_0, h_0)|\Fdwncomp_0]=0$, then the fact that  $\E\left[(w_1-w_*)^{\otimes 2}\right] =\E\left[(w_0-w_*)^{\otimes 2}\right] $, and expanding the derivative of $F$ around $w_*$ (this is where we require the regularity assumption \Cref{asu:cocoercivityREG}), we get that:
\begin{align*}
    \gamma \left( F''(w_*) \otimes I + I \otimes F''(w_*) + O(\gamma)\right) \E_{\pi_{\gamma,v}}\left[(w-w_*)^{\otimes 2}\right]
    &\underset{\gamma \to 0}{=}  \gamma^2 \E_{(w,h)\sim \pigv}\left[ \Xi(w,h)^{\otimes 2}\right].
   \end{align*}
   Thus:
   \begin{align*}
   \E_{\pi_{\gamma,v}}\left[(w-w_*)^{\otimes 2}\right]
    &\underset{\gamma \to 0}{=}  \gamma A \E_{(w,h)\sim \pigv}\left[ \Xi(w,h)^{\otimes 2}\right]  + O(\gamma^2).\\
    \Rightarrow \E_{\pi_{\gamma,v}}\left[\SqrdNrm{(w-w_*)}\right]
    &\underset{\gamma \to 0}{=}  \gamma \tr \left( A \E_{(w,h)\sim \pigv}\left[ \Xi(w,h)^{\otimes 2}\right] \right)  + O(\gamma^2).
\end{align*}
Finally, we use that $\E_{(w,h)\sim \pigv}\left[ \cov(\Xi(w,h))\right] \underset{\gamma \to 0}{=}  \cov  (\Xi(w_*,h_*)) + O(\gamma)$ (which is derived from \Cref{asu:regnoise}) to get \Cref{lem:sqrdnorm_asympt}.

More formally, we can rely on Theorem 4 in \citet{dieuleveut_bridging_2018}: under \Cref{asu:cocoercivityREG,asu:strongcvx,asu:cocoercivity,asu:noise_sto_grad,asu:bounded_noises_across_devices,asu:expec_quantization_operator,asu:regnoise}, all assumptions required for the application of the theorem are verified and the result follows.

To conclude the proof, it only remains  to control $\cov  (\Xi(w_*,h_*))$. We have the following Lemma:
\begin{lemma}\label{lem:covoptLB}
Under \Cref{asu:noise_sto_gradLB,asu:expec_quantization_operatorLB,asu:bounded_noises_across_devicesLB}, we have that, for any variant $v$ of the algorithm, with the constant $E$ given in \Cref{thm:cvgce_artemis} depending on the variant:
\begin{equation}
    \tr\left(\cov  (\Xi(w_*,h_*))\right) = \Omega\left(\frac{\gamma E}{\mu N}\right). 
\end{equation}
\end{lemma}

Combining \Cref{lem:covoptLB,lem:sqrdnorm_asympt} and using the observation that $A$ is lower bounded by $\frac{1}{2L}$ independently of $\gamma, N, \sigmstar,B $, we have proved the following proposition:
\begin{proposition}\label{prop:cvdistpoint2}
Under \Cref{asu:noise_sto_gradLB,asu:expec_quantization_operatorLB,asu:bounded_noises_across_devicesLB,asu:cocoercivityREG,asu:strongcvx,asu:cocoercivity,asu:noise_sto_grad,asu:bounded_noises_across_devices,asu:expec_quantization_operator,asu:regnoise}, we have that 
\begin{equation}
  \E[\|w_k-w_*\|^2] \underset{k\to \infty}{\to }  \E_{\pi_{\gamma,v}}\left[ \SqrdNrm{w-w_*}\right] \underset{\gamma \to 0} = \Omega \left(\frac{\gamma E}{\mu N}\right) + O(\gamma^2),
\end{equation}
where the constant in the $\Omega$ is independent of $N, \sigmstar, \gamma, B$ (it depends only on the regularity of the operator $A$).
\end{proposition}

Before giving the proof, we make a couple of observations:
\begin{enumerate}
    \item This shows that the upper bound on the limit mean squared error given in \Cref{thm:cvgce_artemis} \textbf{is tight} with respect to $N, \sigmstar, \gamma, B$. This underlines that the conditions on the problem that we have used are  the correct  ones to understand convergence.
    \item The upper bound is possibly not tight with respect to $\mu$, as is clear from the proof: the tight bound is actually $\tr(A\cov  (\Xi(w_*,h_*)))$. Getting a tight upper bound involving the eigenvalue decomposition of $A$ instead of only $\mu$ is an open direction.
    \item In the memory-less case, $h \equiv 0$ and all the proof can be carried out analyzing only the distribution of the iterates $(w_k)_k$ and not necessarily the couple $(w_k, (h^i_k)_i)_k$.
\end{enumerate}
We now give the proof of \Cref{lem:covoptLB}.

\begin{proof}
With memory, we have the following:
\begin{align*}
     \tr\left(\cov  (\Xi(w_*,h_*))\right) &= \E\left[\left\| \C_{\dwn} \left(\frac{1}{N} \sum_\iN \C_\up(\g_{1}^i(w_*) - h^i_*)+ h_*^i \right)\right\|^2\right] \\
     &\overset{\text{(i)}}{=} (1+\omgC^\dwn) \E\left[\left\| \frac{1}{N} \sum_\iN \C_\up(\g_{1}^i(w_*) - h^i_*)+ h_*^i \right\|^2\right] \\
    &\overset{\text{(ii)}}{=} \frac{(1+\omgC^\dwn)}{N^2}  \sum_\iN \E\left[\left\|  \C_\up(\g_{1}^i(w_*) - h^i_*)\right\|^2\right] \\
    &\overset{\text{(iii)}}{=}  \frac{(1+\omgC^\dwn)}{N^2}   \sum_{i=1}^N \E\left[\left\|  B_k^i \C_\up(\g_{1}^i(w_*) - h^i_*)\right\|^2\right] \\
  &\overset{\text{(iv)}}{=}  \frac{(1+\omgC^\dwn)}{N^2}   \sum_{i=1}^N (1+\omgC^\up) \E\left[\left\|  \g_{1}^i(w_*) - h^i_*\right\|^2\right] \\
  &\overset{\text{(v)}}{\geq}  \frac{(1+\omgC^\dwn)}{N}   (1+\omgC^\up) \frac{\sigmstar^2}{b} .
\end{align*}
At line (i) we use \Cref{asu:expec_quantization_operatorLB} for the downlink compression operator with constant $\omgC^\dwn$. 
At line (ii) we use the fact that $ \sum_{i=1}^{N}  h_*^i = \nabla F(w_*)=0 $, the  independence of the random variables $\C_\up(\g_{1}^i(w_*) - h_*^i), \C_\up(\g_{1}^j(w_*)- h_*^j)$ for $i\neq j$ and the fact that they have 0 mean.
Line (iii) makes appear the bernouilli variable $B_k^i$ that mark if a worker is activate or not at round $k$.
We use  \Cref{asu:expec_quantization_operatorLB} for the uplink compression operator with constant $\omgC^\up$ in line (iv); and finally \Cref{asu:noise_sto_gradLB} at line (v) to lower bound the variance of the gradients at the optimum. This proof applies to both simple and double compression with $\omgC^\dwn =0 $ or not.

Remark that for the variant 2 of \Artemis, the constant $E$ given in \Cref{thm:cvgce_artemis} has a factor $\alpha^2 \cst (\omgC+1)$: combining with the value of $\cst$, this term is indeed of the order of $(1+\omgC^\dwn) (1+\omgC^\up)$.

Without memory, we have the following computation:
\begin{align*}
     \tr\left(\cov  (\Xi(w_*,0))\right) &= \E\left[\left\| \C_{\dwn} \left(\frac{1}{N} \sum_{i=1}^{N} \C_\up(\g_{1}^i(w_*) ) \right)\right\|^2\right] \\
      &\overset{\text{(i)}}{=} (1+\omgC^\dwn) \E\left[\left\| \frac{1}{N} \sum_{i=1}^{N} \C_\up(\g_{1}^i(w_*) ) - h^i_*  \right\|^2\right] \\
    &\overset{\text{(ii)}}{=} \frac{(1+\omgC^\dwn)}{N^2}  \sum_{i=1}^{N} \E\left[\left\|  \C_\up(\g_{1}^i(w_*))- h^i_*\right\|^2\right] \\
     &\overset{\text{(iii)}}{=} \frac{(1+\omgC^\dwn)}{N^2}  \sum_{i=1}^{N} \E\left[\left\|  \C_\up(\g_{1}^i(w_*))-\g_{1}^i(w_*)\right\|^2 + \left\|  \g_{1}^i(w_*) - h^i_*\right\|^2\right]
\end{align*}

At line (i) we use \Cref{asu:expec_quantization_operatorLB} for the downlink compression operator with constant $\omgC^\dwn$ and the fact that  $ \sum_{i=1}^{N}  h_*^i = \nabla F(w_*)=0$, then at line (ii) the independence of the random variables $\C_\up(\g_{1}^i(w_*))- h^i_*$ with mean 0, then a Bias Variance decomposition at line (iii).
\begin{align*}
     \tr\left(\cov  (\Xi(w_*,0))\right) &\overset{\text{(iv)}}{=} \frac{(1+\omgC^\dwn)}{N^2}  \sum_{i=1}^{N} \E\left[\omgC^\up \left\|  (\g_{1}^i(w_*))\right\|^2 +  \left\|  \g_{1}^i(w_*) - h^i_*\right\|^2\right] \\
  &\overset{\text{(v)}}{=} \frac{(1+\omgC^\dwn)}{N^2}  \sum_{i=1}^{N} \E\left[\omgC^\up \left(\left\|  \g_{1}^i(w_*)- h_i^*\right\|^2 + \left\|   h_i^*\right\|^2 \right)  +  \left\|  \g_{1}^i(w_*) - h^i_*\right\|^2\right] \\  &\overset{\text{(vi)}}{=} \frac{(1+\omgC^\dwn)}{N}   \left((\omgC^\up +1)\frac{\sigmstar^2}{b} +\omgC^\up B^2 \right).
\end{align*}

Next we use \Cref{asu:expec_quantization_operatorLB} for the uplink compression operator with constant $\omgC^\up$ at line (iv). Line (v) is another Bias-Variance decomposition and we finally conclude by using \Cref{asu:bounded_noises_across_devicesLB,asu:noise_sto_gradLB} at line (vi) and reorganizing terms.

We have showed the lower bound both with or without memory, which concludes the proof.
\end{proof}

\end{document}